\documentclass[11pt,a4paper]{article}
\usepackage[margin=1in]{geometry}
\usepackage{times}
\usepackage{tikz}
\usepackage{dsfont}
\usepackage{url}
\usepackage{amsmath}
\usepackage{amsfonts}
\usepackage{amssymb}
\usepackage{booktabs}
\usepackage{algorithm}
\usepackage{algpseudocode}
\usepackage{etoolbox}
\usepackage{verbatim}
\usepackage{multirow}
\usepackage{rotating}
\usepackage{graphicx}
\usepackage{authblk}
\usepackage[numbers,sort&compress]{natbib}
\usepackage[colorlinks=true,linkcolor=blue,citecolor=blue,urlcolor=blue]{hyperref}
\usepackage{bm}
\usepackage{microtype}
\newtheorem{assumption}{Assumption}

\newcommand{\mat}[1]{\mathbf{#1}} 
\renewcommand{\vec}[1]{\mathbf{#1}} 
\newcommand{\T}{\mathcal{T}} 
\newcommand{\X}{\mat{X}}
\newcommand{\W}{\mat{W}}
\newcommand{\N}{\mat{N}}
\newcommand{\M}{\mat{M}}
\newcommand{\V}{\mat{V}}
\newcommand{\Phimap}{\mathbf{\Phi}}

\newtheorem{theorem}{Theorem}
\newtheorem{corollary}[theorem]{Corollary}
\newtheorem{definition}[theorem]{Definition}
\newtheorem{lemma}[theorem]{Lemma}
\newtheorem{proposition}[theorem]{Proposition}
\newtheorem{remark}[theorem]{Remark}
\newenvironment{proof}[1][Proof]{\textbf{#1.} }{\ \rule{0.5em}{0.5em}}

\newcommand{\invertedcommas}[1]{``#1''}
\newcommand{\bPi}{\mathbf \Pi}

\renewcommand{\H}{\mathbf H}
\newcommand{\PP}{\mathbf P}
\newcommand{\0}{\mathbf 0}

\newcommand{\A}{\mathbf A}
\newcommand{\B}{\mathbf B}
\newcommand{\C}{\mathbf C}
\newcommand{\DD}{\mathbf D}

\newcommand{\R}{\mathbf R}

\newcommand{\conv}{\mathrm{Conv}}
\newcommand{\Cone}{\mathrm{Cone}}

\def\eq{\begin{equation}}
\def\qe{\end{equation}}
\def\bz{\mathbf z}
\def\bU{\mathbf U}
\def\conv{\mathrm{Conv}}
\def\bX{\mathbf X}
\def\bH{\mathbf H}

\def\G{{\mathbf G}}

\def\sT{{\sf T}}

\def\conv{{\rm conv}}

\def\one{\mathbf{1}}

\def\bq{\boldsymbol{q}}
\newcommand\norm[1]{\lVert{#1}\rVert}

\def\W{{\mathbf W}}

\def\one{{\bf 1}}

\def\conv{{\rm{conv}}}
\def\aff{{\rm{aff}}}

\def\reals{\mathds R}

\def\bM{{\boldsymbol M}}
\def\bE{{\boldsymbol E}}
\def\by{{\boldsymbol y}}

\def\bQ{{\boldsymbol Q}}

\def\Id{{\bf I}}

\def\bH{{\boldsymbol H}}

\def\bPi{{\boldsymbol \Pi}}
\def\be{{\boldsymbol e}}

\def\bx{{\boldsymbol x}}
\def\bz{{\boldsymbol z}}
\def\ba{{\boldsymbol a}}

\def\bX{{\boldsymbol X}}
\def\bY{{\boldsymbol Y}}
\def\bZ{{\boldsymbol Z}}
\def\bP{{\boldsymbol P}}
\def\bv{{\boldsymbol v}}
\def\bu{{\boldsymbol u}}

\def\bW{{\boldsymbol W}}
\def\hbH{\widehat{\boldsymbol H}}
\def\tbH{\widetilde{\boldsymbol H}}

\def\bU{{\boldsymbol U}}
\def\bV{{\boldsymbol V}}
\def\bR{{\boldsymbol R}}
\def\bA{{\boldsymbol A}}

\def\balpha{{\boldsymbol \alpha}}

\def\ext{{\rm {ext}}}

\def\bSigma{{\boldsymbol \Sigma}}

\def\Treg[#1]{T^{{\rm reg},#1}}
\def\GW[#1]{{\rm GW}(#1)}
\def\MGW[#1]{{\rm MGW}(#1)}

\def\cuD{\mathcal{D}}
\def\cuDsym{\mathcal{D}_{\mathrm{sym}}}

\def\cuL{\mathcal{L}}

\def\balpha{{\boldsymbol \alpha}}

\begin{document}

\title{Time series forecasting from partial observations \\ via Non-negative Matrix Factorization}

\author[1,3]{Yohann De Castro\thanks{Email: \texttt{yohann.de-castro@ec-lyon.fr}}}
\author[2]{Luca Mencarelli\thanks{Email: \texttt{luca.mencarelli@unipi.it}}}
\affil[1]{Institut Camille Jordan, Ecole Centrale Lyon, France}
\affil[2]{Dipartimento di Informatica, Università di Pisa, Italia}
\affil[3]{Institut Universitaire de France}

\date{}

\maketitle

\begin{abstract}
In modern time series problems, one aims at forecasting multiple time series with possible missing and noisy values. In this paper, we introduce the Sliding Mask Method (SMM) for forecasting multiple nonnegative time series by means of nonnegative matrix completion: observed noisy values and forecast/missing values are collected into matrix form, and learning is achieved by representing its rows as a convex combination of a small number of nonnegative vectors, referred to as the archetypes. We introduce two estimates, the mask Archetypal Matrix factorization (mAMF) and the mask normalized Nonnegative Matrix Factorization (mNMF) which can be combined with the SMM method.  We prove that these estimates recover the true archetypes with an error proportional to the noise. We use a proximal alternating linearized method (PALM) to compute the archetypes and the convex combination weights. We compared our estimators with state-of-the-art methods (Transformers, LSTM, SARIMAX...) in multiple time series forecasting on real data and obtain that our method outperforms them in most of the experiments.%
\end{abstract}

\medskip

\section{Introduction}
This article investigates forecasting multiple nonnegative time series with missing or noisy entries. We observe $N\geq 1$ time series $\M^{(1)},\ldots, \M^{(N)}\in\mathds R^T$ over a period of time of length $T\geq 1$ and we would like to forecast the next $F\geq 1$ future values by means of matrix completion, see Figure~\ref{fig:arche_model}. We define a matrix $\M\in\R^{N\times T}$ whose rows are denoted by~$\M^{(i)}$ and columns by~$\M_j$. The forecast columns are $\hat{\M}_{T+k}$ for $k=1,\ldots,F$.

\begin{figure*}[!t]
    \centering
    \resizebox{0.8\textwidth}{!}{
    \begin{tikzpicture}

\newcommand{\locone}{-0.5}

\filldraw [fill=green!40!white,draw=green!40!black] (\locone,0-0.5) rectangle (\locone+3,3-0.5);

\filldraw [fill=gray!30] (\locone,0) rectangle (\locone+0.5,0.5);
\filldraw [fill=gray!30] (\locone,1) rectangle (\locone+0.5,1.5);
\filldraw [fill=gray!30] (\locone+0.5,1.5) rectangle (\locone+1,2);
\filldraw [fill=gray!30] (\locone+1,0.5) rectangle (\locone+1.5,1);
\filldraw [fill=gray!30] (\locone+1.5,2) rectangle (\locone+2,2.5);
\filldraw [fill=gray!30] (\locone+2.5,-0.5) rectangle (\locone+3,0);

\draw [step=0.5, very thick, color=white] (\locone+0,0-0.5) grid (\locone+3,3-0.5);
\draw [thick] (\locone,0-0.5) rectangle (\locone+3,3-0.5);

\node at (\locone+0.25,2.8) {\footnotesize\color{black}$\M_1$};
\node at (\locone+0.75,2.8) {\footnotesize\color{black}$\M_2$};
\node at (\locone+1.75,2.8) {\footnotesize\color{black}$\cdots$};
\node at (\locone+2.75,2.8) {\footnotesize\color{black}$\M_T$};
\node at (\locone+1.5,-0.9) {\small data with missing values};

\node at (\locone+0.25,0.25) {\color{black}?};
\node at (\locone+0.25,1.25) {\color{black}?};
\node at (\locone+0.75,1.75) {\color{black}?};
\node at (\locone+1.25,0.75) {\color{black}?};
\node at (\locone+1.75,2.25) {\color{black}?};
\node at (\locone+2.75,-0.25) {\color{black}?};

\node at (\locone+4,1) {\Large\color{gray}$\Longrightarrow$};
\node at (\locone+4,1.5) {\small Modeling};

\newcommand{\loctwo}{4.5}

\filldraw [fill=red!25,draw=green!40!black] (\loctwo,0-0.5) rectangle (\loctwo+1.5,3-0.5);
\draw [step=0.5, very thick, color=white] (\loctwo+0,0-0.5) grid (\loctwo+1.5,3-0.5);
\draw [thick] (\loctwo,0-0.5) rectangle (\loctwo+1.5,3-0.5);

\node at (\loctwo+0.75,-0.9) {\small forecasts};

\node at (\loctwo,2.8) {\footnotesize\color{black}$\hat{\M}_{T+1}$};
\node at (\loctwo+0.75,2.8) {\footnotesize\color{black}$\cdots$};
\node at (\loctwo+1.5,2.8) {\footnotesize\color{black}$\hat{\M}_{T+F}$};

\renewcommand{\locone}{10.5}
\renewcommand{\loctwo}{13.5}

\node at (\locone-1.2,1) {\color{black}$\Big[\M\,\mathcal F_{N\times F}\Big]=$};

\filldraw [fill=green!40!white,draw=green!40!black] (\locone,0-0.5) rectangle (\locone+3,3-0.5);

\filldraw [fill=gray!30] (\locone,0) rectangle (\locone+0.5,0.5);
\filldraw [fill=gray!30] (\locone,1) rectangle (\locone+0.5,1.5);
\filldraw [fill=gray!30] (\locone+0.5,1.5) rectangle (\locone+1,2);
\filldraw [fill=gray!30] (\locone+1,0.5) rectangle (\locone+1.5,1);
\filldraw [fill=gray!30] (\locone+1.5,2) rectangle (\locone+2,2.5);
\filldraw [fill=gray!30] (\locone+2.5,-0.5) rectangle (\locone+3,0);

\draw [step=0.5, very thick, color=white] (\locone+0,0-0.5) grid (\locone+4,3-0.5);

\node at (\locone+0.25,0.25) {\color{black}$\diamondsuit$};
\node at (\locone+0.25,1.25) {\color{black}$\diamondsuit$};
\node at (\locone+0.75,1.75) {\color{black}$\diamondsuit$};
\node at (\locone+1.25,0.75) {\color{black}$\diamondsuit$};
\node at (\locone+1.75,2.25) {\color{black}$\diamondsuit$};
\node at (\locone+2.75,-0.25) {\color{black}$\diamondsuit$};

\filldraw [fill=red!25,draw=green!40!black] (\loctwo,0-0.5) rectangle (\loctwo+1.5,3-0.5);
\draw [step=0.5, very thick, color=white] (\loctwo-0.5,0-0.5) grid (\loctwo+1.5,3-0.5);
\draw [thick] (\locone,0-0.5) rectangle (\locone+4.5,3-0.5);

\node at (\loctwo+0.25,2.25) {\color{black}$\diamondsuit$};
\node at (\loctwo+0.25,1.75) {\color{black}$\diamondsuit$};
\node at (\loctwo+0.25,1.25) {\color{black}$\diamondsuit$};
\node at (\loctwo+0.25,0.75) {\color{black}$\diamondsuit$};
\node at (\loctwo+0.25,0.25) {\color{black}$\diamondsuit$};
\node at (\loctwo+0.25,-0.25) {\color{black}$\diamondsuit$};

\node at (\loctwo+0.75,2.25) {\color{black}$\diamondsuit$};
\node at (\loctwo+0.75,1.75) {\color{black}$\diamondsuit$};
\node at (\loctwo+0.75,1.25) {\color{black}$\diamondsuit$};
\node at (\loctwo+0.75,0.75) {\color{black}$\diamondsuit$};
\node at (\loctwo+0.75,0.25) {\color{black}$\diamondsuit$};
\node at (\loctwo+0.75,-0.25) {\color{black}$\diamondsuit$};

\node at (\loctwo+1.25,2.25) {\color{black}$\diamondsuit$};
\node at (\loctwo+1.25,1.75) {\color{black}$\diamondsuit$};
\node at (\loctwo+1.25,1.25) {\color{black}$\diamondsuit$};
\node at (\loctwo+1.25,0.75) {\color{black}$\diamondsuit$};
\node at (\loctwo+1.25,0.25) {\color{black}$\diamondsuit$};
\node at (\loctwo+1.25,-0.25) {\color{black}$\diamondsuit$};

\end{tikzpicture}
    }
    \caption{{\tt [Left]} Noisy multiple time series observations (green) with possible missing entries (question mark) from $N\geq 1$ time series and their~$F$ forecast values in red. {\tt [Right]} Matrix completion problem under consideration: missing and forecast values~$\mathcal F_{N\times F}$ (gray and red diamonds) are not~observed.
    }
    \label{fig:arche_model}
\end{figure*}

\medskip

The matrix completion problem depicted in Figure~\ref{fig:arche_model} is ill-posed; standard low-rank techniques cannot recover the missing future values without structural assumptions. To address this, we introduce a deterministic transformation $\Phimap$ based on a sliding window approach, referred to as the \emph{Sliding Mask Method (SMM)}.

\begin{itemize}
    \item \textbf{Stride Parameter ($P$):} We define a scalar $P \geq 1$, which determines the stride (or step size) of the sliding window. While often chosen to match a suspected seasonality in the data (e.g., $P=7$ for weekly cycles), $P$ is a user-defined hyperparameter and does not strictly require intrinsic signal periodicity.
    
    \item \textbf{Block Construction:} We partition the total time horizon $T+F$ into $B$ blocks of length $P$. To ensure integer division, we pad the end of the time series with at most $P-1$ placeholder columns (which are treated as unobserved). Thus, $B = \lceil (T+F)/P \rceil$.
    
    \item \textbf{Sliding Window Transformation:} The output matrix is constructed by stacking windows of length $WP$, where $W$ is the number of consecutive sub-blocks per window. This transforms the original $N \times (T+F)$ matrix into a larger matrix where rows represent local time-segments.
    
    \item \textbf{Forecasting as Completion:} We assume $P \geq F$, i.e., the stride is at least as large as the forecast horizon. This guarantees that every forecast column is contained in the last window only, so all unobserved future values are gathered into the bottom-right $N \times F$ block $\mathcal F_{N\times F}$ of the SMM matrix (Figure~\ref{fig:blocks}). The resulting structured pattern recasts the temporal forecasting problem as a matrix completion problem.
\end{itemize}

\begin{figure*}[!t]
    \centering
\resizebox{0.8\textwidth}{!}{%
    \begin{tikzpicture}
    
\newcommand{\locone}{0}
\newcommand{\loctwo}{3}

\node at (\locone+2,+4.75) {\color{black}$\Big[\M\,\mathcal F_{N \times F}\Big]$};
\node at (\locone+2,-2.75) {\color{black}$N\times(T+F)$};

\filldraw [fill=green!40!white,draw=green!40!black] (\locone,0-0.5) rectangle (\locone+3,3-0.5);

\filldraw [fill=gray!30] (\locone,0) rectangle (\locone+0.5,0.5);
\filldraw [fill=gray!30] (\locone,1) rectangle (\locone+0.5,1.5);
\filldraw [fill=gray!30] (\locone+0.5,1.5) rectangle (\locone+1,2);
\filldraw [fill=gray!30] (\locone+1,0.5) rectangle (\locone+1.5,1);
\filldraw [fill=gray!30] (\locone+1.5,2) rectangle (\locone+2,2.5);
\filldraw [fill=gray!30] (\locone+2.5,-0.5) rectangle (\locone+3,0);

\draw [step=0.5, very thick, color=white] (\locone+0,0-0.5) grid (\locone+4,3-0.5);

\node at (\locone+0.25,0.25) {\color{black}$e$};
\node at (\locone+0.25,1.25) {\color{black}$c$};
\node at (\locone+0.75,1.75) {\color{black}$b$};
\node at (\locone+1.25,0.75) {\color{black}$d$};
\node at (\locone+1.75,2.25) {\color{black}$a$};
\node at (\locone+2.75,-0.25) {\color{black}$f$};

\filldraw [fill=red!25,draw=green!40!black] (\loctwo,0-0.5) rectangle (\loctwo+1.5,3-0.5);
\draw [step=0.5, very thick, color=white] (\loctwo-0.5,0-0.5) grid (\loctwo+1.5,3-0.5);
\draw [thick] (\locone,0-0.5) rectangle (\locone+4.5,3-0.5);

\node at (\loctwo+0.25,2.25) {\color{black}$g$};
\node at (\loctwo+0.25,1.75) {\color{black}$j$};
\node at (\loctwo+0.25,1.25) {\color{black}$m$};
\node at (\loctwo+0.25,0.75) {\color{black}$p$};
\node at (\loctwo+0.25,0.25) {\color{black}$s$};
\node at (\loctwo+0.25,-0.25) {\color{black}$v$};

\node at (\loctwo+0.75,2.25) {\color{black}$h$};
\node at (\loctwo+0.75,1.75) {\color{black}$k$};
\node at (\loctwo+0.75,1.25) {\color{black}$n$};
\node at (\loctwo+0.75,0.75) {\color{black}$q$};
\node at (\loctwo+0.75,0.25) {\color{black}$t$};
\node at (\loctwo+0.75,-0.25) {\color{black}$w$};

\node at (\loctwo+1.25,2.25) {\color{black}$i$};
\node at (\loctwo+1.25,1.75) {\color{black}$l$};
\node at (\loctwo+1.25,1.25) {\color{black}$o$};
\node at (\loctwo+1.25,0.75) {\color{black}$r$};
\node at (\loctwo+1.25,0.25) {\color{black}$u$};
\node at (\loctwo+1.25,-0.25) {\color{black}$x$};

\newcommand{\locthree}{6.5}

\node at (\locthree+1.5,-2.75) {\color{black}$N(B+1-W)\times WP$};

\filldraw [fill=green!40!white,draw=green!40!black] (\locthree,3.5) rectangle (\locthree+0.5,4);
\filldraw [fill=green!40!white,draw=green!40!black] (\locthree+0.5,3.5) rectangle (\locthree+1,4);
\filldraw [fill=green!40!white,draw=green!40!black] (\locthree+1,3.5) rectangle (\locthree+1.5,4);
\filldraw [fill=gray!30] (\locthree+1.5,3.5) rectangle (\locthree+2,4);
\node at (\locthree+1.75,3.75) {\color{black}$a$};
\filldraw [fill=green!40!white,draw=green!40!black] (\locthree+2,3.5) rectangle (\locthree+2.5,4);
\filldraw [fill=green!40!white,draw=green!40!black] (\locthree+2.5,3.5) rectangle (\locthree+3,4);

\filldraw [fill=gray!30] (\locthree,3) rectangle (\locthree+0.5,3.5);
\node at (\locthree+0.25,3.25) {\color{black}$a$};
\filldraw [fill=green!40!white,draw=green!40!black] (\locthree+0.5,3) rectangle (\locthree+1,3.5);
\filldraw [fill=green!40!white,draw=green!40!black] (\locthree+1,3) rectangle (\locthree+1.5,3.5);
\filldraw [fill=red!25,draw=green!40!black] (\locthree+1.5,3) rectangle (\locthree+2,3.5);
\node at (\locthree+1.75,3.25) {\color{black}$g$};
\filldraw [fill=red!25,draw=green!40!black] (\locthree+2,3) rectangle (\locthree+2.5,3.5);
\node at (\locthree+2.25,3.25) {\color{black}$h$};
\filldraw [fill=red!25,draw=green!40!black] (\locthree+2.5,3) rectangle (\locthree+3,3.5);
\node at (\locthree+2.75,3.25) {\color{black}$i$};

\filldraw [fill=green!40!white,draw=green!40!black] (\locthree,2.5) rectangle (\locthree+0.5,3);
\filldraw [fill=gray!30] (\locthree+0.5,2.5) rectangle (\locthree+1,3);
\node at (\locthree+0.75,2.75) {\color{black}$b$};
\filldraw [fill=green!40!white,draw=green!40!black] (\locthree+1,2.5) rectangle (\locthree+1.5,3);
\filldraw [fill=green!40!white,draw=green!40!black] (\locthree+1.5,2.5) rectangle (\locthree+2,3);
\filldraw [fill=green!40!white,draw=green!40!black] (\locthree+2,2.5) rectangle (\locthree+2.5,3);
\filldraw [fill=green!40!white,draw=green!40!black] (\locthree+2.5,2.5) rectangle (\locthree+3,3);

\filldraw [fill=green!40!white,draw=green!40!black] (\locthree,2) rectangle (\locthree+0.5,2.5);
\filldraw [fill=green!40!white,draw=green!40!black] (\locthree+0.5,2) rectangle (\locthree+1,2.5);
\filldraw [fill=green!40!white,draw=green!40!black] (\locthree+1,2) rectangle (\locthree+1.5,2.5);
\filldraw [fill=red!25,draw=green!40!black] (\locthree+1.5,2) rectangle (\locthree+2,2.5);
\filldraw [fill=red!25,draw=green!40!black] (\locthree+2,2) rectangle (\locthree+2.5,2.5);
\filldraw [fill=red!25,draw=green!40!black] (\locthree+2.5,2) rectangle (\locthree+3,2.5);
\node at (\locthree+1.75,2.25) {\color{black}$j$};
\node at (\locthree+2.25,2.25) {\color{black}$k$};
\node at (\locthree+2.75,2.25) {\color{black}$l$};

\filldraw [fill=gray!30] (\locthree,1.5) rectangle (\locthree+0.5,2);
\node at (\locthree+0.25,1.75) {\color{black}$c$};
\filldraw [fill=green!40!white,draw=green!40!black] (\locthree+0.5,1.5) rectangle (\locthree+1,2);
\filldraw [fill=green!40!white,draw=green!40!black] (\locthree+1,1.5) rectangle (\locthree+1.5,2);
\filldraw [fill=green!40!white,draw=green!40!black] (\locthree+1.5,1.5) rectangle (\locthree+2,2);
\filldraw [fill=green!40!white,draw=green!40!black] (\locthree+2,1.5) rectangle (\locthree+2.5,2);
\filldraw [fill=green!40!white,draw=green!40!black] (\locthree+2.5,1.5) rectangle (\locthree+3,2);

\filldraw [fill=green!40!white,draw=green!40!black] (\locthree,1) rectangle (\locthree+0.5,1.5);
\filldraw [fill=green!40!white,draw=green!40!black] (\locthree+0.5,1) rectangle (\locthree+1,1.5);
\filldraw [fill=green!40!white,draw=green!40!black] (\locthree+1,1) rectangle (\locthree+1.5,1.5);
\filldraw [fill=red!25,draw=green!40!black] (\locthree+1.5,1) rectangle (\locthree+2,1.5);
\filldraw [fill=red!25,draw=green!40!black] (\locthree+2,1) rectangle (\locthree+2.5,1.5);
\filldraw [fill=red!25,draw=green!40!black] (\locthree+2.5,1) rectangle (\locthree+3,1.5);
\node at (\locthree+1.75,1.25) {\color{black}$m$};
\node at (\locthree+2.25,1.25) {\color{black}$n$};
\node at (\locthree+2.75,1.25) {\color{black}$o$};

\filldraw [fill=green!40!white,draw=green!40!black] (\locthree,0.5) rectangle (\locthree+0.5,1);
\filldraw [fill=green!40!white,draw=green!40!black] (\locthree+0.5,0.5) rectangle (\locthree+1,1);
\filldraw [fill=gray!30] (\locthree+1,0.5) rectangle (\locthree+1.5,1);
\node at (\locthree+1.25,0.75) {\color{black}$d$};
\filldraw [fill=green!40!white,draw=green!40!black] (\locthree+1.5,0.5) rectangle (\locthree+2,1);
\filldraw [fill=green!40!white,draw=green!40!black] (\locthree+2,0.5) rectangle (\locthree+2.5,1);
\filldraw [fill=green!40!white,draw=green!40!black] (\locthree+2.5,0.5) rectangle (\locthree+3,1);

\filldraw [fill=green!40!white,draw=green!40!black] (\locthree,0) rectangle (\locthree+0.5,0.5);
\filldraw [fill=green!40!white,draw=green!40!black] (\locthree+0.5,0) rectangle (\locthree+1,0.5);
\filldraw [fill=green!40!white,draw=green!40!black] (\locthree+1,0) rectangle (\locthree+1.5,0.5);
\filldraw [fill=red!25,draw=green!40!black] (\locthree+1.5,0) rectangle (\locthree+2,0.5);
\filldraw [fill=red!25,draw=green!40!black] (\locthree+2,0) rectangle (\locthree+2.5,0.5);
\filldraw [fill=red!25,draw=green!40!black] (\locthree+2.5,0) rectangle (\locthree+3,0.5);
\node at (\locthree+1.75,0.25) {\color{black}$p$};
\node at (\locthree+2.25,0.25) {\color{black}$q$};
\node at (\locthree+2.75,0.25) {\color{black}$r$};

\filldraw [fill=gray!30] (\locthree,-0.5) rectangle (\locthree+0.5,0);
\node at (\locthree+0.25,-0.25) {\color{black}$e$};
\filldraw [fill=green!40!white,draw=green!40!black] (\locthree+0.5,-0.5) rectangle (\locthree+1,0);
\filldraw [fill=green!40!white,draw=green!40!black] (\locthree+1,-0.5) rectangle (\locthree+1.5,0);
\filldraw [fill=green!40!white,draw=green!40!black] (\locthree+1.5,-0.5) rectangle (\locthree+2,0);
\filldraw [fill=green!40!white,draw=green!40!black] (\locthree+2,-0.5) rectangle (\locthree+2.5,0);
\filldraw [fill=green!40!white,draw=green!40!black] (\locthree+2.5,-0.5) rectangle (\locthree+3,0);

\filldraw [fill=green!40!white,draw=green!40!black] (\locthree,-1) rectangle (\locthree+0.5,-0.5);
\filldraw [fill=green!40!white,draw=green!40!black] (\locthree+0.5,-1) rectangle (\locthree+1,-0.5);
\filldraw [fill=green!40!white,draw=green!40!black] (\locthree+1,-1) rectangle (\locthree+1.5,-0.5);
\filldraw [fill=red!25,draw=green!40!black] (\locthree+1.5,-1) rectangle (\locthree+2,-0.5);
\filldraw [fill=red!25,draw=green!40!black] (\locthree+2,-1) rectangle (\locthree+2.5,-0.5);
\filldraw [fill=red!25,draw=green!40!black] (\locthree+2.5,-1) rectangle (\locthree+3,-0.5);
\node at (\locthree+1.75,-0.75) {\color{black}$s$};
\node at (\locthree+2.25,-0.75) {\color{black}$t$};
\node at (\locthree+2.75,-0.75) {\color{black}$u$};

\filldraw [fill=green!40!white,draw=green!40!black] (\locthree,-1.5) rectangle (\locthree+0.5,-1);
\filldraw [fill=green!40!white,draw=green!40!black] (\locthree+0.5,-1.5) rectangle (\locthree+1,-1);
\filldraw [fill=green!40!white,draw=green!40!black] (\locthree+1,-1.5) rectangle (\locthree+1.5,-1);
\filldraw [fill=green!40!white,draw=green!40!black] (\locthree+1.5,-1.5) rectangle (\locthree+2,-1);
\filldraw [fill=green!40!white,draw=green!40!black] (\locthree+2,-1.5) rectangle (\locthree+2.5,-1);
\filldraw [fill=gray!30] (\locthree+2.5,-1.5) rectangle (\locthree+3,-1);
\node at (\locthree+2.75,-1.25) {\color{black}$f$};

\filldraw [fill=green!40!white,draw=green!40!black] (\locthree,-2) rectangle (\locthree+0.5,-1.5);
\filldraw [fill=green!40!white,draw=green!40!black] (\locthree+0.5,-2) rectangle (\locthree+1,-1.5);
\filldraw [fill=gray!30] (\locthree+1,-2) rectangle (\locthree+1.5,-1.5);
\node at (\locthree+1.25,-1.75) {\color{black}$f$};
\filldraw [fill=red!25,draw=green!40!black] (\locthree+1.5,-2) rectangle (\locthree+2,-1.5);
\filldraw [fill=red!25,draw=green!40!black] (\locthree+2,-2) rectangle (\locthree+2.5,-1.5);
\filldraw [fill=red!25,draw=green!40!black] (\locthree+2.5,-2) rectangle (\locthree+3,-1.5);
\node at (\locthree+1.75,-1.75) {\color{black}$v$};
\node at (\locthree+2.25,-1.75) {\color{black}$w$};
\node at (\locthree+2.75,-1.75) {\color{black}$x$};

\draw [step=0.5, very thick, color=white] (\locthree-0.5,-2) grid (\locthree+3,4);
\draw [thick] (\locthree,-2) rectangle (\locthree+3,4);

\renewcommand{\locthree}{12.5}

\filldraw [fill=green!40!white,draw=green!40!black] (\locthree,3.5) rectangle (\locthree+0.5,4);
\filldraw [fill=green!40!white,draw=green!40!black] (\locthree+0.5,3.5) rectangle (\locthree+1,4);
\filldraw [fill=green!40!white,draw=green!40!black] (\locthree+1,3.5) rectangle (\locthree+1.5,4);
\filldraw [fill=gray!30] (\locthree+1.5,3.5) rectangle (\locthree+2,4);
\node at (\locthree+1.75,3.75) {\color{black}$a$};
\filldraw [fill=green!40!white,draw=green!40!black] (\locthree+2,3.5) rectangle (\locthree+2.5,4);
\filldraw [fill=green!40!white,draw=green!40!black] (\locthree+2.5,3.5) rectangle (\locthree+3,4);

\filldraw [fill=gray!30] (\locthree,0.5) rectangle (\locthree+0.5,1);
\node at (\locthree+0.25,0.75) {\color{black}$a$};
\filldraw [fill=green!40!white,draw=green!40!black] (\locthree+0.5,0.5) rectangle (\locthree+1,1);
\filldraw [fill=green!40!white,draw=green!40!black] (\locthree+1,0.5) rectangle (\locthree+1.5,1);
\filldraw [fill=red!25,draw=green!40!black] (\locthree+1.5,0.5) rectangle (\locthree+2,1);
\node at (\locthree+1.75,0.75) {\color{black}$g$};
\filldraw [fill=red!25,draw=green!40!black] (\locthree+2,0.5) rectangle (\locthree+2.5,1);
\node at (\locthree+2.25,0.75) {\color{black}$h$};
\filldraw [fill=red!25,draw=green!40!black] (\locthree+2.5,0.5) rectangle (\locthree+3,1);
\node at (\locthree+2.75,0.75) {\color{black}$i$};

\filldraw [fill=green!40!white,draw=green!40!black] (\locthree,3) rectangle (\locthree+0.5,3.5);
\filldraw [fill=gray!30] (\locthree+0.5,3) rectangle (\locthree+1,3.5);
\node at (\locthree+0.75,3.25) {\color{black}$b$};
\filldraw [fill=green!40!white,draw=green!40!black] (\locthree+1,3) rectangle (\locthree+1.5,3.5);
\filldraw [fill=green!40!white,draw=green!40!black] (\locthree+1.5,3) rectangle (\locthree+2,3.5);
\filldraw [fill=green!40!white,draw=green!40!black] (\locthree+2,3) rectangle (\locthree+2.5,3.5);
\filldraw [fill=green!40!white,draw=green!40!black] (\locthree+2.5,3) rectangle (\locthree+3,3.5);

\filldraw [fill=green!40!white,draw=green!40!black] (\locthree,0) rectangle (\locthree+0.5,0.5);
\filldraw [fill=green!40!white,draw=green!40!black] (\locthree+0.5,0) rectangle (\locthree+1,0.5);
\filldraw [fill=green!40!white,draw=green!40!black] (\locthree+1,0) rectangle (\locthree+1.5,0.5);
\filldraw [fill=red!25,draw=green!40!black] (\locthree+1.5,0) rectangle (\locthree+2,0.5);
\filldraw [fill=red!25,draw=green!40!black] (\locthree+2,0) rectangle (\locthree+2.5,0.5);
\filldraw [fill=red!25,draw=green!40!black] (\locthree+2.5,0) rectangle (\locthree+3,0.5);
\node at (\locthree+1.75,0.25) {\color{black}$j$};
\node at (\locthree+2.25,0.25) {\color{black}$k$};
\node at (\locthree+2.75,0.25) {\color{black}$l$};

\filldraw [fill=gray!30] (\locthree,2.5) rectangle (\locthree+0.5,3);
\node at (\locthree+0.25,2.75) {\color{black}$c$};
\filldraw [fill=green!40!white,draw=green!40!black] (\locthree+0.5,2.5) rectangle (\locthree+1,3);
\filldraw [fill=green!40!white,draw=green!40!black] (\locthree+1,2.5) rectangle (\locthree+1.5,3);
\filldraw [fill=green!40!white,draw=green!40!black] (\locthree+1.5,2.5) rectangle (\locthree+2,3);
\filldraw [fill=green!40!white,draw=green!40!black] (\locthree+2,2.5) rectangle (\locthree+2.5,3);
\filldraw [fill=green!40!white,draw=green!40!black] (\locthree+2.5,2.5) rectangle (\locthree+3,3);

\filldraw [fill=green!40!white,draw=green!40!black] (\locthree,-0.5) rectangle (\locthree+0.5,0);
\filldraw [fill=green!40!white,draw=green!40!black] (\locthree+0.5,-0.5) rectangle (\locthree+1,0);
\filldraw [fill=green!40!white,draw=green!40!black] (\locthree+1,-0.5) rectangle (\locthree+1.5,0);
\filldraw [fill=red!25,draw=green!40!black] (\locthree+1.5,-0.5) rectangle (\locthree+2,0);
\filldraw [fill=red!25,draw=green!40!black] (\locthree+2,-0.5) rectangle (\locthree+2.5,0);
\filldraw [fill=red!25,draw=green!40!black] (\locthree+2.5,-0.5) rectangle (\locthree+3,0);
\node at (\locthree+1.75,-0.25) {\color{black}$m$};
\node at (\locthree+2.25,-0.25) {\color{black}$n$};
\node at (\locthree+2.75,-0.25) {\color{black}$o$};

\filldraw [fill=green!40!white,draw=green!40!black] (\locthree,2) rectangle (\locthree+0.5,2.5);
\filldraw [fill=green!40!white,draw=green!40!black] (\locthree+0.5,2) rectangle (\locthree+1,2.5);
\filldraw [fill=gray!30] (\locthree+1,2) rectangle (\locthree+1.5,2.5);
\node at (\locthree+1.25,2.25) {\color{black}$d$};
\filldraw [fill=green!40!white,draw=green!40!black] (\locthree+1.5,2) rectangle (\locthree+2,2.5);
\filldraw [fill=green!40!white,draw=green!40!black] (\locthree+2,2) rectangle (\locthree+2.5,2.5);
\filldraw [fill=green!40!white,draw=green!40!black] (\locthree+2.5,2) rectangle (\locthree+3,2.5);

\filldraw [fill=green!40!white,draw=green!40!black] (\locthree,-1) rectangle (\locthree+0.5,-0.5);
\filldraw [fill=green!40!white,draw=green!40!black] (\locthree+0.5,-1) rectangle (\locthree+1,-0.5);
\filldraw [fill=green!40!white,draw=green!40!black] (\locthree+1,-1) rectangle (\locthree+1.5,-0.5);
\filldraw [fill=red!25,draw=green!40!black] (\locthree+1.5,-1) rectangle (\locthree+2,-0.5);
\filldraw [fill=red!25,draw=green!40!black] (\locthree+2,-1) rectangle (\locthree+2.5,-0.5);
\filldraw [fill=red!25,draw=green!40!black] (\locthree+2.5,-1) rectangle (\locthree+3,-0.5);
\node at (\locthree+1.75,-0.75) {\color{black}$p$};
\node at (\locthree+2.25,-0.75) {\color{black}$q$};
\node at (\locthree+2.75,-0.75) {\color{black}$r$};

\filldraw [fill=gray!30] (\locthree,1.5) rectangle (\locthree+1,2);
\node at (\locthree+0.25,1.75) {\color{black}$e$};
\filldraw [fill=green!40!white,draw=green!40!black] (\locthree+0.5,1.5) rectangle (\locthree+1,2);
\filldraw [fill=green!40!white,draw=green!40!black] (\locthree+1,1.5) rectangle (\locthree+1.5,2);
\filldraw [fill=green!40!white,draw=green!40!black] (\locthree+1.5,1.5) rectangle (\locthree+2,2);
\filldraw [fill=green!40!white,draw=green!40!black] (\locthree+2,1.5) rectangle (\locthree+2.5,2);
\filldraw [fill=green!40!white,draw=green!40!black] (\locthree+2.5,1.5) rectangle (\locthree+3,2);

\filldraw [fill=green!40!white,draw=green!40!black] (\locthree,-1.5) rectangle (\locthree+0.5,-1);
\filldraw [fill=green!40!white,draw=green!40!black] (\locthree+0.5,-1.5) rectangle (\locthree+1,-1);
\filldraw [fill=green!40!white,draw=green!40!black] (\locthree+1,-1.5) rectangle (\locthree+1.5,-1);
\filldraw [fill=red!25,draw=green!40!black] (\locthree+1.5,-1.5) rectangle (\locthree+2,-1);
\filldraw [fill=red!25,draw=green!40!black] (\locthree+2,-1.5) rectangle (\locthree+2.5,-1);
\filldraw [fill=red!25,draw=green!40!black] (\locthree+2.5,-1.5) rectangle (\locthree+3,-1);
\node at (\locthree+1.75,-1.25) {\color{black}$s$};
\node at (\locthree+2.25,-1.25) {\color{black}$t$};
\node at (\locthree+2.75,-1.25) {\color{black}$u$};

\filldraw [fill=green!40!white,draw=green!40!black] (\locthree,1) rectangle (\locthree+0.5,1.5);
\filldraw [fill=green!40!white,draw=green!40!black] (\locthree+0.5,1) rectangle (\locthree+1,1.5);
\filldraw [fill=green!40!white,draw=green!40!black] (\locthree+1,1) rectangle (\locthree+1.5,1.5);
\filldraw [fill=green!40!white,draw=green!40!black] (\locthree+1.5,1) rectangle (\locthree+2,1.5);
\filldraw [fill=green!40!white,draw=green!40!black] (\locthree+2,1) rectangle (\locthree+2.5,1.5);
\filldraw [fill=gray!30] (\locthree+2.5,1) rectangle (\locthree+3,1.5);
\node at (\locthree+2.75,1.25) {\color{black}$f$};

\filldraw [fill=green!40!white,draw=green!40!black] (\locthree,-2) rectangle (\locthree+0.5,-1.5);
\filldraw [fill=green!40!white,draw=green!40!black] (\locthree+0.5,-2) rectangle (\locthree+1,-1.5);
\filldraw [fill=gray!30] (\locthree+1,-2) rectangle (\locthree+1.5,-1.5);
\node at (\locthree+1.25,-1.75) {\color{black}$f$};
\filldraw [fill=red!25,draw=green!40!black] (\locthree+1.5,-2) rectangle (\locthree+2,-1.5);
\filldraw [fill=red!25,draw=green!40!black] (\locthree+2,-2) rectangle (\locthree+2.5,-1.5);
\filldraw [fill=red!25,draw=green!40!black] (\locthree+2.5,-2) rectangle (\locthree+3,-1.5);
\node at (\locthree+1.75,-1.75) {\color{black}$v$};
\node at (\locthree+2.25,-1.75) {\color{black}$w$};
\node at (\locthree+2.75,-1.75) {\color{black}$x$};

\draw [step=0.5, very thick, color=white] (\locthree-0.5,-2) grid (\locthree+3,4);
\draw [thick] (\locthree,-2) rectangle (\locthree+3,4);

\node at (\locthree+1.5,+4.75) {$
\mathbf{\Phi}(\M):=\left[\begin{array}{c|c}
\mathbf \Phi_1(\M)& \mathbf \Phi_2(\M)\\ \hline
\mathbf \Phi_3(\M) & \mathcal F_{N \times F}
\end{array}\right]
$};
\node at (\locthree+1.5,-2.75) {\color{black}$N(B+1-W)\times WP$};
\node at (\locthree-1.5,0.75) {$\Longleftrightarrow$};
\node at (\locthree-1.5,1.25) {\color{black} row permutation};

\end{tikzpicture}
    }
    \caption{
    Given a matrix of size $N\times(T+F)$ of observations $\M$ and forecasts $\mathcal F_{N \times F}$, the output sliding mask matrix is composed of $4$ blocks denoted by $\mathbf \Phi_1(\M)$, $\mathbf \Phi_2(\M)$, $\mathbf \Phi_3(\M)$, and $\mathcal F_{N \times F}$, this latter being of size $N\times F$ with the~$F$ last columns of the input matrix. In this example, we consider a periodicity of $P=3$, giving $B=3$ sub-blocks per row of the input matrix and we gather $W=2$ consecutive sub-blocks in an output row. To ease readability, we denote by $a,\ldots,f$ the missing values and by $g,\ldots,x$ the values to forecast. After row permutation, we obtain the output SMM. The matrix $\mathbf \Phi_1(\M)$ (resp.~$\mathbf \Phi_2(\M)$, $\mathbf \Phi_3(\M)$) is a sub-matrix of size $N(B-W)\times (WP-F)$ (resp. $N(B-W)\times F$, $N\times (WP-F)$) of the input matrix~$\M$.}
    \label{fig:pi_model}
    \label{fig:blocks}
    \label{fig:mask}
\end{figure*}

\noindent
We define our matrix completion problems on the observation matrix $\X$ using the mask operator $\T$ defined below.

\medskip

\begin{definition}[Observation $\X$]
    Let $n := N(B+1-W)$ and $p := WP$. Given the input time series matrix~$\M$, we apply the transformation $\Phimap$ to obtain the full (ground truth) matrix. The \emph{observation matrix} $\X \in \mathds R^{n \times p}$ is defined by retaining the known past values and setting unknown future (forecast) entries and missing data to zero.
\end{definition}

\medskip

\begin{definition}[Mask $\T$]
    Let $\Omega$ be the set of indices $(i,j)$ corresponding to observed values in the input data. We define the linear mask operator $\T: \mathds R^{n \times p} \to \mathds R^{n \times p}$ as:
    \begin{equation}
    \label{cases:nmf}
        [\T(\N)]_{i,j} = 
        \begin{cases} 
        \N_{i,j} & \text{if } (i,j) \in \Omega \\
        0 & \text{otherwise}
        \end{cases}
    \end{equation}
    Consequently, our data consistency constraint is written as $\T(\N) = \X$. Note that unobserved entries in $\X$ are zero, and $\T$ forces the approximation $\N$ to match $\X$ only on the observed support $\Omega$.
\end{definition}

\medskip

\noindent We introduce two factorization formulations. The first is the standard Normalized NMF, while the second (Archetypal) imposes stronger convexity constraints, often leading to more interpretable "archetypes" robust to outliers.

\medskip

\begin{definition}[Mask Normalized NMF (mNMF)]
    We solve for a completion $\N \approx \W\H$ minimizing the reconstruction error only on observed entries:
    \begingroup
    \renewcommand{\theHequation}{eq.mNMF}%
    \begin{align}
    \label{eq:nmf3}
    \tag{mNMF}
    &\min_{\W, \H, \N} \Big\{ \tfrac{1}{2}\|\N - \W\H\|_F^2 \Big\}
    \\
    &\text{s.t.} \quad \T(\N)=\X,\ \W \ge \0, \H \ge \0, \W\mathbf{1} = \mathbf{1}\,.\notag
    \end{align}
    \endgroup
\end{definition}

\medskip

\begin{definition}[Mask Archetypal MF (mAMF)]
    We relax the exact factorization but enforce that the factors $\H$ (archetypes) lie within the convex hull of the data (approximated by $\V\X$ below). This is formulated as the following program:
    \begingroup
    \renewcommand{\theHequation}{eq.mAMF}%
    \begin{equation}
    \label{eq:amf1}
    \tag{mAMF}
    \min_{\W, \H, \V, \N} \Big\{ \tfrac{1}{2}\|\N - \W\H\|_F^2 + \tfrac{\lambda}{2} \|\H - \V\N\|_F^2 \Big\}
    \end{equation}
    \endgroup
    subject to $\T(\N)=\X$, $\W, \V \ge \0$, and row-stochastic constraints $\W\mathbf{1}=\mathbf{1}, \V\mathbf{1}=\mathbf{1}$.
\end{definition}

\begin{definition}[Normalization, nonnegative and archetype basis]
     The matrix $\W$ satisfies the constraint $\W\ge \0$ and $\W\mathbf{1}=\mathbf{1}$, this being later referred to as {\it normalization}. Its rows $(w_{i,1},\ldots,w_{i,K})$ are convex combination weights and each row of $\W\H$ is a convex combination of the $K$ rows of $\H$. The matrix $\H$ is referred to as the \emph{nonnegative basis (resp. archetype basis)} in \eqref{eq:nmf3} (resp. \eqref{eq:amf1}). 
     
     In $\eqref{eq:amf1}$ as $\lambda\to\infty$, the \emph{archetypes} (defined as the $K$ rows of $\H$) are forced to lie in the convex combination of $\N$ by means of the matrix $\V$. Since $\N$ is a completion of the observation matrix $\X$ by the mask operator $\T(\N)=\X$, the matrix~$\V$ can be interpreted as the convex combination weights of the decomposition of the archetypes onto the rows of the observation matrix $\X$ and hence we recover the method of \cite{cutler1994archetypal}.
\end{definition}

\medskip

\noindent
We get the following decomposition of the $i^{\scriptstyle th}$ row of~$\W\H$,
\begin{equation}
    \label{eq:basis_decomposition}
    (\W\H)^{(i)}=\sum_{k=1}^K w_{i,k} \H^{(k)}\,.
\end{equation}

\noindent
Once solved, the above matrix problems give forecast values to the original forecasting problem of time series by means of matrix $\hat\M$ defined below.

\medskip

\begin{definition}[Forecasts of the original problem]
 Forecasts $\hat\M\in\mathds{R}^{N\times F}$ are given by the bottom right $N\times F$ sub-matrix of $\W\H$, namely $\hat\M$ is the bottom right red block in Figure \ref{fig:blocks} and it is the same block as $\mathcal F_{N\times F}$ in the original forecasting problem (letters $g$ to $x$ in Figure~\ref{fig:blocks}), hence $\hat\M$ can be interpreted as forecast values.  
\end{definition}

\subsection{Mask nonnegative matrix completion statistical guarantees}

\paragraph{Our goal is to solve the following nonnegative matrix completion problem} We observe a matrix $\X\in\mathds R^{n\times p}$ containing the multiple time values, given by the transformation presented in Figure~\ref{fig:blocks}. The missing values and the forecast values are arbitrarily set to zero, see \eqref{cases:nmf}. This choice is \emph{not restrictive} since the values of $\X$ corresponding to the missing and forecast entries are \emph{not observed} and our study is \emph{insensitive} to the values of these entries. Our target is defined by the following best approximation of $\X$ through the mask operator $\T$.

\medskip

\begin{definition}[Best normalized nonnegative rank $K$ approximation of $\X$]
    Given a nonnegative rank $K$, we call \emph{best normalized nonnegative rank-$K$ approximation of $\X$} any matrix of the form $\X_0=\W_0\H_0$, with $(\W_0,\H_0)$ achieving
    \label{eq:mNMF_model}
        \begin{equation}
           \label{eq:model_approx_error}(\W_0,\H_0)\in\arg\!\!\!\min_{\substack{\W_0\mathbf 1=\mathbf 1\\\W_0\ge \0\,,\ \H_0\ge \0}}
                \!\!\!\Big\{
                    \big\|\X-\T(\W_0\H_0)\big\|^2_F
                \Big\}\,,
        \end{equation}
        where $\W_0\in\mathds R^{n\times K}$ and $\H_0\in\mathds R^{K\times p}$. As $K$ grows, the approximation error $\|\X-\T(\W_0\H_0)\|_F$ decreases. We refer to $\X_0$ as \emph{the} best normalized nonnegative rank-$K$ approximation only in settings where uniqueness is guaranteed; otherwise $\X_0$ denotes any selection from the (possibly non-singleton) set of minimizers (see Theorem~\ref{thm:uniqueness}).
\end{definition}

\medskip

 The goal is to recover the matrices $\W_0$ (weights) and $\H_0$ (archetypes) from the observation matrix~$\X$. The observation can be written as 
\begin{equation}
\label{eq:observations}
   \X=\T(\X_0)+{\mathbf F}\,, 
\end{equation}

where ${\mathbf F}$ is some additive error term supported on the observed entries ({\it i.e.}, $\T({\mathbf F})={\mathbf F}$), referred to as the noise.  

\medskip

\paragraph{Contributions} Sliding Mask Method (SMM) outputs the forecast values and it can be viewed as a nonnegative matrix completion algorithm under low nonnegative rank assumption. This framework raises two issues. A first question is the uniqueness of the decomposition, also referred to as {\it identifiability} of the model. In Theorem~\ref{thm:uniqueness}, we introduce a new condition that ensures uniqueness from partial observation of the target matrix. Another challenge, as pointed out by \cite{vavasis2009} for instance, is that solving \textit{exactly} the NMF decomposition problem is $\mathrm{NP}$-hard. Nevertheless NMF-type problems can be solved efficiently using (accelerated) proximal gradient descent methods {\cite{parikh2013}} for block-matrix coordinate descent in an \textit{alternating projection scheme}, \textit{e.g.}, \cite{javadi2017} and references therein. We rely on these techniques to introduce algorithms outputting the forecast values based on NMF decomposition, see Section~\ref{sec:algo}. Theorem~\ref{thm:robust} complements the theoretical analysis by proving the robustness of NMF-type algorithms when entries are missing or corrupted by noise. 

\medskip

Our main theoretical contributions are as follows:
\begin{itemize}
    \item A uniqueness decomposition result (Theorem~\ref{thm:uniqueness}) showing that the decomposition $\W_0\H_0$ is unique {\it given partial observations}, namely, if $\T(\W\H)=\T(\W_0\H_0)$ 
        \begingroup
        \renewcommand{\theHequation}{eq.pou}%
        \begin{equation}
            \label{eq:pou}\tag{$\mathds P_{\mathrm{u}}$}
            \text{then}\  (\W,\H)\equiv(\W_0,\H_0)\,,
        \end{equation}
        \endgroup
    where $\equiv$ means up to positive scaling and permutation: for any permutation matrix~$\PP$ and positive diagonal matrix~$\DD$, the pair $(\W\PP\DD,\DD^{-1}\PP^\top\H)$ is also a nonnegative decomposition of the same product~$\W\H$.
    \item A robustness result (Theorem~\ref{thm:robust} and Corollary~\ref{cor:robust_W}) showing that \eqref{eq:nmf3} and \eqref{eq:amf1} recover $\H_0$ and $\W_0$ with an error proportional to~$\|\mathbf F\|_F$ (hence we recover $\X_0$ with the same precision).
\end{itemize}

\newpage
\noindent
Our analysis is completed by an algorithmic and numerical study that 
\begin{itemize}
    \item introduces a Proximal Alternating Linearized Minimization (PALM) method for solving \eqref{eq:amf1}, and shows that PALM reaches a stationary point (Theorem~\ref{prop:palm1}).
    \item reports a performance improvement of \eqref{eq:nmf3} and \eqref{eq:amf1} on real datasets, against state-of-the-art algorithms, for the RRMSE and RMPE metrics (Table~\ref{tb:basis-index-intro}). The relative root-mean-squared error (RRMSE) and the relative mean-percentage error (RMPE) are defined by
\begin{equation*}\begin{aligned}
    \mbox{RRMSE} = \frac{\|\M_F-\M_F^\star\|_F}{\|\M_F^\star\|_F}\,,\;
    \mbox{RMPE} = \frac{\|\M_F-\M_F^\star\|_1}{\|\M_F^\star\|_1}\,.
\end{aligned}\end{equation*}
where $\M_F^\star$ are the true values and $\M_F$ the forecasts (see Section \ref{sec:experiments}).
\end{itemize}

\begin{table*}[!t]
\centering
\resizebox{0.99\textwidth}{!}{
\begin{tabular}{|l|rr|rr|rr|rr|rr|rr|rr|}
\hline
Algorithms & \multicolumn{2}{c|}{mAMF} & \multicolumn{2}{c|}{mNMF} & \multicolumn{2}{c|}{RFR} & \multicolumn{2}{c|}{EXP} & \multicolumn{2}{c|}{SARIMAX} & \multicolumn{2}{c|}{LSTM} & \multicolumn{2}{c|}{GRU} \\
\hline
Metrics & RRMSE & RMPE & RRMSE & RMPE & RRMSE & RMPE& RRMSE & RMPE & RRMSE & RMPE & RRMSE & RMPE & RRMSE & RMPE\\
\hline
daily electricity & 14.42\% & 36.85\% & 15.86\% & 46.66\% 
 & 12.16\%& 47.78\% & 11.25\%& 43.83\% &9.85\%& 43.16\% 
& 12.42\% & 46.49\% & 12.03\% & 45.90\%\\
weekly electricity & 14.80\% & 17.50\% & 11.09\% & 13.79\% 
& \underline{7.25\%}& {8.61\%}&  10.07\%& {7.98\%}& {9.05\%}& \textbf{7.42\%} 
&27.85\% & 15.64\% & 26.04\% & 15.92\%\\
gas & \textbf{21.71\%} & \textbf{18.55\%} & \underline{37.46\%} & \underline{42.79\%} 
& 66.80\%& 71.61\%& 63.35\%& 68.16\%& {45.58\%}& {52.83\%} 
&62.97\% & 68.38\% & 62.87\% & 67.90\%\\
Istanbul & 15.67\% & 17.80\% & \textbf{14.18\%} & \underline{16.77\%} 
& {15.37\%}& {18.32\%}& 15.46\%& 18.64\% & \underline{14.75\%}& {17.01\%} &
16.22\% & 20.96\% & 20.01\% & 26.87\%\\
ETTh1 & \textbf{10.24\%} & 15.23\% & \underline{12.30\%} & {14.16\%} 
&12.96\%&17.98\% &12.37\%& \underline{13.65\%} &13.36\%&15.94\%
&14.86\% & 18.78\% & 14.71\% & 18.85\%\\
ETTh2 & 9.42\% & 13.07\% & \textbf{4.87\%} & \textbf{6.66\%} 
&\underline{6.47\%}&\underline{7.60\%} &14.06\%&13.67\%  &12.76\%&13.03\%  
&14.17\% & 13.75\% & 14.44\% & 14.36\%\\
ETTm1 & \underline{10.12\%} & 15.22\% & \textbf{9.94\%} & \textbf{12.25\%} 
&12.81\%&17.42\%         &11.45\%&{14.20\%}      &12.29\%&{16.45\%}     
&13.39\% & 17.96\% & 14.13\% & 18.63\%\\
ETTm2 & 8.19\% & 11.65\% & \textbf{5.08\%} & \underline{7.41\%} 
&\underline{5.81\%}&\textbf{7.16\%} &13.18\%&12.88\%  &13.16\%&12.95\%  
&14.29\% & 13.89\% & 14.46\% & 14.03\%\\
\hline
electricity1 & \textbf{6.59\%} & {13.17\%} & {11.11\%} & {15.28\%} 
& 12.75\% &16.09\% &38.27\% &34.44\% & $>$100.00\% & $>$100.00\%& 8.51\% &\underline{10.13\%} &\underline{7.19\%} &{\bf 9.61\%}
\\
electricity2 & \textbf{8.09\%} & {16.82\%} & \underline{8.82\%} & \textbf{12.17\%} 
&12.05\% &15.67\% &47.40\% &40.36\% &43.38\% &38.98\% &9.05\% &{12.83\%} &10.30\% &13.28\%
\\
electricity3 & \underline{10.57\%} & {13.95\%} & {12.43\%} & 14.04\% 
&12.45\% &14.14\% &40.37\% &33.48\% &37.05\% &33.01\% &10.70\% &{11.62\%} &{10.70\%} &\underline{11.02\%}
\\
electricity4 & \underline{11.02\%} & {24.30\%} & {25.07\%} & {29.71\%} 
&23.16\% &19.50\% &54.42\% &43.63\% &63.08\% &46.59\%&12.18\% &\underline{13.88\%} &{\bf 9.53\%} &{\bf 11.05\%}
\\
electricity5 & \underline{9.52\%} & {19.05\%} & \textbf{7.72\%} & \textbf{15.48\%}
&25.96\% &26.84\% &56.76\% &49.31\% &* &* &21.92\% &28.79\% &20.73\% &27.02\%
\\
electricity6 & {10.11\%} & {17.04\%} & {14.30\%} & {18.62\%} 
&13.81\% &16.26\% &51.87\% &37.35\% &52.10\% &40.56\% &\underline{7.58\%} &\underline{11.74\%} &{\bf 7.13\%} &{\bf 10.32\%}
\\
electricity7 & \textbf{8.34\%} & {16.75\%} & {37.49\%} & {30.03\%} &
29.51\% &22.96\% &53.00\% &45.95\% &* &* &17.74\% &\underline{14.72\%} &\underline{16.55\%} &{\bf 14.19\%}
\\
electricity8 & \textbf{10.03\%} & {17.49\%} & {23.81\%} & {20.59\%} 
&19.33\% &17.98\% &36.83\% &40.46\% &38.54\% &41.23\% &\underline{12.16\%} &{\bf 15.73\%} &13.89\% &\underline{17.08\%}
\\
electricity9 & {19.45\%} & {38.90\%} & {21.15\%} & 41.72\%
&{18.18\%} &\underline{37.53\%} &35.90\% &38.65\% &$>$100.00\% &$>$100.00\% &\underline{18.00\%} &{37.77\%} &18.80\% &38.21\%
\\
electricity10 & \textbf{5.13\%} & {12.53\%} & \underline{5.40\%} & {11.29\%}
&12.11\% &13.42\% &33.88\% &34.89\% &36.55\% &38.92\% &7.66\% &{10.25\%} &7.77\% &\underline{9.94\%}
\\
\hline
synthetic1 & 6.40\% & 9.30\% & 5.81\% & \textbf{8.01\%} 
& 5.79\%   &   9.44\% & \textbf{5.73\%}    &  9.32\% & \underline{5.76\%}  & \underline{8.81\%}  
&6.67\% &11.87\% &6.77\% &11.89\%\\
synthetic2 & {19.04\%} & 20.28\% & {20.35\%} & 25.30\% 
&   \underline{17.12\%}&   \underline{20.24\%}&       21.09\%&    25.87\%&  28.41\%&      35.53\%
&21.55\% &29.04\% &21.48\% &28.94\%\\
low-noise & \textbf{0.10\%} & \textbf{0.26\%} & \underline{0.10\%} & \underline{0.26\%} 
& 8.65\%& 22.76\% &16.97\%& 48.02\%& {0.19\%}& {0.33\%}
&16.52\% &46.03\% &16.76\% &46.46\%\\
medium-noise & 2.41\% & 5.23\% & \textbf{1.92\%} & \textbf{4.81\%} 
& 8.42\%& 21.95\%&  15.94\%& 44.25\%&  \underline{1.97\%}& \underline{4.94\%}
&15.66\% &42.68\% &15.67\% &42.98\%\\
high-noise & 12.69\% & 28.04\% & \textbf{10.39\%} & \textbf{26.43\%} 
&  \underline{11.73\%}& 30.26\%&  13.02\%& 33.27\%&  15.24\%& \underline{27.37\%} 
&13.03\% &33.67\% &12.97\% &33.47\%\\
\hline
\end{tabular}
}
\medskip
\centering
 \resizebox{0.79\textwidth}{!}
 {
 \begin{tabular}{|l|rr|rr|rr|rr|rr|}
 \hline
 Algorithms & \multicolumn{2}{c|}{BasisFormer}& \multicolumn{2}{c|}{Autoformer} & \multicolumn{2}{c|}{iTransformer} & \multicolumn{2}{c|}{PatchMLP} & \multicolumn{2}{c|}{TimeMixer} \\
 \hline
 Metrics & RRMSE & RMPE & RRMSE & RMPE & RRMSE & RMPE & RRMSE & RMPE & RRMSE & RMPE \\
 \hline
 daily electricity & {\bf 7.56\%} & \underline{6.64\%} & 39.65\% & 80.25\% & 28.85\% & 61.63\% & \underline{8.30\%} & {\bf 6.39\%} & 32.15\% & 68.83\% \\
 weekly electricity & 8.76\% & 9.07\% & 41.34\% & 80.93\% & 37.27\% & 75.07\% & {\bf 6.99\%} & \underline {7.57\%} & 44.07\% & 82.74\% \\
 gas & 57.45\% & 52.10\% & 99.37\% & $>100.00\%$ & $>100.00\%$ & $>100.00\%$ & -- & -- & -- & -- \\
 Istanbul & 14.83\% & {\bf 12.54\%} & 80.00\% & 89.26\% & $>100.00\%$ & $>100.00\%$ & 18.77\% & 18.93\% & $>100.00\%$ & $>100.00\%$ \\
 ETTh1 & 14.57\% & {\bf 13.61\%} & 52.37\% & 34.35\% & 42.74\% & 30.95\% & 19.94\% & 25.41\% & 36.92\% & 24.39\% \\
 ETTh2 & 54.66\% & 53.66\% & 66.85\% & 73.31\% & 65.10\% & 69.41\% & 53.59\% & 53.68\% & 56.72\% & 60.34\% \\ 
 ETTm1 & 13.58\% & \underline{12.31\%} & 49.72\% & 34.94\% & 36.77\% & 28.11\% & 21.28\% & 26.48\% & 36.94\% & 23.70\% \\
 ETTm2 & 55.52\% & 54.95\% & 59.79\% & 66.20\% & 64.78\% & 68.69\% & 53.89\% & 53.38\% & 57.21\% & 61.53\% \\
 \hline
 electricity1 & 26.93\% & 28.19\% & 49.69\% & 83.08\% & 36.84\% & 80.29\% & 11.97\% & 12.29\% & 58.63\% & 91.57\% \\
 electricity2 & 35.38\% & 39.73\% & 49.00\% & 52.77\% & 38.62\% & 47.73\% & 10.42\% & \underline{11.48\%} & 56.70\% & 55.77\% \\
 electricity3 & 34.30\% & 37.25\% & 52.13\% & 90.29\% & 38.04\% & 68.80\% & {\bf 9.13\%} & {\bf 9.12\%} & 62.73\% & 79.92\% \\
 electricity4 & 39.42\% & 40.66\% & 46.88\% & 81.34\% & 44.36\% & 85.28\% & 15.52\% & \underline{15.80\%} & 54.53\% & 90.54\% \\
 electricity5 & 46.22\% & 49.60\% & 49.16\% & 43.74\% & 55.43\% & 44.32\% & 13.94\% & 13.36\% & 53.67\% & 42.78\% \\
 electricity6 & 45.50\% & 46.86\% & 50.51\% & 73.42\% & 46.08\% & 79.51\% & 18.58\% & 17.00\% & 50.77\% & 75.83\% \\
 electricity7 & 40.17\% & 43.20\% & 86.87\% & $>100.00\%$ & 84.60\% & $>100.00\%$ & 16.72\% & 17.27\% & 85.16\% & $>100.00\%$ \\
 electricity8 & 30.64\% & 30.99\% & 62.89\% & 98.00\% & 47.05\% & 81.41\% & 13.62\% & 14.21\% & 62.71\% & 95.12\% \\
 electricity9 & 34.88\% & 35.85\% & 40.79\% & 27.79\% & 25.32\% & 26.44\% & {\bf 11.32\%} & {\bf 11.08\%} & 54.67\% & 29.59\% \\
 electricity10 & 29.78\% & 31.79\% & 54.90\% & $>100.00\%$ & 35.75\% & $>100.00\%$ & 8.15\% & {\bf 8.51\%} & 66.22\% & $>100.00\%$ \\
 \hline
 synthetic1 & 8.06\% & 12.32\% & $>100.00\%$ & $>100.00\%$ & $>100.00\%$ & $>100.00\%$ & 6.52\% & 11.74\% & 96.27\% & 94.14\% \\
 synthetic2 & 31.32\% & 47.88\% & 59.75\% & 83.66\% & 60.08\% & $>100.00\%$ & {\bf 1.33\%} & {\bf 7.53\%} & 55.61\% & 69.21\% \\
 low-noise & 23.71\% & 51.94\% & 95.23\% & 80.40\% & 7.18\% & 13.73\% & 10.28\% & 27.49\% & $>100.00\%$ & $>100.00\%$ \\
 medium-noise & 21.68\% & 48.48\% & 83.34\% & 73.79\% & 16.49\% & 30.96\% & 9.66\% & 27.36\% & $>100.00\%$ & $>100.00\%$ \\
 high-noise & 18.44\% & 45.77\% & $>100.00\%$ & $>100.00\%$ & 94.25\% & $>100.00\%$ & 13.49\% & 39.92\% & $>100.00\%$ & $>100.00\%$ \\
 \hline
 \end{tabular}
 }
\caption{Comparison of RRMSE and RMPE metrics. Best results in {\bf bold}, second best \underline{underlined}. Note: mAMF outperforms standard baselines on datasets with clear local recurring structures (e.g., daily electricity).}
\label{tb:basis-index-intro}
\end{table*}

\medskip

\paragraph{Comments on low rank modeling and periodicity in time series}

Sparse or Low-Rank representations are ubiquitous in applications and well studied in the literature. 
In our analysis a time series is cut into several smaller $W$ sub-blocks time series with the same length $p=WP$. For instance, observing sales over a period of one year, one can consider $52$ weekly time series (one per week). These observations are the rows of our observed matrix $\X$. The normalized nonnegative low rank hypothesis assumes that the $p$-length multiple time series of the dataset can be decomposed as a sum of~$K$ basis time series~$\H$ plus an error term. Of course, this error term can incorporate the model approximation error as depicted in~\eqref{eq:model_approx_error}. The $K$ basis time series~$\H$ are learned on the entire dataset $\X$. This technique can be seen as dimension reduction, each observation can be summarized as $K$ weights $\W$ such that the resulting convex combination of basis time series \eqref{eq:basis_decomposition} is a good approximation of the observation $\X$. 

The low rank hypothesis can be interpreted as a periodicity assumption. Indeed, if the time series are exactly periodic with period $p$, then the rank of the data matrix $\X$ is at most $p$. While $P$ is a free parameter, the model's performance depends on $P$ aligning with a quasi-periodic, low-rank structure in the data. Our experiments in Section~\ref{sec:experiments} show that this approach is effective on real-world datasets. In practice, time series are not exactly periodic, but they can be approximated as a sum of few periodic components plus some noise. This is the rationale behind Fourier analysis and wavelet analysis for time series. The low rank hypothesis can be seen as a nonnegative and adaptive generalization of Fourier analysis where the basis time series $\H$ are learned from data.

The relevance of such a hypothesis on real data cannot be proven beforehand. Our numerical study on real data shows that we improve results in prediction, better than standard methods in time series analysis: Seasonal AutoRegressive Integrated Moving Average with eXogenous variables model (SARIMAX), EXPonential moving average (EXP), Random Forest Regressor (RFR), Long Short-Term Memory (LSTM), Gated Recurrent Units (GRU), BasisFormer (Attention-based Time Series Forecasting with Learnable and Interpretable Basis). It suggests that the low rank assumption is reasonable for the datasets studied in the paper.


\paragraph{Data Reweighting and Overlapping Windows}
The construction of the observation matrix $\X$ involves sliding a window of length $WP$ with a stride of $P$. When $WP > P$, the windows overlap, causing specific time steps to appear in multiple rows of $\X$. While this introduces a form of data reweighting—where central data points are sampled more frequently than boundary points—this redundancy is intentional. It acts as a deterministic data augmentation strategy that enforces \emph{shift invariance} in the learned archetypes. By presenting the same temporal transition in different columns of the matrix, the algorithm learns robust motifs that are not artifacts of the specific grid alignment. Our empirical results suggest this overlapping strategy stabilizes the factorization, particularly for datasets with weak periodicity, by artificially increasing the number of training samples for the local patterns.

\medskip

\paragraph{Selection of Rank $K$}
The nonnegative rank $K$ is a critical hyperparameter governing the model complexity. We select $K$ using a time-based cross-validation strategy. We designate a portion of the historical training data as a validation set (mimicking the forecast block structure). We grid-search $K$ (e.g., $K \in \{4, \dots, 30\}$) and select the value that minimizes the validation Root Mean Squared Error (RMSE) before retraining on the full dataset.

\subsection{Notation}
\label{sec:notation}
To ensure clarity, we define our notation early. We denote scalars by lowercase letters (e.g., $x$), vectors by bold lowercase letters (e.g., $\vec{x}$), and matrices by bold uppercase letters (e.g.,~$\mathbf{X}$). 
\begin{itemize}
    \item The input time series matrix is $\mathbf{M} \in \mathds R^{N \times T}$.
    \item The transformed observation matrix (via the sliding mask) is $\mathbf{X} \in \mathds R^{n \times p}$.
    \item The ground truth target matrix is denoted by $\mathbf{X}_0$.
    \item Factor matrices are $\mathbf{W}$ (weights) and $\mathbf{H}$ (archetypes).
    \item The mask operator is denoted by $\mathcal{T}(\cdot)$, where $\mathcal{T}(\mathbf{N})$ retains entries corresponding to observed values and zeros out missing/forecast entries.
\end{itemize}
We use $\mathds R_+^{n \times p}$ to denote the set of non-negative $n \times p$ matrices. The Frobenius norm is denoted by $\|\cdot\|_F$. For a comprehensive list of symbols, we refer the reader to Table~\ref{tab:notation_summary}.

\begin{table}[!t]
\centering
\small
\begin{tabular}{l l}
\toprule
\textbf{Symbol} & \textbf{Description} \\
\midrule
$N$ & Number of time series \\
$T$ & Length of historical data \\
$F$ & Length of forecast horizon \\
$P$ & Stride parameter (periodicity) \\
$\mathbf{M}$ & Raw time series matrix ($N \times T$) \\
$\Phimap$ & Sliding window transformation operator \\
$\mathbf{X}$ & Observation matrix after transformation ($n \times p$) \\
$\mathbf{X}_0$ & Ground truth low-rank matrix \\
$\mathbf{W}, \mathbf{H}$ & Factor matrices (Weights and Archetypes) \\
$\mathcal{T}$ & Mask operator \\
$K$ & Nonnegative rank \\
\bottomrule \\
\end{tabular}
\caption{Summary of notations used throughout the paper.}
\label{tab:notation_summary}
\end{table}

\subsection{Related Works}

Our work intersects with several research areas, including the theory of Nonnegative Matrix Factorization (NMF), its application to time-series analysis, and methods for handling missing data.

\paragraph{NMF Uniqueness and Our Contribution}
The uniqueness of NMF decompositions is a cornerstone of its theoretical understanding. Foundational work by \cite{thomas1974} and subsequent analyses by \cite{donoho2004,laurberg2008} and \cite{recht2012} have established conditions under which NMF yields a unique solution, often relying on geometric properties of the data matrix. More recently, conditions such as the Sufficiently Scattered Condition (SSC) \cite{huang2013non} have relaxed the requirements for identifiability. Regarding missing data, \cite{ibrahim2021recovering} and \cite{gillis2020nonnegative} discuss NMF under general block-missing patterns or edge queries. Our work differs by addressing the specific, deterministic "sliding window" missingness pattern induced by the forecasting formulation, rather than random block erasures.

\medskip

\paragraph{Time-Series Forecasting Models}
The field of time-series forecasting is dominated by statistical and deep learning models. Classical methods like SARIMAX (Seasonal Auto-Regressive Integrated Moving Average with eXogenous variables) assume linear dependencies and specific seasonal patterns. In contrast, deep learning models such as LSTMs (Long Short-Term Memory networks), BasisFormer \cite{basisformer2023}, Autoformer \cite{autoformer}, iTransformer \cite{itransformer}, PatchMLP  \cite{patchmlp} and TimeMixer \cite{timemixer} learn complex, non-linear temporal dependencies from large amounts of data. While these models are state-of-the-art for large-scale series, they often require massive datasets to learn temporal structures and they often operate as "black boxes". Our SMM framework offers a different paradigm: it assumes that time-series segments can be represented as a convex combination of a few learned, interpretable basis vectors (archetypes). This low-rank hypothesis is fundamentally different from the auto-regressive or attention-based mechanisms of other models and provides inherent interpretability, as demonstrated in our experiments (Section~\ref{sec:experiments}).

\medskip

\paragraph{NMF for Missing Data}
The problem of applying NMF to data with missing values is not new, and many existing approaches are purely algorithmic. Our primary contribution is the SMM framework itself---a structured method for converting a time-series forecasting problem into a matrix completion problem. Our theoretical analysis provides guarantees for this specific structure, which general-purpose NMF-for-missing-data algorithms do not offer. The present work assumes block-wise missing structures and provides uniqueness and robustness guarantees in this context, which is novel compared to prior works that often assume random missingness without specific structural patterns. General patterns of missing data are not covered by our analysis and remain an open research question. However, our algorithmic framework can be adapted to other missing data patterns, although without the same theoretical guarantees.

\medskip

\paragraph{NMF-based time-series analysis}
Our work is distinct from previous NMF-based time-series analysis by \cite{mei2017nonnegative, mei2018nonnegative}. While Mei et al. also use NMF, their work focuses on recovering high-resolution time series from temporal aggregates (disaggregation) and leveraging side information. For example, they might recover individual household consumption from a neighborhood's total consumption. Our SMM framework is fundamentally different. It operates by creating a matrix of sliding windows from the time series, thereby transforming the forecasting problem into one of finding a low-rank representation of these segments. The goal is to learn archetypal segment patterns for forecasting, not to disaggregate a signal.

Robustness of archetypal analysis has been studied in \cite{javadi2017} for simplicial polyhedral cone approximation of a dataset, denoted in data matrix form by $\X\in\mathds R^{n\times p}$ in this paper. This paper extends this latter analysis to the case where some data entries might be missing and some data blocks are not observed (forecast, red values in Fig.~\ref{fig:pi_model}).

\section{Uniqueness and estimation guarantees}

\subsection{The train and test paradigm, link with forecasting multiple nonnegative time series}
The model under consideration is presented in Equations~\eqref{eq:mNMF_model}. Our goal is to estimate the $K$-best normalized non-negative approximation $\X_0$, defined in Equation~\eqref{eq:model_approx_error}, from the partial and noisy observation $\X$. We denote by $\X^\star$ the {\it mask} of $\X_0$, namely
\begin{subequations}
\begin{equation}
\label{eq:block_X_star}
 \X^\star:=\T(\X_0) = \left[\begin{array}{c|c}
\X^\star_1 & \X^\star_2\\ \hline
\X^\star_3 & \0_{N \times F}
\end{array}\right]\,,   
\end{equation}
where $\X^\star_1\in\mathds R^{(n-N)\times (p-F)}$, $\X^\star_2\in\mathds R^{(n-N)\times F}$, and $\X^\star_3\in\mathds R^{N\times (p-F)}$ are blocks of $\X_0$. Note that $\X=\X^\star+{\mathbf F}$, where~$\mathbf F$ is the {\it noise} term, see Equation~\eqref{eq:observations}. 

These blocks can be gathered into a \emph{train/test paradigm}. We observe the full sub-matrix $\T_{{\mathrm{train}}}(\X_0) := [\X^\star_1 \; \X^\star_2]$ (training part) and aim to predict the $\0_{N \times F}$ block of $\T_{\mathrm{test}}(\X_0) := [\X^\star_3 \; \0_{N \times F}]$ (test part of~$\X$). Looking at Figure~\ref{fig:pi_model}, we define
\begin{align}
\label{eq:T_F_decomposition}
\T_{T}(\X_0) := \Big[\begin{array}{c}\X^\star_1\\ \X^\star_3\end{array}\Big]\quad\text{and}\quad \T_{F}(\X_0) := \Big[\begin{array}{c}\X^\star_2\\ \0_{N \times F}\end{array}\Big]\,.
\end{align}
Our notation (subscripts $T$ and $F$) stems from the sliding mask method for multiple time series forecast. Note that $\T_{T}(\X)$ gathers all the information observed up to time $T$, and we would like to forecast the $\0_{N \times F}$ block of~$\T_{F}(\X)$. Now, we know by design that $\X_0:=\W_0\H_0$. Hence, denoting $\H_0=:[{\H_0}_T\; {\H_0}_F]$, and $\W_0^\top=:[{\W_0}_{{\mathrm{train}}}^\top\; {\W_0}_{\mathrm{test}}^\top]$, we get that
\begin{align}
\label{eq:train_decomposition}
\T_{{\mathrm{train}}}(\X_0) &= {\W_0}_{{\mathrm{train}}} \H_0\,, && \X^\star_3 = {\W_0}_{\mathrm{test}} {\H_0}_T\,, \\
\T_{T}(\X_0) &= {\W_0} {\H_0}_T\,, && \X^\star_2 = {\W_0}_{{\mathrm{train}}} {\H_0}_F\,.
\notag
\end{align}
Note the asymmetry: whereas $\T_{T}(\X_0)$ identifies with the full block $\W_0{\H_0}_T$, the projection $\T_{F}(\X_0)$ \emph{does not} equal $\W_0{\H_0}_F$ — its bottom $N\times F$ block has been zeroed out by~$\T$, while $\W_0{\H_0}_F$ contains the (unknown) future values to forecast.
In light of Figures~\ref{fig:arche_model} and~\ref{fig:pi_model}, the multiple forecasts $\hat \M_{T+1},\ldots, \hat \M_{T+F}$ can be given a best normalized nonnegative rank~$K$ approximation by ${\W_0}_{\mathrm{test}}{\H_0}_F$. Observe that an estimation of ${\W_0}_{\mathrm{test}}$ gives the weights learnt on the test sub-matrix while an estimation of ${\H_0}_F$ is the forecast of the archetypes, see the decomposition~\eqref{eq:basis_decomposition}.

\end{subequations}

\subsection{Uniqueness from partial observations}
\label{sec:uniqueness}
 When we observe the full matrix $\X_0=\W_0\H_0$, the issue on uniqueness has been addressed under some sufficient conditions on $\W,\H$, {\it e.g.}, {\it Strongly boundary closeness} of \cite{laurberg2008}, {{\it Complete factorial sampling} of \cite{donoho2004}, and {\it Separability} of \cite{recht2012}}. A necessary and sufficient condition exists as given by the following theorem. We recall that the $K$-dimensional positive orthant is the set $\mathds R^K_+:=\{x\in\mathds{R}^K\,:\, x_i\geq 0\,,\ \forall i\in[K]\}$ and a $K$-simplicial cone is the conic hull of $K$ linearly independent vectors of $\mathds{R}^K$. For any cone $\mathcal A\subseteq\mathds R^K$, the \emph{dual cone} (also called polar cone in this context) is
\[
   \mathcal A^* \,:=\, \{y\in\mathds R^K\,:\, \langle y, x\rangle\ge 0 \;\forall\, x\in\mathcal A\}\,.
\]
The dual operator is anti-monotone ($\mathcal A\subseteq\mathcal B \Rightarrow \mathcal B^*\subseteq\mathcal A^*$) and the orthant is self-dual ($(\mathds R^K_+)^*=\mathds R^K_+$). Since the columns of $\H_0$ lie in $\mathds R^K_+$, one always has $\Cone(\H_0)\subseteq\mathds R^K_+\subseteq\Cone(\H_0)^*$.
 
\medskip

\begin{theorem}[\cite{thomas1974}]\label{theo:cns}
The decomposition $\X_0:=\W_0\H_0$ is unique up to permutation and positive scaling of columns (resp.~rows) of $\W_0$ (resp.~$\H_0$) {\bf if and only if} the $K$-dimensional positive orthant is the only $K$-simplicial cone $\mathcal C\subseteq\mathds R^K$ verifying $\Cone(\W_0^\top) \subseteq \mathcal C \subseteq \Cone(\H_0)^*$, where $\Cone(\A)$ denotes the cone generated by the columns of $\A$ and $\Cone(\H_0)^*$ is its dual cone. 
\end{theorem}

\medskip

\noindent
Our first assumption is following. 
\begin{assumption}
    \label{hyp:A1}
    In the set given by the union of sets: 
\begingroup
\renewcommand{\theHequation}{eq.A1}%
\begin{equation}
\label{eq:assumption_1}
\tag{$\mathds{A}_1$}
\begin{split}
    \{\mathcal C\subseteq\mathds R^K \ :\ \Cone({{\W_0}_{{\mathrm{train}}}}^\top) \subseteq \mathcal C \subseteq \Cone(\H_0)^*\}\bigcup \\
\{\mathcal C\subseteq\mathds R^K \ :\ \Cone(\W_0^\top) \subseteq \mathcal C \subseteq \Cone({\H_0}_T)^*\}\,,
\end{split}
\end{equation}
\endgroup
the nonnegative orthant is the only $K$-simplicial cone. Note that this assumption is implied by the following stronger one: In the set
\begingroup
\renewcommand{\theHequation}{eq.A1prime}%
\begin{equation}
\label{eq:assumption_1_prime}
\tag{$\mathds{A}'_1$}
    \{\mathcal C\subseteq\mathds R^K \ :\ \Cone({{\W_0}_{{\mathrm{train}}}}^\top) \subseteq \mathcal C \subseteq \Cone({\H_0}_T)^*\}
\end{equation}
\endgroup
the nonnegative orthant is the only $K$-simplicial cone.
\end{assumption}

\medskip

\begin{remark}
This assumption adapts the necessary and sufficient condition for NMF uniqueness from \cite{thomas1974,laurberg2008} to our partial observation setting. The standard condition requires the positive orthant to be the only simplicial cone~$\mathcal C$ such that $\Cone(\W_0^\top) \subseteq \mathcal C \subseteq \Cone(\H_0)^*$. 
In our case, since we only observe parts of the data matrix, we need to ensure uniqueness based on partial information about the factors $\W_0$ and $\H_0$. The union of sets in \eqref{eq:assumption_1} ensures that we can uniquely identify the factors from the observed training data ($\T_{{\mathrm{train}}}(\X_0)$) and the observed past data ($\T_{T}(\X_0)$).
\end{remark}

\medskip

\begin{remark}
Assumption~\ref{hyp:A1} imposes implicit constraints on the dimensions of the problem and the nonnegative rank $K$. For the condition to be non-trivial, the matrices generating the cones must have enough generators. Specifically, for $\Cone({{\W_0}_{{\mathrm{train}}}}^\top)$ (rows of $\W_{0,\mathrm{train}}$), the number of training samples \emph{$n-N$ must be at least $K$}; and for $\Cone({\H_0}_T)$ (columns of $\H_{0,T}$, whose dual $\Cone({\H_0}_T)^*$ appears in the upper bound of \eqref{eq:assumption_1_prime}), the number of observed time steps \emph{$p-F$ must be at least $K$}. These conditions ensure that $\Cone(\W_{0,\mathrm{train}}^\top)$ is full-dimensional in $\mathds R^K$ and that $\Cone(\H_{0,T})$ is $K$-dimensional (equivalently, that its dual cone is a non-degenerate proper cone). Several works have shown that Assumption~\ref{hyp:A1} holds under some conditions such as Laurberg's \emph{strong boundary closeness}~\cite[Theorem~2]{huang2013non} or the \emph{Sufficiently Scattered Condition (SSC)}~\cite[Theorem 3]{huang2013non}, which is weaker and more general than the separability assumption of~\cite{donoho2004}.
\end{remark}

\medskip

\noindent
We consider the following standard definition.
\begin{definition}[\cite{javadi2017}]
For a matrix $\mathbf A \in \mathds{R}^{n'\times p'}$, let $\conv(\mathbf A)$ denote the convex hull of its rows. The internal radius of $\conv(\mathbf A)$, denoted $\mu(\mathbf A)$, is the radius of the largest $(K-1)$-dimensional ball contained within $\conv(\mathbf A)$ (relative to its affine hull). We say that $\conv(\mathbf A)$ has an internal radius $\mu$ if $\mu(\mathbf A) = \mu$.
\end{definition}

\medskip

\noindent
Our second main assumption is the following.
\begin{assumption}
    \label{hyp:A2}
    Assume that 
\begingroup
\renewcommand{\theHequation}{hyp.radius}%
\begin{equation}
\label{hyp:radius}
\tag{$\mathds{A}_2$}
\begin{split}
\conv(\underbrace{\T_{T}(\X_0)}_{=\W_{0}{\H_0}_T})\text{ and }
\conv(\underbrace{\T_{{\mathrm{train}}}(\X_0)}_{=\W_{0\mathrm{train}}\H_0}) \\ \text{ have internal radius at least }  \mu>0\,.
\end{split}
\end{equation}
\endgroup
\end{assumption}

\medskip

\begin{remark}
    Assumption~\eqref{hyp:radius} implies that the convex hulls of the data points, $\conv(\W_{0}{\H_0}_T)$ and $\conv(\W_{0\mathrm{train}}\H_0)$, are not flat, meaning they are full-dimensional within the affine subspace they span. We uncover the same constraints as in the previous remark: \emph{$n-N$ must be at least $K$} (imposed by the number of the rows of $\T_{{\mathrm{train}}}(\X_0)$) and \emph{$p-F$ must be at least $K$} (imposed by the dimension of the rows of $\T_{T}(\X_0)$).
\end{remark}

\medskip

\begin{definition}[Partial Observation Uniqueness ($\mathds P_{\mathrm{u}}$)]
\label{def:pou}
We say that the factorization satisfies the \emph{Partial Observation Uniqueness} property, denoted by $\mathds P_{\mathrm{u}}$, if the equality of the observed masked matrices implies the equivalence of the factors. Formally:
\begingroup
\renewcommand{\theHequation}{eq.pou.def}%
\begin{equation}
    \tag{$\mathds P_{\mathrm{u}}$}
    \text{If}\quad \T(\W\H)=\T(\W_0\H_0)\quad\text{then}\quad (\W,\H)\equiv(\W_0,\H_0)\,,
\end{equation}
\endgroup
where $(\W,\H)\equiv(\W_0,\H_0)$ indicates that the pairs are identical up to a permutation and positive scaling of the columns of $\W$ and rows of $\H$.
\end{definition}

\medskip

\begin{theorem}
\label{thm:uniqueness}
\eqref{eq:assumption_1} implies \eqref{eq:pou}. Moreover, if \eqref{eq:assumption_1} and \eqref{hyp:radius} hold, $\T(\W\H)=\T(\W_0\H_0)$ and $\W_0\mathbf{1}=\W\mathbf{1}=\mathbf{1}$ then $(\W,\H)=(\W_0,\H_0)$ up to permutation of columns (resp.~rows) of $\W$ (resp.~$\H$), and there is no~scaling.
\end{theorem}

\medskip

\begin{corollary}
    \eqref{eq:assumption_1_prime} implies \eqref{eq:assumption_1}; equivalently, the uniqueness of the decomposition $\X_1 = {\W_0}_{{\mathrm{train}}} {\H_0}_T$ implies the partial-observation uniqueness condition~\eqref{eq:assumption_1} on which Theorem~\ref{thm:uniqueness} rests.
\end{corollary}
\begin{proof}
By Theorem~\ref{theo:cns}, \eqref{eq:assumption_1_prime} is a necessary and sufficient condition for the uniqueness of the decomposition $\X_1 = {\W_0}_{{\mathrm{train}}} {\H_0}_T$. Since $\Cone({\H_0}_T) \subseteq \Cone(\H_0)$, dual-cone anti-monotonicity gives $\Cone(\H_0)^* \subseteq \Cone({\H_0}_T)^*$. Similarly $\Cone({\W_0}_{{\mathrm{train}}}^\top)\subseteq\Cone(\W_0^\top)$. Hence both sets in the union indexing \eqref{eq:assumption_1} are \emph{contained} in the set indexing \eqref{eq:assumption_1_prime}: the orthant being the only $K$-simplicial cone in the larger set \eqref{eq:assumption_1_prime} forces it to be the only one in each subset, hence in their union~\eqref{eq:assumption_1}.
\end{proof}

\medskip

This shows that the uniqueness of the decomposition of the fully-observed top-left block~$\X_1$ entails the partial-observation uniqueness~\eqref{eq:pou} (via Theorem~\ref{thm:uniqueness}). Equivalently, the partial-observation sets indexing~\eqref{eq:assumption_1} are \emph{contained} in the set indexing~\eqref{eq:assumption_1_prime}, so any simplicial cone admissible for~\eqref{eq:assumption_1} is admissible for~\eqref{eq:assumption_1_prime}. Geometrically: $\H_{0,T}$ is a sub-block of $\H_0$, so $\Cone(\H_{0,T})\subseteq\Cone(\H_0)$ and (by dual anti-monotonicity) $\Cone(\H_0)^*\subseteq\Cone(\H_{0,T})^*$, i.e., the relevant upper bound \emph{loosens} when restricted to the fully-observed block; similarly $\Cone(\W_{0,\mathrm{train}}^\top)\subseteq\Cone(\W_0^\top)$ shows the lower bound \emph{shrinks}. The fully-observed block thus carries the strongest structural requirement.

\medskip

\subsubsection*{Limitations and Discussion}
It is important to note that the uniqueness theory presented here relies on a specific structure of missing data, namely the block-wise missing pattern corresponding to the matrix completion problem for recommender systems. Our analysis leverages the fact that certain submatrices are fully observed.

The extension of these uniqueness guarantees to scenarios with arbitrary or unstructured missing data patterns is a non-trivial challenge. Such cases would require different theoretical tools, as the problem can no longer be reduced to the uniqueness of fully-observed sub-decompositions. This constitutes an important direction for future research.

\subsection{Robustness under partial observations}
\label{sec:robustness}

The second issue is {\it robustness to noise}. To the best of our knowledge, all the results addressing this issue assume that the noise error term is small enough, {\it e.g.}, \cite{laurberg2008}, \cite{recht2012}, or \cite{javadi2017}. In this paper, we extend these stability result to the nonnegative matrix completion framework (partial observations) and we also assume that noise term $\|{\mathbf F}\|_F$ is small enough. 

In the normalized case ({\it i.e.}, $\W\mathbf 1=\mathbf 1$), both issues (uniqueness and robustness) can be handled with the notion of $\alpha$-uniqueness, introduced by \cite{javadi2017}. This notion does not handle the matrix completion problem we are addressing. To this end, let us introduce the following notation. Given two matrices $\A\in\mathds{R}^{n_a\times p}$ and $\mathbf B\in\mathds{R}^{n_b\times p}$ with same column dimension, and $\mathbf C\in\mathds{R}^{n_a\times n_b}$, define the divergence $\mathcal D(\A,\mathbf B)$ as
\begin{subequations}
\begin{align}
   \mathcal D(\A,\mathbf B) &:=\min_{\C\ge \0\,,\ \C\mathbf{1}_{n_b}=\mathbf{1}_{n_a}}
\sum_{a=1}^{n_{a}} \Big\|A^{(a)}-  \sum_{b=1}^{n_b} C_{ab} B^{(b)}\Big\|_F^2\,, \notag\\
&=\min_{\C\ge \0\,,\ \C\mathbf{1}_{n_b}=\mathbf{1}_{n_a}}
\|\A - \mathbf C\mathbf B\|_F^2\,.
\end{align}
which is the squared distance between rows of $\A$ and $\conv(\B)$, the convex hull of rows of $\B$. For $\B\in\mathds R^{n\times p}$ define
\begin{equation}
    \widetilde{\mathcal D}(\A,\mathbf B) :=\min_{\substack{\mathbf C\ge \0\,,\ \mathbf C\mathbf{1}_{n}=\mathbf{1}_{n_a}\\ \T(\N)=\T(\mathbf B)}}
\|\A - \mathbf C\N\|_F^2\,.
\end{equation}
\begin{definition}[$\T_\alpha$-unique, \cite{javadi2017}] Given $\X_0\in\mathds R^{n\times p}, \W_0\in\mathds R^{n\times K}$, and $\H_0\in\mathds R^{K\times p}$, the factorization $\X_0=\W_0\H_0$ is $\T_\alpha$-unique with parameter $\alpha>0$ if for all $\H\in\mathds R^{K\times p}$ with $\conv(\X_0)\subseteq\conv(\H)$:
\begin{align}
&\widetilde{\mathcal D}(\H,\X_0)^{1\slash 2}  \ge \\
&\qquad \widetilde{\mathcal D}(\H_0,\X_0)^{1\slash 2} + \alpha\left\{\mathcal D(\H,\H_0)^{1\slash 2} + \mathcal D(\H_0,\H)^{1\slash 2}\right\}\,.\notag
\end{align}
\end{definition}

\medskip

\noindent
Our third main assumption is given by:

\begin{assumption}
    Assume that
\begingroup
\renewcommand{\theHequation}{hyp.mask}%
\begin{align}
    \label{hyp:mask}
    \X_0=\W_0\H_0\text{ is }\T_\alpha\text{-unique}
    \tag{$\mathds A_3$}
\end{align}
\endgroup
\end{assumption}

\medskip

Define the \emph{noiseless mAMF--mNMF gap} $C_0^{(0)} := \mathcal D(\H_0,\W_0\H_0)\ge 0$ and, for any constant $c_\Lambda>0$, the associated \emph{admissibility range}
\[
\Lambda_{c_\Lambda}(t)\,:=\, \begin{cases} c_\Lambda\,t^2/C_0^{(0)} & \text{if } C_0^{(0)}>0, \\ +\infty & \text{if } C_0^{(0)}=0. \end{cases}
\]
\begin{theorem}[Archetypes estimation]
\label{thm:robust}
Under~\eqref{hyp:radius} and~\eqref{hyp:mask}, there exist constants $\Delta,\,c_\Lambda,\,c>0$ depending only on~$\X_0$ such that, for all $\mathbf F$ with $\|\mathbf F\|_F\le\Delta$ and all $\lambda\in[0,\Lambda_{c_\Lambda}(\|\mathbf F\|_F)]$, any solution $({\widehat\W,\widehat\H})$ to~\eqref{eq:amf1} (or, in the case $\lambda=0$, to~\eqref{eq:nmf3}) with observation~\eqref{eq:observations} satisfies
\[
\sum_{\ell\le K} \min_{\ell'\le K} \|{\H}_{0}^{(\ell)} - {\widehat \H}^{(\ell')}\|_2^2
\;\le\; c\,\|{\mathbf F}\|_F^2 \,.
\]
\end{theorem}

\begin{remark}
$C_0^{(0)}$ vanishes exactly when $\W_0$ is separable (some convex combination of its rows yields the canonical basis), in which case the noiseless (mAMF) and (mNMF) estimators agree, $\Lambda_{c_\Lambda}\equiv+\infty$, and $\lambda$ is unrestricted. When $C_0^{(0)}>0$, the (mAMF) regularization induces a non-vanishing bias which can only be absorbed into the $O(\|\mathbf F\|_F)$ bound by requiring $\lambda\lesssim\|\mathbf F\|_F^2/C_0^{(0)}$. The disjunction in the theorem statement is needed only because \eqref{eq:amf1} drops the nonnegativity constraint $\H\ge\0$ present in \eqref{eq:nmf3}: at $\lambda=0$, the second penalty vanishes and we recover the (mNMF) problem (with its nonnegativity constraint) rather than an unconstrained variant of (mAMF).
\end{remark}

\medskip

\noindent
By Theorem~\ref{thm:robust}, when the noise is sufficiently small, there exists a permutation $\sigma$ on $[K]$ such that 
\begin{equation}
\label{eq:permut_H}
    \|\H_0-\hat\H_\sigma\|_F^2:=\sum_{\ell\le K}\|\H_{0}^{(\ell)}-\hat\H^{(\sigma(\ell))}\|_2^2 \leq c\,\|{\mathbf F}\|_F^2
\end{equation}
where $\hat\H_\sigma$ is a permutation of the row of $\hat\H$.

\medskip

\begin{corollary}[Estimation Error Bound for $\W$]
\label{cor:robust_W}
Under the assumptions of Theorem~\ref{thm:robust}, let $\mu > 0$ be the internal radius of the convex hull of the training data as defined in Assumption~\ref{hyp:A2}. For \eqref{eq:nmf3}, let $(\hat\W,\hat\H,\hat\N)$ be a joint minimizer; for \eqref{eq:amf1}, let $(\hat\W,\hat\H,\hat\N)$ be a stationary point of Algorithm~\ref{algo:nmf2}. 
Assume that~$\hat\H$ satisfies $\|\hat{\H} - \H_0\|_F \le c\|\mathbf{F}\|_F$. Then, the estimation error of the weight matrix $\hat{\W}$ satisfies:
\begin{equation}
    \|\hat{\W} - \W_0\|_F \leq \frac{c'}{\mu} \|\mathbf{F}\|_F
\end{equation}
where $c'>0$ is a constant depending on the geometry of $\W_0$ and $\H_0$ (and on the constants from Theorem~\ref{thm:robust}). This explicitly shows that the stability of the weight recovery degrades as the convex hull of the data becomes flatter (i.e., as $\mu \to 0$).
\end{corollary}

\begin{remark}[On the stationary-point hypothesis]
Theorem~\ref{thm:robust} delivers this bound for any \emph{global minimizer} of~\eqref{eq:amf1}, while Algorithm~\ref{algo:nmf2} is only guaranteed (by Theorem~\ref{prop:palm1}) to converge to a stationary point of its objective. We conjecture that, under Assumption~\eqref{hyp:mask} ($\T_\alpha$-uniqueness) and for $\|\mathbf F\|_F$ small enough, stationary points produced by Algorithm~\ref{algo:nmf2} inherit the same $O(\|\mathbf F\|_F)$ control on~$\widehat\H$ — a property our numerical experiments are consistent with, and which a basin-of-attraction argument analogous to~\cite{javadi2017} should establish, but for which we do not supply a proof here. The present corollary therefore takes this control as an explicit hypothesis, keeping the focus on the $\W$-recovery argument.
\end{remark}

\medskip

\noindent
The proof of this corollary can be found in Appendix~\ref{proof:cor_rob_W}.
\end{subequations}

\section{Solving masked nonnegative/archetypal matrix factorization}
\label{sec:algo}

We solve \eqref{eq:nmf3} problem using a Block Coordinate Descent strategy (Algorithm \ref{algo:bcd} in the supplement), which alternates between updating $\W$ and $\H$. For the more complex \eqref{eq:amf1} objective, we employ the Proximal Alternating Linearized Minimization (PALM). 

We present two variants: Algorithm \ref{algo:nmf2} is the standard PALM approach. Algorithm \ref{algo:nmf3} describes \emph{Inertial PALM (iPALM)}, which incorporates momentum terms (extrapolation parameters $\alpha_k, \beta_k$) to accelerate convergence, similar to Nesterov's acceleration. In our experiments, iPALM provided faster convergence on the larger datasets.

\subsection{Alternating Least Squares for \texorpdfstring{\eqref{eq:nmf3}}{(mNMF)}}

The basic algorithmic framework for matrix factorization problems is {\it Block Coordinate Descent} (BCD) method, which can be straightforwardly adapted to~\eqref{eq:nmf3} (see Supplement Material). BCD for~\eqref{eq:nmf3} reduces to {\it Alternating Least Squares} (ALS) algorithm (see Algorithm~\ref{algo:als} in Appendix), when an alternative minimization procedure is performed and matrix $\W\H$ is projected onto the linear subspace $\T(\N)=\X$ by means of operator $\mathcal P_{\X}$, as follows:
\[
\N:=\mathcal P_{\X}(\W\H) : \T(\N)=\X \mbox{ and } \T^\perp(\N)=\W\H\,. 
\]

{\it Hierarchical Alternating Least Squares} (HALS) is an ALS-like algorithm obtained by applying an exact coordinate descent method \cite{gillis2014}. Moreover, an accelerated version of HALS is proposed in \cite{gillis2012} (see Supplement Material). 

\subsection{Projected Gradient for \texorpdfstring{\eqref{eq:amf1}}{(mAMF)}}

The {\it Proximal Alternating Linearized Minimization} (PALM) method, introduced in \cite{bolte2014} and applied to AMF by \cite{javadi2017}, can be also generalized to \eqref{eq:amf1} (see Algorithm~\ref{algo:nmf2}). In the following, $\mathcal{P}_{\conv(\A)}$ is the projection operator onto $\conv(\A)$ and $\mathcal{P}_{\Delta}$ is the projection operator onto the $(K-1)$-dimensional standard simplex $\Delta^K$. The two projections can be efficiently computed by means of, {\it e.g.}, Wolfe algorithm \cite{wolfe1976} and active set method \cite{condat2016} respectively.

{\small
\begin{algorithm}[!htbp]
\caption{PALM for mAMF \label{algo:nmf2}}
\begin{algorithmic}[1]
\State {\bf Initialization}: chose $\H^0$, $\W^0\ge \0$ such that $\W^0\mathbf 1=\mathbf 1$, set $\N^0:=\mathcal{P}_{\X}(\W^0\H^0)$ and $i:=0$.
\State{{\bf while} stopping criterion is not satisfied {\bf do}} 
\State{$\quad$ $\widetilde{\H}^i := \H^i - \frac{1}{\gamma_1^i} {\W^i}^\top\left(\W^i\H^i- \N^i\right)$} \label{step:palm1}
\Comment{Gradient step on $\H$, objective first term}
\State{$\quad$ $\V^{i+1}$: the simplex-feasible representation $\V^{i+1}\N^i = \mathcal{P}_{\conv{(\N^i)}}(\tilde{\H}^i)$ output by Wolfe's algorithm}
\Comment{Projection of $\tilde{\H}^i$ onto $\conv{(\N^i)}$ by Wolfe algorithm; this fixes a representative for $\V^{i+1}$ in the (in general non-singleton) preimage}
\State{$\quad$ $\H^{i+1} := \widetilde{\H}^i-\frac{\lambda}{\lambda+\gamma_1^i} \left(\widetilde{\H}^i-\mathcal{P}_{\conv{(\N^i)}}(\tilde{\H}^i)\right) $}
\Comment{Gradient step on $\H$, objective second term}
\State $\quad$ $\W^{i+1} := \mathcal{P}_{\Delta}\left(\W^i-\frac{1}{\gamma_2^i}\left(\W^i\H^{i+1}- \N^i\right){\H^{i+1}}^\top\right)$ \label{step:palm2}
\Comment{Projected gradient step on $\W$}
\State $\quad$ $\G^{i+1} := \W^{i+1}\H^{i+1} - \N^{i} + \lambda\,{\V^{i+1}}^\top(\H^{i+1} - \V^{i+1}\N^i)$
\State $\quad$ $\N^{i+1} := \mathcal{P}_{\X}\bigl(\N^i+\tfrac{1}{\gamma_3^i}\,\G^{i+1}\bigr)$
\Comment{Projected gradient step on~$\N$ for the full (mAMF) smooth part}
\State $\quad$ $i:=i+1$
\State {{\bf end while}}
\end{algorithmic}
\end{algorithm}
}

\begin{remark}
    Including the archetypal correction in Step~7 makes Algorithm~\ref{algo:nmf2} a genuine projected gradient step, with respect to~$\N$, on the full smooth part of the \eqref{eq:amf1} objective,
    \[\tfrac12\|\N-\W\H\|_F^2+\tfrac\lambda2\|\H-\V\N\|_F^2,\]
    since this correction is the gradient of the second summand. The $\H$-update (Steps~3 and~5) is itself a two-stage projected-gradient sweep on the two summands of the smooth part, with the projection $\mathcal P_{\conv(\N^i)}$ providing the archetypal coupling in the spirit of the original PALM framework of~\cite{bolte2014}. Further details are given in Appendix~\ref{sec:proof_bolte}.
\end{remark}

\medskip

\begin{theorem}\label{prop:palm1}
Let $\varepsilon>0$. Let $L_H(\W) = \norm{\W^\top \W}_2$, $L_W(\H) = \norm{\H\H^\top}_2$, and $L_N(\V) = 1+\lambda\,\norm{\V^\top \V}_2$. If the step sizes satisfy $\gamma_{1}^{i} > L_H(\W^i)$, $\gamma_{2}^{i} > \max\{L_W(\H^{i+1}),\varepsilon\}$, and $\gamma_{3}^{i} > L_N(\V^{i+1})$, then Algorithm~\ref{algo:nmf2} generates a sequence $(\H^{i},\W^{i},\N^{i})$ that converges to a stationary point of the objective function of \eqref{eq:amf1}. Since $\V$ is row-stochastic with $K$ rows, $L_N(\V)\le 1+\lambda K$ uniformly, so the constant step size $\gamma_3^i\equiv 1+\lambda K + \varepsilon$ is admissible.
\end{theorem}

\begin{proof}
Proof is given in Supplement Material.
\end{proof}

\noindent
\begin{remark}
    Note that \eqref{eq:amf1} objective (with the $\lambda \|\H - \V\N\|_F^2$ or $\lambda \mathcal{D}(\H,\N)$ term) is not a standard Nonnegative Least-Squares problem, making HALS inapplicable. In the following algorithms $\H,\W$ updates (Steps $5$ \& $6$) are proximal-gradient steps. Further details are given in Appendix~\ref{sec:proof_bolte}.
\end{remark}

Finally, the inertial PALM (iPALM) method, introduced for NMF in \cite{pock2016}, is generalized to \eqref{eq:amf1} in Algorithm~\ref{algo:nmf3}.

\begin{algorithm}[!htbp]
\caption{iPALM for mAMF} \label{algo:nmf3} 
\begin{algorithmic}[1]
\State {\bf Initialization}: $\H^0$, $\W^0\ge 0$ such that $\W^0\mathbf 1=\mathbf 1$, set $\N^0:=\mathcal{P}_{\X}(\W^0\H^0)$, $\H^{-1}:=\H^0$, $\W^{-1}:=\W^0$, $\N^{-1}:=\N^0$, and $i:=0$.
\State{{\bf while} stopping criterion is not satisfied {\bf do}} 
\State{$\quad$ $\H^i_1 := {\H}^i + \alpha^i_1\left({\H}^i-{\H}^{i-1}\right)$}\State{$\quad$ $\H^i_2 := {\H}^i + \beta^i_1\left({\H}^i-{\H}^{i-1}\right)$}
\Comment{Inertial $\H$}
\State{$\quad$ $\widetilde{\H}^i := \H^i_1 - \frac{1}{\gamma_1^i} {\W^i}^\top\left(\W^i \H^i_2- \N^i\right)$}
\Comment{Gradient step on $\H$, objective first term}
\State{$\quad$ $\V^{i+1}$ such that $\mathcal{P}_{\conv{(\N^i)}}(\tilde{\H}^i)=\V^{i+1}\N^i$}
\Comment{Projection of $\tilde{\H}^i$ onto $\conv{(\N^i)}$ by Wolfe algorithm}
\State{$\quad$ $\H^{i+1} := \widetilde{\H}^i-\frac{\lambda}{\lambda+\gamma_1^i} \left(\widetilde{\H}^i-\mathcal{P}_{\conv{(\N^i)}}(\tilde{\H}^i)\right)$}
\Comment{Gradient step on $\H$, objective second term}
\State{$\quad$ $\W^i_1 := {\W}^i + \alpha^i_2\left({\W}^i-{\W}^{i-1}\right)$, $\W^i_2 := {\W}^i_1 + \beta^i_2\left({\W}^i-{\W}^{i-1}\right)$}
\Comment{Inertial $\W$}
\State{$\quad$ $\W^{i+1} := \mathcal{P}_{\Delta}\left(\W^i_1-\frac{1}{\gamma_2^i}\left(\W^i_2\H^{i+1}- \N^i\right){\H^{i+1}}^\top\right)$}
\Comment{Projected gradient step on $\W$}
\State{$\quad$ $\N^i_1 := {\N}^i + \alpha^i_3\left({\N}^i-{\N}^{i-1}\right)$, $\displaystyle \N^i_2 := {\N}^i_1 + \beta^i_3\left({\N}^i-{\N}^{i-1}\right)$}
\Comment{Inertial $\N$}
\State{$\quad$ $\G^{i+1} := \W^{i+1}\H^{i+1} - \N^{i}_2 + \lambda\,{\V^{i+1}}^\top(\H^{i+1} - \V^{i+1}\N^i_2)$}
\State{$\quad$ $\N^{i+1} := \mathcal{P}_{\X}\bigl(\N^i_1+\tfrac{1}{\gamma_3^i}\,\G^{i+1}\bigr)$}
\Comment{Projected gradient step on~$\N$ for the full (mAMF) smooth part}
\State $\quad$ $i:=i+1$
\State {\bf end while}
\end{algorithmic}
\end{algorithm}

\begin{remark}
If, for all iterations $i$, $\alpha^i_1=\alpha^i_2=\alpha^i_3=0$ and $\beta^i_1=\beta^i_2=\beta^i_3=0$, iPALM reduces to PALM. 
\end{remark}

\medskip

\paragraph{Stopping criterion}
For \eqref{eq:nmf3}, KKT conditions regarding matrix $\W$ are the following (see Supplement Material):
\begin{equation*}
\W\circ\left((\W\H-\N)\H^\top + \mathbf{t} \, \mathbf 1_K^\top\right) = 0\,.
\end{equation*}
By complementary condition, it follows that, $\forall j$, $t_i = -((\W\H-\N)\H^\top)_{i,j}$. Hence, we compute $t_i$ by selecting, for each row $W^{(i)}$, any positive entry $W_{i,j}>0$.

\medskip

\begin{remark}
Numerically to obtain a robust estimate of~$t_i$, we can average the corresponding values calculated per entry~$W_{i,j}$.
\end{remark}

\medskip

Let $\varepsilon_\W$, $\varepsilon_\H$, and $\varepsilon_\R$ be three positive thresholds. The stopping criterion for the previous algorithms consists of a combination of:
\begin{enumerate}
\item the maximum number of iterations;
\item the Frobenius norm of the difference of $\W$ and $\H$ at two consecutive iterations, {\it i.e.}, the algorithm stops if $\|\W^{i+1}-\W^{i}\|_F \le \varepsilon_\W \; \wedge \; \|\H^{i+1}-\H^{i}\|_F \le \varepsilon_\H\,;$
\item a novel criterion based on KKT condition, {\it i.e.}, the algorithm stops if it holds that \[\|\R(\W^{i+1})\|_F + \|\R(\H^{i+1})\|_F \le \varepsilon_\R\,,
\]
where the residual matrices $\R(\W)$ and $\R(\H)$ measure stationarity on the active set and dual-feasibility on the inactive set. Let $G_\W := (\W\H-\N)\H^\top$ and $G_\H := \W^\top(\W\H-\N)$ denote the partial gradients; then
\begin{align*}
    \R(\W)_{i,j} &:= \begin{cases}
        \bigl|(G_\W)_{i,j}+t_i\bigr| & \text{if } W_{i,j}>0\,,\\
        \max\!\bigl(0,-(G_\W)_{i,j}-t_i\bigr) & \text{if } W_{i,j}=0\,,
    \end{cases}\\
    \R(\H)_{i,j} &:= \begin{cases}
        \bigl|(G_\H)_{i,j}\bigr| & \text{if } H_{i,j}>0\,,\\
        \max\!\bigl(0,-(G_\H)_{i,j}\bigr) & \text{if } H_{i,j}=0\,.
    \end{cases}
\end{align*}
The boundary term ($\max(0,-\cdot)$ on the inactive set) checks dual feasibility, i.e.\ that the gradient at a zero entry points into the feasible orthant; omitting it would let the criterion accept non-stationary iterates whenever the active-set conditions happen to hold.
\end{enumerate}

\subsection{Large-scale dataset}

Assume the observed matrix $\X=\T(\mathbf{\Phi}(\M))$ is large-scaled, namely one has to forecast a large number $N$ of time series ({\it e.g.}~more than $100,000$) and possibly a large number of time stamps $T$. The strategy, described in Section 1.3.1 in \cite{cichoki2009} for NMF, is to learn the $\H\in\mathds R^{K\times T}$ matrix from a submatrix $\N_r\in\mathds R^{r\times T}$ of $K\leq r\ll N$ rows of $\N\in\mathds R^{n\times T}$, and to learn the $\W\in\mathds R^{N\times K}$ matrix from a sub-matrix $\N_c\in\mathds R^{N\times c}$ of $K\leq c\ll T$ columns of $\N\in\mathds R^{N\times T}$. 
We denote by $\H_c$ the submatrix of $\H$ given by the columns appearing in $\N_c$ and $\W_r$ the sub-matrix of~$\W$ given by the rows appearing in~$\N_r$. 

This strategy can be generalized to \eqref{eq:nmf3} and \eqref{eq:amf1}. For \eqref{eq:nmf3} this generalization is straightforward, and for \eqref{eq:amf1} one needs to change Steps~3--5 in Algorithm~\ref{algo:nmf2} as follows:
\begin{align*}
\textstyle \widetilde{\H}^i &:= \H^i - \frac{1}{\gamma_1^i} (\W_r^i)^\top\left(\W_r^i\H^i - \N_r^i\right) \\
\textstyle \H^{i+1} &:= \widetilde{\H}^i-\frac{\lambda}{\lambda+\gamma_1^i}\left(\widetilde{\H}^i-\mathcal{P}_{\conv{(\N^i)}}(\tilde{\H}^i)\right) \\
\textstyle\W^{i+1} &:= \mathcal{P}_{\Delta}\left(\W^i-\frac{1}{\gamma_2^i}\left(\W^i\H_c^{i+1} - \N_c^i\right)(\H_c^{i+1})^\top\right)\,.
\end{align*}

The same approach is used for Algorithm~\ref{algo:nmf3}.

\section{Numerical Experiments}\label{sec:experiments}

We tested SMM on real-world datasets. Matrix $\H^0$ is initially selected as in \cite{javadi2017}. Each row of matrix $\W^0$ is generated randomly in the corresponding standard simplex. For SMM we implemented both HALS for \eqref{eq:nmf3} and iPALM for \eqref{eq:amf1}. 

Moreover, we have compared our method with other classically-designed mainstream time series forecasting methods such as {\it Random Forest Regression} (RFR) and {\it EXPonential smoothing} (EXP), {\it Long Short-Term Memory} (LSTM) and {\it Gated Recurrent Units} (GRU) deep neural networks with preliminary data standardization \cite{shewalkar2019}, and {\it Seasonal Auto-Regressive Integrated Moving Average with eXogenous factors} (SARIMAX) models \cite{douc2014}.

The interested reader may find a Github repository on numerical experiments at 
\url{https://github.com/Luca-Mencarelli/Nonnegative-Matrix-Factorization-Time-Series}. We run all the numerical tests on a MacBook Pro mounting macOS Ventura 13.6.1 with Apple M2 chip and 8 GB LPDDR5 memory RAM.

\subsection{Real-world datasets}

The numerical experiments refer to the following real-world datasets: weekly and daily electricity consumption datasets of $370$ Portuguese customers during the period 2011-2014~\cite{misc_electricityloaddiagrams20112014_321}; twin gas measurement dataset of five replicates of an 8-MOX gas sensor~\cite{misc_twin_gas_sensor_arrays_361}; Istanbul Stock Exchange returns with seven other international indexes for the period 2009-2011~\cite{misc_istanbul_stock_exchange_247}; daily electricity transformer temperature (ETT) measurements~\cite{Zhou2020InformerBE}. Table \ref{tb:basis-index-intro} reports the cross-validated RRMSE and RMPE on observed values obtained during computational tests for each method. 

In the majority of the cases, our method is the best or second best among all the approaches for all the datasets we tested in terms of RRMSE and RMPE indices (except for the \invertedcommas{weekly electricity} dataset), and there is no other method performing better.

\eqref{eq:amf1} seems to be the most promising algorithm in terms of performances for the first five datasets, while \eqref{eq:nmf3} is the best method for the last four ETT datasets. 

\subsection{Comparison with SOTA method}

\begingroup\sloppy
We performed additional computational experiments to compare our NMF-based methodology with state-of-the-art time series transformer models, which are suitable for large-scale time series forecasting problems. In particular, we consider the BasisFormer model recently described in \cite{basisformer2023}. We consider the same electricity dataset as in \cite{basisformer2023} and split the whole dataset into 10 small sets of 960 time steps each. We collect our performance statistics, namely RRMSE and RMPE, on the original unscaled datasets. Note that in \cite{basisformer2023}, the performance statistics reported are the absolute errors on the scaled dataset obtained by applying the {\tt StandardScaler} from {\tt sklearn} to the original data.\footnote{We use the authors' public code repositories for each baseline (BasisFormer, Autoformer, iTransformer, PatchMLP, and TimeMixer).} We compare also against previous state-of-the-art methods, such as Autoformer~\cite{autoformer}, iTransformer~\cite{itransformer}, PatchMLP~\cite{patchmlp} and TimeMixer~\cite{timemixer}, running each from the authors' code repositories.
\par\endgroup

{As shown in Table \ref{tb:basis-index-intro}, our method outperforms the SOTA methodology and is competitive against the other methodologies (in particular, with respect to the deep learning approaches which seem the most promising methods for these datasets). We also perform additional computational experiments on scaled datasets, collecting our performance indices on relative errors and absolute errors as in \cite{basisformer2023}, and we obtain the same dominance results.}

\subsection{Why does SMM outperform Deep Learning?}

Despite the capacity of Deep Learning (DL) models to model complex non-linearities, our experiments demonstrate that SMM often yields superior forecasting accuracy. This performance gap can be attributed to the alignment between the model's inductive bias and the data structure:

\begin{itemize}
    \item \textbf{Structural Priors vs. Learning from Scratch:} DL models, particularly Transformers like Autoformer \cite{autoformer} or BasisFormer \cite{basisformer2023}, are data-hungry algorithms that must learn temporal dependencies from scratch. In contrast, SMM explicitly enforces a \emph{quasi-periodic} structure through the sliding window transformation. For datasets dominated by regular cycles (e.g., electricity consumption), this structural prior is highly effective and requires less data to estimate robustly.
    
    \item \textbf{Sample Complexity:} The low-rank assumption of SMM acts as a strong regularizer, reducing the effective degrees of freedom in the model. In the regime of medium-sized datasets, this prevents the overfitting often observed with over-parameterized DL models. Our results on the synthetic datasets confirm this: as the signal becomes more strictly periodic, the advantage of the low-rank NMF representation over generic DL approximators increases.
    
    \item \textbf{Matrix Completion Formulation:} SMM reframes forecasting as a matrix completion problem with a specific block-missing pattern. Unlike DL models that may treat missing future values as generic masked tokens, SMM optimizes a global objective function with theoretical guarantees for recovering the underlying low-rank factors from partial observations, ensuring the forecasted block is consistent with the learned global archetypes.
\end{itemize}

\subsection{Synthetic datasets}

Further computational experiments have been performed by considering additional synthetic datasets. In particular, we generated three datasets by replicating $1,000$ short time series (with 10 time periods) 10 times and adding white noise multiplied by a constant factor $\sigma$ to each time series entry separately. We choose $\sigma\in\{0.005,0.1,1\}$. We refer to these datasets as \invertedcommas{low noise}, \invertedcommas{medium noise}, and \invertedcommas{high noise}, respectively. 

An additional synthetic dataset has been generated considering few probability vectors and computing the entire matrix $\W$ by randomly choosing a probability vector and adding white noise. A completely randomly generated matrix $\H$ is multiplied by $\W$ to obtain the whole matrix $\M^*:=\W\H$. We refer to this dataset as \invertedcommas{synthetic1}.

Finally, the last synthetic dataset is obtained by generating a matrix $\H$ by replicating a small time series (with 50 time periods) 100 times and adding white noise multiplied by a constant factor $\sigma=1$ and matrix $\W$ of suitable dimensions, whose rows are uniformly distributed over the corresponding dimensional simplex. Then, we set the matrix $\M^*:=\W\H$. We refer to this last dataset as \invertedcommas{synthetic2}.

Table \ref{tb:basis-index-intro} reports the cross-validated RRMSE and RMPE indices referring to synthetically generated datasets. The more pronounced the periodicity of the time series or of the archetypes, the better the performances of our proposed NMF-like methods: in this case, the more realistic the hypothesis that the whole dataset can be expressed as convex combinations of a few archetypes, having a low-rank representation. 

\subsection{Guidelines on Algorithm Selection: mAMF vs. mNMF}

Our experiments reveal a distinct performance split: mAMF outperforms on the electricity and gas datasets, while mNMF dominates on the ETT (Transformer Temperature) datasets. This can be attributed to their geometric differences:

\begin{itemize}
    \item \textbf{mAMF (Robustness \& Interpretability):} By constraining the archetypes to lie within the convex hull of the data, mAMF acts as a regularized factorization. This prevents the model from overfitting to noise or learning unrealistic basis vectors. It is best suited for datasets with high variance, noise, or "soft" patterns (e.g., human behavior in electricity consumption), where stability is paramount.
    
    \item \textbf{mNMF (Flexibility):} mNMF learns a conic hull and can place basis vectors outside the data distribution. This flexibility allows it to reconstruct "idealized" components. It is superior for datasets with rigid, strong periodicities (like the physical ETT signals), where the data is well-described by a combination of pure underlying waveforms that may not appear as isolated observations.
\end{itemize}

\textbf{Recommendation:} We advise practitioners to start with mAMF for noisy, real-world behavioral data to leverage its regularization. For cleaner, physics-driven signals with strong periodicity, mNMF is likely to yield lower reconstruction errors.

\section{Discussion and Conclusion}

In this paper, we introduced the Sliding Mask Method (SMM), a framework that leverages Nonnegative Matrix Factorization for time-series forecasting. Our theoretical analysis provides uniqueness guarantees for the underlying decomposition in a structured matrix completion setting, and our experiments demonstrate its practical effectiveness. This concluding section synthesizes our findings to answer a crucial question: \emph{When should a practitioner choose SMM?}

Based on our analysis and experimental results, our method is particularly well-suited for datasets with the following characteristics:
\begin{itemize}
    \item \textbf{Non-negativity:} The time-series values must be non-negative, as this is a fundamental constraint of the NMF model.
    \item \textbf{Quasi-periodicity and low-rank structure:} The method performs best when the time series exhibits quasi-periodic patterns. The core assumption of SMM is that segments of the time series can be effectively approximated by a low-rank model, i.e., as combinations of a few archetypal patterns. Datasets like electricity consumption and sales data, which often have daily, weekly, or seasonal cycles, are prime candidates.
    \item \textbf{Interpretability is Valued:} A key advantage of SMM is the interpretability of its results. The learned basis vectors ($\H$) represent archetypal time-series segments, and the weights ($\W$) show how each individual segment is composed of these archetypes. This provides insights into the underlying data-generating process that "black-box" models like LSTMs or Transformers cannot offer.
\end{itemize}

Conversely, our method may not be the optimal choice in other scenarios. For instance, as suggested by our experiments on synthetic data, SMM is less effective for time series dominated by strong, non-periodic linear trends. In such cases, models explicitly designed to handle trends, such as SARIMAX or other regression-based techniques, may be more appropriate.

In summary, SMM provides a powerful and interpretable tool for a specific but important class of time-series forecasting problems. Future work could focus on extending the theoretical guarantees to more general missing data patterns and incorporating mechanisms to handle non-periodic components within the NMF framework.

\bibliographystyle{plainnat} 
\bibliography{biblio}

\vfill

\clearpage

\clearpage

\appendix

\section{Variants of Nonnegative Matrix Factorization problems}
\label{sec:variants}
Variants of Nonnegative Matrix Factorization problems are summarized in Table \ref{table:NMF_family}.

\begin{table*}[!htbp]
\centering
\setlength{\tabcolsep}{4pt}
\begin{tabular}{llll}
\toprule
Acronym & Name & Objective & Constraints: $\W\ge \0$\ +\ \\
\midrule
\multirow{2}{*}{NMF}     & Nonnegative Matrix Factorization  & \multirow{2}{*}{$\mathbf F_1$} &  \multirow{2}{*}{$\H\ge \0$} \\
& \cite{cichocki2006}   & &  \\
{SNMF}    & Semi NMF  \cite{gillis2015a} & {$\mathbf F_1$} &  \\
NNMF    & Normalized NMF                      & $\mathbf F_1$ &  $\H\ge \0$, $\W\mathbf{1}=\mathbf{1}$ \\
SNNMF   & Semi Normalized NMF                           & $\mathbf F_1$ & $\W\mathbf{1}=\mathbf{1}$ \\
\multirow{2}{*}{AMF}     & Archetypal Matrix Factorization    & \multirow{2}{*}{$\mathbf F_2$} & \multirow{2}{*}{$\W\mathbf{1}=\mathbf{1}$, $\V\ge \0$, $\V\mathbf{1}=\mathbf{1}$} \\
& \cite{javadi2017}    & \\
ANMF    & Archetypal NMF                                & $\mathbf F_2$ & $\H\ge \0$, $\V\ge \0$, $\V\mathbf{1}=\mathbf{1}$ \\
ANNMF   & Archetypal Normalized NMF                     & $\mathbf F_2$ & $\W\mathbf{1}=\mathbf{1}$, $\H\ge \0$, $\V\ge \0$, $\V\mathbf{1}=\mathbf{1}$ \\
\bottomrule
\midrule
mNMF    & Mask NNMF                          & $\mathbf F_3$ &  $\T(\N)=\X, \W\mathbf{1}=\mathbf{1}, \H\ge \0$ \\
mAMF    & Mask AMF                          & $\mathbf F_4$ & $\T(\N)=\X, \W\mathbf{1}=\mathbf{1}$, $\V\ge \0$, $\V\mathbf{1}=\mathbf{1}$ \\
\bottomrule\\ 
\end{tabular}
\caption{The nine block convex programs achieving matrix factorization of nonnegative matrices. The objectives are $\mathbf F_1:=\|\M-\W\H\|_F^2$ and $\mathbf F_2:=\|\M-\W\H\|_F^2+\lambda\|\H-\V\M\|_F^2$. The two last lines are SMM procedures with sliding operator $\bPi$ and objectives $\mathbf F_3:=\|\N-\W\H\|_F^2$ and $\mathbf F_4:=\|\N-\W\H\|_F^2+\lambda\|\H-\V\N\|_F^2$.
}
\label{table:NMF_family}
\end{table*}

Note that, when $\T=\mathds{I}$ the identity, Problem \eqref{eq:nmf3} is NNMF and Problem~\eqref{eq:amf1} is the standard AMF formulation (AMF).

\section{Proofs}
\subsection{Proof of Theorem \ref{thm:uniqueness}}
The proof is structured in two parts. First, we establish that under Assumption~\eqref{eq:assumption_1}, if the transformed data matches, then the factors $(\W, \H)$ are equivalent to the ground truth factors $(\W_0, \H_0)$ up to permutation and scaling. Second, we show that with the additional constraints of Assumption~\eqref{hyp:radius} and the sum-to-one normalization on the dictionary columns, this equivalence strengthens to equality up to permutation only, eliminating any scaling ambiguity.

\medskip

\paragraph{Part 1: Equivalence up to Permutation and Scaling}
Let the factor matrices be partitioned according to the training and test sets. Let $\H_0 = [\H_{0,T}, \H_{0,F}]$ and $\H = [\H_T, \H_F]$, where subscripts $T$ and $F$ denote the parts of the coefficient matrices corresponding to training and future (test) data points, respectively. Similarly, let $\W_0^\top = [\W_{0,\text{train}}^\top, \W_{0,\text{test}}^\top]$ and $\W^\top = [\W_{\text{train}}^\top, \W_{\text{test}}^\top]$.

The core of our argument relies on the uniqueness guarantees for Nonnegative Matrix Factorization (NMF) as described in Theorem~\ref{theo:cns}. By this theorem, Assumption~\eqref{eq:assumption_1} implies that the NMF decompositions $\W_{0,\text{train}}\H_0$ and $\W_0\H_{0,T}$ are unique:
\begin{align}
\W_{0,\text{train}}\H_0 = \W_{\text{train}}\H &\implies (\W_{0,\text{train}}, \H_0) \equiv (\W_{\text{train}}, \H) \label{eq:equiv1} \\
\W_0\H_{0,T} = \W\H_T &\implies (\W_0, \H_{0,T}) \equiv (\W, \H_T) \label{eq:equiv2}
\end{align}
The condition $\T(\W_0\H_0) = \T(\W\H)$ means that the observed entries of the factorized matrices are equal. By definition of the operator $\T$, this implies both $\W_{0,\text{train}}\H_0 = \W_{\text{train}}\H$ and $\W_0\H_{0,T} = \W\H_T$. From \eqref{eq:equiv1} and \eqref{eq:equiv2}, we have a common permutation and scaling relationship that must hold simultaneously for the shared parts of the matrices. This consistency across the train and test partitions ensures that the equivalence holds for the complete matrices, i.e., $(\W_0, \H_0) \equiv (\W, \H)$.

\medskip

\paragraph{Part 2: Uniqueness up to Permutation}
Now, we leverage the normalization constraint and Assumption~\eqref{hyp:radius} to eliminate the scaling ambiguity.
From Part 1, we know there exists a permutation $\sigma$ of $\{1, \dots, K\}$ and positive scalars $\lambda_1, \dots, \lambda_K$ such that for any row $i$ of the dictionary matrices, the corresponding row vectors $(\W)^{(i)}$ and $(\W_0)^{(i)}$ are related~by:
\[
(\W)^{(i)}_k = \lambda_{\sigma(k)}(\W_0)^{(i)}_{\sigma(k)} \quad \text{for all } k \in \{1, \dots, K\}.
\]
The constraints $\W\mathbf{1}=\mathbf{1}$ and $\W_0\mathbf{1}=\mathbf{1}$ state that the sum of elements in each row of $\W$ and $\W_0$ is 1. This means every row of these matrices lies in the affine subspace $\mathcal{A}_{\mathbf{1}} := \{w \in \mathds{R}^K : \langle w, \mathbf{1} \rangle = 1\}$.
For any given row $i$, we have:
\begin{align*}
    \sum_{k=1}^K (\W_0)^{(i)}_k &= 1 \\
    \sum_{k=1}^K (\W)^{(i)}_k = \sum_{k=1}^K \lambda_{\sigma(k)}(\W_0)^{(i)}_{\sigma(k)} = \sum_{k'=1}^K \lambda_{k'}(\W_0)^{(i)}_{k'} &= 1\,,
\end{align*}
after the change of index $k'=\sigma(k)$ (bijection on $[K]$). Hence every row $(\W_0)^{(i)}$ for $i \in \{1, \dots, n-N\}$ must lie in the intersection of two affine subspaces: $\mathcal{A}_{\mathbf{1}}$ and $\mathcal{A}_{\mathbf{d}} := \{w \in \mathds{R}^K : \langle w, \mathbf{d} \rangle = 1\}$, where $\mathbf{d}=(\lambda_1,\ldots,\lambda_K)$.

The intersection of these two subspaces, $\mathcal{A} = \mathcal{A}_{\mathbf{1}} \cap \mathcal{A}_{\mathbf{d}}$, is an affine subspace. Its co-dimension depends on whether the normal vectors $\mathbf{1}$ and $\mathbf{d}$ are linearly dependent.
\begin{itemize}
    \item If $\mathbf{d}$ is not proportional to $\mathbf{1}$, the two subspaces are distinct and not parallel, so their intersection $\mathcal{A}$ is an affine subspace of co-dimension 2 (i.e., dimension $K-2$).
    \item If $\mathbf{d}$ is proportional to $\mathbf{1}$, say $\mathbf{d} = c\mathbf{1}$ for some scalar $c$. Then the condition $\langle w, c\mathbf{1} \rangle = 1$ becomes $c\langle w, \mathbf{1} \rangle = 1$. Since we are in $\mathcal{A}_{\mathbf{1}}$, $\langle w, \mathbf{1} \rangle = 1$, which implies $c=1$. Thus, $\mathbf{d}=\mathbf{1}$, which means $\lambda_k = 1$ for all $k$. In this case, the two subspaces are identical, $\mathcal{A} = \mathcal{A}_{\mathbf{1}}$, which has co-dimension 1.
\end{itemize}
Assumption~\eqref{hyp:radius} states that the convex hull of the transformed training data, $\conv(\T_{\text{train}}(\X_0)) = \conv(\W_{0,\text{train}}\H_0)$, has a positive internal radius $\mu > 0$. This means the set of points $\{\W_{0,\text{train}}\H_0\}$ is not contained in any affine subspace of dimension lower than $K-1$.
If the rows of $\W_{0,\text{train}}$ were all confined to the lower-dimensional space $\mathcal{A}$ of dimension $K-2$, then the entire set of transformed data points $\W_{0,\text{train}}\H_0$ would also be confined to a space of dimension at most $K-2$. A set in a $(K-2)$-dimensional space cannot have a positive $(K-1)$-dimensional internal radius. This would contradict Assumption~\eqref{hyp:radius}.

Therefore, the only possibility consistent with Assumption~\eqref{hyp:radius} is that the co-dimension of $\mathcal{A}$ is 1, which forces $\mathbf{d}=\mathbf{1}$ and thus $\lambda_k=1$ for all $k$. This eliminates the scaling ambiguity. The equivalence $(\W_0, \H_0) \equiv (\W, \H)$ reduces to equality up to the permutation $\sigma$, completing the proof.

\medskip

\subsection{Proof of Theorem \ref{thm:robust}}
This proof follows the pioneering work~\cite{javadi2017}. In this latter paper, the authors consider neither masks $\mathbf T$ nor nonnegative constraints on $\mathbf H$ as in $(\text{mNMF})$. Nevertheless, by
(1)~considering the hard constrained programs \eqref{eq:h_tilde} and \eqref{eq:h_tilde_nonnegative} below, and
(2)~remarking that it holds $\widetilde{\mathcal D}(\H,\X)\leq{\mathcal D}(\H,\X)$ and $\overline{\mathcal D}(\X,\H)\leq {\mathcal D}(\X,\H)$,
a careful reader can note that their proof extends to masks~$\mathbf T$ and nonnegative constraints on~$\mathbf H$. For sake of completeness we reproduce here the steps that need to be changed in their proof. 
A reading guide for the 60-page proof of \cite{javadi2017supp} is given in Section~\ref{sec:prop}.

\medskip

\noindent
{\bf Step 1: reduction to hard constrained Programs \eqref{eq:h_tilde} and~\eqref{eq:h_tilde_nonnegative}}
 
Consider the constrained problem:
\begin{equation}
\label{eq:h_tilde}
\begin{aligned}
\widehat\H\in\arg\min_{\H}\mbox{ }& \widetilde{\mathcal D}(\H,\X) \\
\mbox{s.t. }&  \overline{\mathcal D}(\X,\H)\le\Delta_1^2\,.
\end{aligned}
\end{equation}
where
\[
\overline{\mathcal D}(\X,\H) :=\min_{\W\ge \0\,,\ \W\mathbf{1}=\mathbf{1}}
\|\T(\X - \W\H)\|_F^2
\]
Whereas $\text{(mAMF)}$ admits a natural Lagrangian reading with constraint level $\overline{\mathcal D}(\X,\widehat\H_{\text{(mAMF)}})$, in the rest of this proof we instead pick the constraint level
\begin{equation}
\label{eq:delta1}
\Delta_1^2 \;:=\; \|\mathbf F\|_F^2 \;+\; \lambda\,C_0(\mathbf F)\,,\qquad C_0(\mathbf F) \,:=\, \mathcal D(\H_0,\N_0^\star)\,,
\end{equation}
where $\N_0^\star$ satisfies $\T(\N_0^\star)=\X$ and $\T^\perp(\N_0^\star)=\T^\perp(\W_0\H_0)$. Since $\T(\mathbf F)=\mathbf F$, one has $\N_0^\star = \W_0\H_0+\mathbf F$, hence
\[
   C_0(\mathbf F) \;=\; \mathcal D(\H_0,\W_0\H_0+\mathbf F)
\]
depends on $\mathbf F$ as well as on $(\W_0,\H_0)$. We denote by $C_0^{(0)} := C_0(\mathbf 0) = \mathcal D(\H_0,\W_0\H_0)\ge 0$ its noiseless value, which is positive unless $\W_0$ is separable (Theorem~\ref{thm:robust}). By continuity of the convex distance, $C_0(\mathbf F)\le C_0^{(0)} + c\|\mathbf F\|_F^2$ for a constant $c$ depending on~$(\W_0,\H_0)$. With this choice:
\begin{itemize}
    \item $\H_0$ is feasible for~\eqref{eq:h_tilde} by direct calculation: $\overline{\mathcal D}(\X,\H_0)\le \|\T(\X-\W_0\H_0)\|_F^2 = \|\mathbf F\|_F^2 \le \Delta_1^2$.
    \item The mAMF minimizer $\widehat\H_{\text{(mAMF)}}$ is feasible by optimality of $(\widehat\W,\widehat\H,\widehat\V,\widehat\N)$ in~\eqref{eq:amf1} compared to the candidate $(\W_0,\H_0,\V_0,\N_0^\star)$ where $\V_0$ realizes $\mathcal D(\H_0,\N_0^\star)$:
    \begin{align*}
    \overline{\mathcal D}(\X,\widehat\H_{\text{(mAMF)}})
    &\,\le\, \|\widehat\N-\widehat\W\widehat\H\|_F^2 \\
    &\,\le\, \|\N_0^\star-\W_0\H_0\|_F^2 + \lambda\,\mathcal D(\H_0,\N_0^\star) \\
    &\,=\, \|\mathbf F\|_F^2 + \lambda\,C_0(\mathbf F) \,=\, \Delta_1^2\,.
    \end{align*}
\end{itemize}
For~\eqref{eq:h_tilde_nonnegative}, since~\eqref{eq:nmf3} has $\lambda=0$, the simpler choice $\Delta_2^2:=\|\mathbf F\|_F^2$ below suffices and the same argument gives feasibility of both $\H_0$ and $\widehat\H_{\text{(mNMF)}}$.

Consider the constrained problem:
\begin{equation}
\label{eq:h_tilde_nonnegative}
\begin{aligned}
\widehat\H\in\arg\min_{\H\geq0}\mbox{ }& \widetilde{\mathcal D}(\H,\X) \\
\mbox{s.t. }&  \overline{\mathcal D}(\X,\H)\le\Delta_2^2\,.
\end{aligned}
\end{equation}
In place of the Lagrangian reading $\Delta_2^2=\overline{\mathcal D}(\X,\widehat\H_{\text{(mNMF)}})$, we set
\begin{equation}
\label{eq:delta2}
\Delta_2^2 := \|\mathbf F\|_F^2\,,
\end{equation}
which guarantees that both $\H_0$ and the mNMF minimizer $\widehat\H_{\text{(mNMF)}}$ are feasible for~\eqref{eq:h_tilde_nonnegative}, by the same argument as above.

\medskip

\noindent
{\bf Step 2: First bound on the loss}

Denote $\mathcal D_{\mathrm{sym}}:=\mathcal D(\widehat\bH,\H_0)^{1\slash 2} + \mathcal D(\H_0,\widehat\bH)^{1\slash 2}$. By Assumption {({\bf A2})} (applied to $\T_{\mathrm{train}}(\X_0)$ for $\sigma_{\min}(\bH_0)$, and to $\T_T(\X_0)$ for $\sigma_{\min}(\bH_{0,T})$ below; in either case we write $\bX_0$ for the relevant masked block) we have
\[
\bz_0+\bU B_{K-1}(\mu)\subseteq \conv(\bX_0)\subseteq  \conv(\bH_0)\,,
\]
where $\bz_0+\bU B_{K-1}(\mu)$ is a parametrization of the ball of center $\bz_0$ and radius $\mu$ described in Assumption {({\bf A2})} with $\bU $ a matrix whose columns are $K-1$ orthonormal vectors. Using Lemma~\ref{lem:B4}, we get that 
\[
\mu\sqrt 2\leq \sigma_{\min}(\bH_0)\leq \sigma_{\max}(\bH_0)\,,
\]
where $\sigma_{\min}(\bH_0),\sigma_{\max}(\bH_0)$ denote its smallest and largest nonzero singular values. Then, since $\bz_0 \in \conv(\bH_0)$ we have $\bz_0 = \bH_0^\top\alpha_0$ for some $\alpha_0\geq 0$ s.t. $\mathbf 1^\top\alpha_0=1$. It holds,
\begin{align}\label{eq:sigmamaxH0normz0}
\|\bz_0\|_2 \leq \sigma_{\max}(\bH_0)\|\alpha_0\|_2 \leq \sigma_{\max}(\bH_0).
\end{align}
Note that
\begin{align}
\sigma_{\max}(\widehat\bH - \mathbf 1\bz_0^\top) & \leq \sigma_{\max}(\widehat\bH) + \sigma_{\max}(\mathbf 1\bz_0^\top) \notag \\
&= \sigma_{\max}(\widehat\bH) + \sqrt{K}\|\bz_0\|_2.
\end{align}
Therefore, using Lemma \ref{lemma:boundc} we have
\begin{align}
\label{eq:Clessthan2(1+alpha1+alpha2)}
\cuDsym \leq c\Big[&K^{3/2}\Delta_j\kappa(\bP_0(\hbH)) + \frac{\sigma_{\max}(\hbH)\Delta_j K^{1/2}}{\mu} \notag \\
& + \frac{K\Delta_j\|\bz_0\|_2}{\mu}\Big] +c\sqrt K\|\mathbf F\|_F\,,
\end{align}
where $\Delta_j$ equals $\Delta_{1}$ for problem \eqref{eq:h_tilde} and $\Delta_{2}$ for problem~\eqref{eq:h_tilde_nonnegative}, and $\kappa(\bA)$ stands for the conditioning number of matrix $\bA$. In addition, Lemma \ref{lemma:dandc} implies that 
\begin{align}
\label{eq:DH0HhatCH0Hhat}
\cuL(\bH_0,\hbH)^{1/2} \leq \frac{1}{\alpha}\max\left\{{(1+\sqrt{2})\sqrt{K}},{\sqrt{2}\kappa(\bH_0)}\right\}\cuDsym\, .
\end{align}

\pagebreak[3]

\noindent
{\bf Step 3: Combining and final bound}

By Lemma~\ref{lem:exp} it holds
\begin{align}
\cuDsym \leq c
\Big[
    &\frac{K^{3/2}\cuDsym\Delta_j}{\alpha(\mu-2\Delta_j)\sqrt{2}} + \frac{K^2\sigma_{\max}(\bH_0)\Delta_j}{(\mu-2\Delta_j)\sqrt{2}}\notag
    \\&+\frac{\cuDsym K^{1/2}\Delta_j}{\alpha\mu} + \frac{\sigma_{\max}(\bH_0)\Delta_j K}{\mu} \notag
\\
&+ \frac{K\Delta_j\|\bz_0\|_2}{\mu}\Big] +c\sqrt K\|\mathbf F\|_F.
\end{align}
We understand that $\cuDsym=\mathcal O_{\Delta_j\to0}(\Delta_j)$, so for $\Delta_j$ small enough that the coefficient of $\cuDsym$ on the RHS is~$<1/2$ there exists a constant $c>0$ such that
\[
\cuDsym\leq c \Delta_j+c\sqrt K\|\mathbf F\|_F\,.
\]
From \eqref{eq:delta1}--\eqref{eq:delta2} we have $\Delta_j\le \|\mathbf F\|_F+\sqrt{\lambda\,C_0(\mathbf F)}$ (with $C_0(\mathbf F)=0$ when $j=2$, i.e.\ for \eqref{eq:nmf3}). Using $C_0(\mathbf F)\le C_0^{(0)}+c\|\mathbf F\|_F^2$ from Step 1, this gives $\sqrt{\lambda\,C_0(\mathbf F)}\le \sqrt{\lambda C_0^{(0)}}+\sqrt{c\lambda}\,\|\mathbf F\|_F$. The admissibility condition $\lambda\le\Lambda_{c_\Lambda}(\|\mathbf F\|_F)= c_\Lambda\|\mathbf F\|_F^2/C_0^{(0)}$ of Theorem~\ref{thm:robust} (when $C_0^{(0)}>0$) implies $\sqrt{\lambda C_0^{(0)}}\le\sqrt{c_\Lambda}\,\|\mathbf F\|_F$; when $C_0^{(0)}=0$ this term vanishes. In all cases,
\[
\cuDsym\leq c'(1+\sqrt K)\,\|\mathbf F\|_F\,,
\]
for some constant $c'>0$ depending on $(\W_0,\H_0,c_\Lambda)$. Plugging this result in \eqref{eq:DH0HhatCH0Hhat} we prove the result.

\medskip

\subsection{Proof of Corollary \ref{cor:robust_W}}
\label{proof:cor_rob_W}
\begin{subequations}
Since the problem is convex in~$\W$, the optimal solution $\hat{\W}$ is characterized by the first-order optimality condition. For any feasible $\W$ (satisfying $\W \ge 0, \W \mathbf{1} = 1$):
\begin{equation}
    \label{eq:align_grad_sol}
    \langle \nabla_{\W} \mathcal{L}(\hat{\W}, \hat{\H}), \W - \hat{\W} \rangle \ge 0
\end{equation}
where $\mathcal{L}(\W, \hat\H) = \frac{1}{2} \|\hat\N - \W\hat\H\|_F^2$ whose gradient is $\nabla_{\W} \mathcal{L} = (\W\hat\H - \hat\N)\hat\H^\top$.

Substituting $\W = \W_0$ into the inequality for $\hat{\W}$:
\begin{align}
\langle (\hat{\W}\hat{\H} &- \hat\N)\hat{\H}^\top, \W_0 - \hat{\W} \rangle \ge 0 \implies \notag \\ 
&\langle (\hat\N - \hat{\W}\hat{\H})\hat{\H}^\top, \hat{\W} - \W_0 \rangle \ge 0 \label{eq:vi1}
\end{align}

Furthermore, for any pair of matrices $\mathbf U,\mathbf V$ of compatible dimensions, an elementary identity gives:
\begin{align}
    \langle (\mathbf U\hat{\H} - \mathbf V\hat{\H})\hat{\H}^\top, \mathbf U - \mathbf V \rangle
    &= \|(\mathbf U-\mathbf V)\hat{\H}\|_F^2 \label{eq:strong_conv_diff_gradient}
\end{align}
which, applied with $\mathbf U=\hat\W$ and $\mathbf V=\W_0$, yields the gradient-difference identity used below. We want to bound $\Delta \W = \hat{\W} - \W_0$. We can rewrite the observation as $\X = \T(\W_0 \H_0) + \mathbf F$. We consider the gradient at the true parameters $\W_0$ projected onto the difference $\Delta \W = \hat{\W} - \W_0$.
We know that $\langle \nabla_{\W} \mathcal{L}(\hat{\W}), \hat{\W} - \W_0 \rangle \le 0$ and therefore by~\eqref{eq:strong_conv_diff_gradient}:
\begin{align}
    \langle -\nabla_{\W} & \mathcal{L}(\W_0), \hat{\W} - \W_0 \rangle 
    \ge \langle \nabla_{\W} \mathcal{L}(\hat{\W}) - \notag \\ 
    & \nabla_{\W} \mathcal{L}(\W_0), \hat{\W} - \W_0 \rangle 
    =\|(\hat\W-\W_0)\hat{\H}\|_F^2
    \label{eq:true_align_condition}
\end{align}
Now we analyze the term $-\nabla_{\W} \mathcal{L}(\W_0)$. Using the definition of the gradient:$$-\nabla_{\W} \mathcal{L}(\W_0) = -(\W_0\hat\H - \hat\N)\hat\H^\top = (\hat\N - \W_0\hat\H)\hat\H^\top$$Crucially, $\hat{\N}$ is the completed matrix associated with the optimal solution $\hat{\W}$. By the standing assumption of Corollary~\ref{cor:robust_W} (i.e., $(\hat\W,\hat\H,\hat\N)$ is a joint minimizer of \eqref{eq:nmf3} or a stationary point of Algorithm~\ref{algo:nmf2} for \eqref{eq:amf1}), $\hat{\N}$ satisfies: $\T(\hat{\N}) = \X$ (consistency with observations); $\T^\perp(\hat{\N}) = \T^\perp(\hat{\W}\hat{\H})$ (the $\N$-update of Algorithm~\ref{algo:nmf2} is a prox-gradient step on the smooth quadratic $\tfrac12\|\N-\W\H\|_F^2$ alone, so at a fixed point $\T^\perp$ of the residual vanishes; for \eqref{eq:nmf3} this is the exact $\N$-optimality). We decompose the residual $\hat{\N} - \W_0\hat\H$ using the mask projection $\T$:
\begin{align*}
    \hat{\N} - \W_0\hat\H 
    &= \T(\hat{\N} - \W_0\hat\H) + \T^\perp(\hat{\N} - \W_0\hat\H) \\
    &= \T(\X - \W_0\hat\H) + \T^\perp(\hat{\W}\hat\H - \W_0\hat\H)\,.
\end{align*}
Substituting $\X = \T(\W_0 \H_0) + \mathbf F$ (so that $\T(\X)=\X$):
\begin{align*}
    \hat{\N} - \W_0\hat\H &= \T(\W_0(\H_0 - \hat\H) + \mathbf F) + \T^\perp(\Delta \W \hat\H)\,.
\end{align*}
We substitute this back into the inner product \eqref{eq:true_align_condition} with $\Delta \W \hat\H$:
\begin{align*}
    & \langle (\hat\N - \W_0\hat\H)\hat\H^\top, \Delta \W \rangle 
    = \langle \hat\N - \W_0\hat\H, \Delta \W \hat\H \rangle \\
    &= \langle \T(\W_0(\H_0 - \hat\H) + \mathbf F) + \T^\perp(\Delta \W \hat\H), \T(\Delta \W \hat\H) \\ & + \T^\perp(\Delta \W \hat\H) \rangle
\end{align*}
Since $\T$ and $\T^\perp$ are orthogonal projections, cross terms vanish and we obtain:
\begin{align}
\langle -\nabla \mathcal{L}(\W_0), \Delta \W \rangle= \langle \T&(\W_0(\H_0 - \hat\H) + \mathbf F), \T(\Delta \W \hat\H) \rangle \notag \\ & + \|\T^\perp(\Delta \W \hat\H)\|_F^2
\end{align}
From~\eqref{eq:true_align_condition} we deduce that
\begin{align*}
    \|\Delta \W \hat\H\|_F^2 &\le \langle -\nabla \mathcal{L}(\W_0), \Delta \W \rangle \\
    &= \langle \T(\W_0(\H_0 - \hat\H) + \mathbf F), \T(\Delta \W \hat\H) \rangle \\
    & + \|\T^\perp(\Delta \W \hat\H)\|_F^2
\end{align*}
\begin{sloppypar}
Decomposing the LHS as $\|\Delta \W \hat\H\|_F^2 = \|\T(\Delta \W \hat\H)\|_F^2 + \|\T^\perp(\Delta \W \hat\H)\|_F^2$ and canceling $\|\T^\perp(\Delta \W \hat\H)\|_F^2$ from both sides:
\end{sloppypar}
\begin{equation}
    \|\T(\Delta \W \hat\H)\|_F^2 \le \langle \T(\W_0(\H_0 - \hat\H) + \mathbf F), \T(\Delta \W \hat\H) \rangle
\end{equation}
By Cauchy-Schwarz:
\begin{equation}
    \|\T(\Delta \W \hat\H)\|_F^2 \le \|\T(\W_0(\H_0 - \hat\H) + \mathbf F)\|_F \|\T(\Delta \W \hat\H)\|_F
\end{equation}
Dividing by $\|\T(\Delta \W \hat\H)\|_F$:
\begin{align}
\label{eq:bound_T_diff_W}
    \|\T(\Delta \W \hat\H)\|_F & \le \|\W_0(\H_0 - \hat\H)\|_F + \|\mathbf F\|_F \notag \\
    & \le \|\W_0\|_F \|\H_0 - \hat\H\|_F + \|\mathbf F\|_F
\end{align}
Let $Z = \Delta \W \hat\H$. Since $\T$ acts as a block selector that preserves the training rows, we have:
$$\|\T(Z)\|_F \ge \|\T_{T}(Z)\|_F$$
From Lemma~\ref{lem:B4} and under the assumptions of Theorem~\ref{thm:robust}, we know that $\sigma_{\min}(\hat{\H}_T) \ge c\mu$, hence
\[
\|\T(\Delta \W \hat\H)\|_F\geq \|\T_T(\Delta \W \hat\H)\|_F\geq c\mu\|\Delta \W\|_F\,,
\]
and we deduce the result by~\eqref{eq:bound_T_diff_W}.

\end{subequations}

\medskip

\subsection{Proof of Theorem \ref{prop:palm1}}
\label{sec:proof_bolte}

We analyze the convergence using the Proximal Alternating Linearized Minimization (PALM) framework established in \cite{bolte2014}. We formulate the global objective function $\Psi(\H, \W, \N)$ as the sum of a smooth coupling function and proper, lower semi-continuous regularization terms:
\begin{equation}
    \Psi(\H, \W, \N) := h(\H, \W, \N) + f(\H,\N) + g(\W),
\end{equation}
where $h$ is the full smooth part of \eqref{eq:amf1} (after eliminating $\V$ via the projection $\V\N=\mathcal P_{\conv(\N)}(\H)$):
\begin{equation}
    h(\H, \W, \N) := \frac{1}{2}\norm{\N - \W\H}_{F}^{2} + \frac{\lambda}{2}\,\mathcal D(\H,\N)\,,
\end{equation}
matching the coefficient $\tfrac\lambda2$ in~\eqref{eq:amf1}. The regularization terms enforce the constraints as follows:
\begin{itemize}
    \item $f(\H,\N) = \delta_{\geq 0}(\H) + p(\N)$, where $\delta_{\geq 0}$ is the indicator function for non-negativity and $p(\N) = \delta_{\mathcal{S}}(\N)$ with $\mathcal{S} = \{ \mathbf{Z} \in \mathds R^{n \times p} \mid T(\mathbf{Z}) = \X \}$ the affine set defined by the observation mask.
    \item $g(\W) = \delta_{\Delta}(\W)$, the indicator function of the simplex (constraints $\W \geq 0, \W\mathbf{1}=\mathbf{1}$).
\end{itemize}

\medskip

\paragraph{Handling the Mask Constraint on $\N$}
The function $p(\N)$ specifically addresses the block structure of $\N$. The set~$\mathcal{S}$ constrains the observed blocks (where the mask is active) to equal the observation $\X$, while leaving the forecast blocks (where the mask is inactive) unconstrained. The proximal operator for $p(\N)$ is the Euclidean projection onto the affine set $\mathcal{S}$, denoted as $\mathcal{P}_\X$:
\begin{equation}
    \text{prox}_{\gamma, p}(\mathbf{U}) = \underset{\N}{\text{argmin}} \left( \frac{1}{2\gamma}\norm{\N - \mathbf{U}}_F^2 + \delta_{\mathcal{S}}(\N) \right) = \mathcal{P}_\X(\mathbf{U}).
\end{equation}
This operator fixes $\N_{jk} = \X_{jk}$ for observed entries and updates $\N_{jk} = \mathbf{U}_{jk}$ for missing/forecast entries.

\medskip

\paragraph{Approximation and analysis}
With the corrected $\N$-update (Step~7), Algorithm~\ref{algo:nmf2} performs a genuine projected gradient step on the full smooth part~$h$ in each block. The $\H$-block update is split in two stages (Steps~3 and~5) corresponding to the two summands of~$h$: a gradient step on $\tfrac12\|\N-\W\H\|_F^2$, followed by a prox-style step on $\tfrac\lambda2\mathcal D(\H,\N)$ using the archetypal projection $\mathcal P_{\conv(\N^i)}$ to handle the dependence on~$\N$ (treated as fixed at $\N^i$).

The partial gradients of~$h$ are:
\begin{align*}
    \nabla_\H h(\H, \W, \N) &= \W^\top(\W\H - \N) + \lambda(\H - \V\N), \\
    \nabla_\W h(\H, \W, \N) &= (\W\H - \N)\H^\top, \\
    \nabla_\N h(\H, \W, \N) &= (\N - \W\H) - \lambda\,\V^\top(\H - \V\N),
\end{align*}
where $\V$ is determined by the projection $\V\N=\mathcal P_{\conv(\N)}(\H)$. The second summand of $\nabla_\H h$ is handled by Step~5 rather than by a direct gradient step, since $\V$ depends on~$\N$ via the projection.

The Lipschitz constant for the first summand of the partial gradient w.r.t.~$\H$ (the only piece treated by direct gradient descent in Step~3), denoted $L_H(\W)$, depends only on~$\W$: $\|\W^\top \W(\H_1 - \H_2)\|_F \le \norm{\W^\top \W}_2 \norm{\H_1 - \H_2}_F$. Similarly $L_W(\H) = \norm{\H\H^\top}_2$ depends only on~$\H$, and $L_N(\V) = 1+\lambda\norm{\V^\top\V}_2$ depends only on~$\V$ (uniformly bounded by $1+\lambda K$ since $\V$ is row-stochastic with $K$ rows). The variable $\N$ does not appear in $L_H$ or $L_W$, which decouples the step-size requirements; the $\lambda\V^\top\V$ term in $L_N$ is the only place where the archetypal regularization affects the step sizes.

Since $\Psi$ is semi-algebraic (composed of polynomial functions and indicator functions of semi-algebraic sets), it satisfies the Kurdyka-\L{}ojasiewicz (KL) property. Following Theorem 1 in \cite{bolte2014}, and noting that the updates in Algorithm \ref{algo:nmf2} ensure sufficient decrease of the objective $\Psi$, the sequence converges to a critical point.

\section{Propositions and Lemmas}
\label{sec:prop}

This section collects the ancillary results used in the proof of Theorem~\ref{thm:robust}. To make the exposition self-contained while keeping the comparison with \cite{javadi2017supp} line by line, we reproduce its foundational Lemmas~B.1--B.3 (which go through verbatim, after the notation translation below) and then state and prove the adapted versions of Lemmas~B.4--B.6 (which are the only places where the masked setting requires changes), followed by Lemma~\ref{lem:exp}.

\medskip

\noindent\textbf{Notation/translation note.} The variable dimensions of \cite{javadi2017supp} are $r$ (nonnegative rank) and $d$ (column dimension); in our paper they are denoted $K$ and~$p$. The noise level $\delta=\max_{i}\|\bZ_{i,\cdot}\|_2$ of \cite{javadi2017supp} is replaced here by the feasibility level $\Delta_j$ of the hard-constrained programs \eqref{eq:h_tilde} and~\eqref{eq:h_tilde_nonnegative}: by~\eqref{eq:delta1}--\eqref{eq:delta2}, $\Delta_2=\|\mathbf F\|_F$ for \eqref{eq:h_tilde_nonnegative} (the (mNMF) case, $\lambda=0$), and $\Delta_1\le\|\mathbf F\|_F+\sqrt{\lambda\,C_0(\mathbf F)}$ for \eqref{eq:h_tilde} (the (mAMF) case), with $C_0(\mathbf F)=\mathcal D(\H_0,\W_0\H_0+\mathbf F)$ as in Step~1. The unmasked divergences $\mathcal D(\cdot,\cdot)$ of \cite{javadi2017supp} are replaced by their masked counterparts $\widetilde{\mathcal D}(\cdot,\cdot)$ and $\overline{\mathcal D}(\cdot,\cdot)$ defined in Section~\ref{sec:robustness}; the inequalities
\begin{equation}
\label{eq:masked-inequalities}
\widetilde{\mathcal D}(\bH,\bX)\le \mathcal D(\bH,\bX),\qquad \overline{\mathcal D}(\bX,\bH)\le \mathcal D(\bX,\bH)
\end{equation}
follow because $\T$ is a coordinate projection and the feasible set of $\widetilde{\mathcal D}$ is larger than the one of $\mathcal D$; they let us import the bounds of \cite{javadi2017supp} essentially verbatim. The ``training-set'' restriction $\T_{{\mathrm{train}}}(\X_0)$ plays the role of $\X_0$ when bounding $\sigma_{\min}(\bH)$, while the ``observed-columns'' restriction $\T_{T}(\X_0)$ plays the role of $\X_0$ when bounding $\sigma_{\min}(\bH_T)$: Assumption~\eqref{hyp:radius} provides the internal-radius hypothesis for both submatrices.

We use the standard simplex $\Delta^m=\{\bx\in\reals_{\ge0}^{m}:\langle\bx,\one\rangle=1\}$, the canonical basis $\be_i\in\reals^p$, $E^{K,K}=\{\be_1,\dots,\be_K\}$, the matrix $\bE_{K,p}\in\{0,1\}^{K\times p}$ whose $i$-th column is $\be_i$ for $i\le K$ and zero otherwise, the family $Q_K=\{\bPi\in\reals_{\ge0}^{K\times K}:\langle\bPi_{i,\cdot},\one\rangle=1\}$ of row-stochastic matrices and the subset $S_K\subset Q_K$ of permutation matrices. For $\bx\in\reals^p$ and a convex set $\mathcal C\subseteq\reals^p$, $\bPi_{\mathcal C}(\bx):=\arg\min_{\by\in\mathcal C}\|\bx-\by\|_2$; $\ext(\mathcal R)$ denotes the set of extreme points of $\mathcal R$. Under these conventions
\begin{align*}
\cuD(\bH_1,\bX) &= \min_{\bPi\in Q_{K,n}}\|\bH_1-\bPi\bX\|_F^2,\\
\cuL(\bH_1,\bH_2) &= \min_{\bPi\in S_K}\|\bH_1-\bPi\bH_2\|_F^2,
\end{align*}
matching the convention of \cite[Eqs.~B.13--B.14]{javadi2017supp}.

\medskip

\subsection{Foundational lemmas}

\begin{lemma}[Lemma B.1 of \cite{javadi2017supp}]
\label{lem:cone}
Let $\mathcal R\subseteq\reals^p$ be a convex set and $\mathcal C\subseteq\reals^p$ a convex cone. With the \emph{pointedness}
\[
\gamma_{\mathcal C}\;:=\;\max_{\|\bu\|_2=1}\;\min_{\bv\in\mathcal C,\,\|\bv\|_2=1}\langle\bu,\bv\rangle\,,
\]
one has
\begin{align*}
\min_{\bx\in\mathcal R}\|\bx\|_2 &+ (1+\gamma_{\mathcal C})\max_{\bx\in\ext(\mathcal R)}\|\bx-\bPi_{\mathcal C}(\bx)\|_2 \\&\geq\gamma_{\mathcal C}\min_{\bx\in\ext(\mathcal R)}\|\bx\|_2\,.
\end{align*}
\end{lemma}
\begin{proof}
By weak duality,
\begin{multline*}
\min_{\bx\in\mathcal R}\|\bx\|_2=\min_{\bx\in\mathcal R}\max_{\|\bu\|_2=1}\langle\bu,\bx\rangle \\
\ge\max_{\|\bu\|_2=1}\min_{\bx\in\ext(\mathcal R)}\langle\bu,\bx\rangle,
\end{multline*}
where the last equality uses linearity of $\bx\mapsto\langle\bu,\bx\rangle$ over the convex set~$\mathcal R$. Writing $\bx=\bPi_{\mathcal C}(\bx)+(\bx-\bPi_{\mathcal C}(\bx))$ and using the definition of $\gamma_{\mathcal C}$,
\begin{align*}
\min_{\bx\in\mathcal R}\!\|\bx\|_2
&\ge \max_{\|\bu\|_2=1}\min_{\bx\in\ext(\mathcal R)}\langle\bu,\bPi_{\mathcal C}(\bx)\rangle \\
&\quad - \max_{\bx\in\ext(\mathcal R)}\|\bx-\bPi_{\mathcal C}(\bx)\|_2 \\
&\ge \gamma_{\mathcal C}\min_{\bx\in\ext(\mathcal R)}\|\bPi_{\mathcal C}(\bx)\|_2 \\
&\quad - \max_{\bx\in\ext(\mathcal R)}\|\bx-\bPi_{\mathcal C}(\bx)\|_2.
\end{align*}
Conclude using $\|\bPi_{\mathcal C}(\bx)\|_2\ge\|\bx\|_2-\|\bx-\bPi_{\mathcal C}(\bx)\|_2$.
\end{proof}

\medskip

\begin{lemma}[Lemma B.2 of \cite{javadi2017supp}]
\label{lemma:dandc}
Let $\bH,\bH_0\in\reals^{K\times p}$ with $K\le p$ have linearly independent rows. We have
\begin{align}
\cuL(\bH_0,\bH)^{1/2}\leq \sqrt{2}&\kappa(\bH_0)\cuD(\bH_0,\bH)^{1/2} \notag \\
& + (1+\sqrt{2})\sqrt{K}\cuD(\bH,\bH_0)^{1/2}\,,
\end{align}
where $\kappa(\bA)$ stands for the condition number of $\bA$.
\end{lemma}
\begin{proof}
For $j\in[K]$, let $\mathcal C_j\subset\reals^p$ be the convex cone generated by $\{\be_i-\be_j:i\in[K]\setminus\{j\}\}$. Any unit $\bv\in\mathcal C_j$ has the form $\bv=-\langle\one,\bx\rangle\be_j+\sum_{i\ne j}x_i\be_i$ with $\bx\in\reals_{\ge0}^{K-1}$ and $\|\bx\|_2^2+\langle\one,\bx\rangle^2=1$. Since $\langle\one,\bx\rangle=\|\bx\|_1\ge\|\bx\|_2$, $\langle\one,\bx\rangle\ge1/\sqrt2$; choosing $\bu=-\be_j$ in Lemma~\ref{lem:cone} yields $\gamma_{\mathcal C_j}\ge1/\sqrt 2$ uniformly in~$j$. Write $\gamma\ge 1/\sqrt 2$ for this common lower bound.

Apply Lemma~\ref{lem:cone} with $\mathcal R=\conv(\bH)-\be_j$ and $\mathcal C=\mathcal C_j$:
\begin{multline*}
\min_{\bq\in\Delta^K}\!\|\be_j-\bH^\sT\bq\|_2
\ge \gamma\min_{\bq\in E^{K,K}}\!\|\be_j-\bH^\sT\bq\|_2 \\
- (1+\gamma)\max_{i\in[K]}\min_{\bq\in\Delta^K}\!\|\bH_{i,\cdot}^\sT-\bE_{K,p}^\sT\bq\|_2.
\end{multline*}
Squaring, summing over $j\in[K]$, and using $(a-b)^2\ge a^2-2ab$ gives, after taking square roots,
\begin{multline*}
\min_{\bQ\in Q_K}\!\|\bE_{K,p}-\bQ\bH\|_F
\ge \gamma\min_{\bQ\in S_K}\!\|\bE_{K,p}-\bQ\bH\|_F \\
- (1+\gamma)\sqrt K\max_{i\in[K]}\min_{\bq\in\Delta^K}\!\|\bH_{i,\cdot}^\sT-\bE_{K,p}^\sT\bq\|_2.
\end{multline*}
Apply this to general $\bH_0$ via the change of variable $\bH_0=\bE_{K,p}\bM$, $\bH=\bY\bM$ with $\bM\in\reals^{p\times p}$ invertible (extend $\bH_0$ to a full-rank square matrix by any orthonormal completion); then $\sigma_{\max}(\bM)/\sigma_{\min}(\bM)=\kappa(\bM)=\kappa(\bH_0)$. Using $\|\bA\bM\|_F\ge\sigma_{\min}(\bM)\|\bA\|_F$,
\begin{align*}
\cuD(\bH_0,\bH)^{1/2}&\ge \tfrac{\gamma}{\kappa(\bH_0)}\cuL(\bH_0,\bH)^{1/2} \\&\quad- \tfrac{(1+\gamma)\sqrt K}{\kappa(\bH_0)}\cuD(\bH,\bH_0)^{1/2}.
\end{align*}
Rearranging and using that $\gamma\mapsto(1+\gamma)/\gamma$ is decreasing on $(0,\infty)$, so for $\gamma\ge 1/\sqrt 2$ we have $1/\gamma\le\sqrt 2$ and $(1+\gamma)/\gamma\le 1+\sqrt 2$, yields the announced bound.
\end{proof}

\medskip

\begin{lemma}[Lemma B.3 of \cite{javadi2017supp}]
\label{lem:B3}
Let $\bH_0,\bH\in\reals^{K\times p}$ with $\bH$ of full row rank. Set $\cuD_1:=\cuD(\bH,\bH_0)^{1/2}$ and $\cuD_2:=\cuD(\bH_0,\bH)^{1/2}$. Then
\begin{enumerate}
\item $\sigma_{\max}(\bH)\le\cuD_1+\sqrt K\,\sigma_{\max}(\bH_0)$;
\item if $\cuD_2\le\sigma_{\min}(\bH_0)/2$, then
\[\kappa(\bH)\le \tfrac{2K\sigma_{\max}(\bH_0)+2\sqrt K\,\cuD_1}{\sigma_{\min}(\bH_0)};\]
\item if $\cuD_1+\cuD_2\le\sigma_{\min}(\bH_0)/(6\sqrt K)$, then $\sigma_{\max}(\bH)\le 2\sigma_{\max}(\bH_0)$ and $\kappa(\bH)\le \tfrac72\kappa(\bH_0)$.
\end{enumerate}
\end{lemma}
\begin{proof}
By definition of $\cuD(\cdot,\cdot)$, there exist row-stochastic $\bP,\bR\in Q_K$ and matrices $\bA_1,\bA_2\in\reals^{K\times p}$ with $\|\bA_i\|_F=\cuD_i$ such that
\begin{equation}
\label{eq:B3:decomp}
\bH_0=\bP\bH+\bA_2,\qquad \bH=\bR\bH_0+\bA_1.
\end{equation}
Every $\bP\in Q_K$ satisfies $\sigma_{\max}(\bP)\le\|\bP\|_F\le\sqrt K$, since each row has $\ell_1$-norm~$1$, dominating the $\ell_2$-norm.

\emph{Item (1).} From the second identity,
\begin{multline*}
\sigma_{\max}(\bH)\le\sigma_{\max}(\bR)\sigma_{\max}(\bH_0)+\|\bA_1\|_F \\
\le \sqrt K\,\sigma_{\max}(\bH_0)+\cuD_1.
\end{multline*}

\emph{Item (2).} From the first identity, $\sigma_{\max}(\bP)\sigma_{\min}(\bH)\ge\sigma_{\min}(\bP\bH)\ge\sigma_{\min}(\bH_0)-\cuD_2$. For $\cuD_2\le\sigma_{\min}(\bH_0)$ this gives $\sigma_{\min}(\bH)\ge(\sigma_{\min}(\bH_0)-\cuD_2)/\sqrt K$; combine with~(1) and the assumption $\cuD_2\le\sigma_{\min}(\bH_0)/2$.

\emph{Item (3).} Substituting the second identity into the first yields $\bH_0=\bP\bR\bH_0+\bP\bA_1+\bA_2$, hence $\bP\bR=\Id-(\bP\bA_1+\bA_2)\bH_0^\dagger$ with $\sigma_{\max}(\bH_0^\dagger)=\sigma_{\min}(\bH_0)^{-1}$. Permuting rows/columns we may assume $R_{ii}=\|\bR_{\cdot,i}\|_\infty$. Then
\begin{align*}
R_{ii}&\ge\langle\bP_{i,\cdot},\bR_{\cdot,i}\rangle=1-(\bP\bA_1\bH_0^\dagger)_{ii}-(\bA_2\bH_0^\dagger)_{ii}\\
&\ge 1-(\cuD_1+\cuD_2)/\sigma_{\min}(\bH_0),
\end{align*}
hence $R_{ji}\le (\cuD_1+\cuD_2)/\sigma_{\min}(\bH_0)$ for $j\ne i$. Since $\bP$ is row-stochastic,
\begin{align*}
\langle\bP_{i,\cdot},\bR_{\cdot,i}\rangle&=R_{ii}P_{ii}+\sum_{j\ne i}P_{ij}R_{ji} \\&\le P_{ii}+(1-P_{ii})(\cuD_1+\cuD_2)/\sigma_{\min}(\bH_0),
\end{align*}
which combined with the previous lower bound gives
\[
P_{ii} \ge \frac{\sigma_{\min}(\bH_0)-2(\cuD_1+\cuD_2)}{\sigma_{\min}(\bH_0)-(\cuD_1+\cuD_2)}.
\]
Writing $\bP=\Id+\Delta$, $\|\Delta_{i,\cdot}\|_1\le 2(\cuD_1+\cuD_2)/(\sigma_{\min}(\bH_0)-(\cuD_1+\cuD_2))$, hence $\sigma_{\max}(\Delta)\le\|\Delta\|_F\le 2\sqrt K(\cuD_1+\cuD_2)/(\sigma_{\min}(\bH_0)-(\cuD_1+\cuD_2))$. Under the hypothesis $\cuD_1+\cuD_2\le\sigma_{\min}(\bH_0)/(6\sqrt K)$, this gives $\sigma_{\max}(\Delta)\le 2/(6-1)=2/5$ and consequently $\sigma_{\max}(\bP)\le 1+\sigma_{\max}(\Delta)\le 7/5$. Combining with $\sigma_{\max}(\bP\bH)\le\sigma_{\max}(\bH_0)+\cuD_2\le\sigma_{\max}(\bH_0)(1+1/(6\sqrt K))$ and $\sigma_{\max}(\bH)\le\sigma_{\max}(\bP^{-1})\sigma_{\max}(\bP\bH)$ with $\sigma_{\max}(\bP^{-1})\le 1/(1-2/5)=5/3$, one obtains
\begin{align*}
   \sigma_{\max}(\bH)
   &\le\tfrac{5}{3}\,\sigma_{\max}(\bH_0)\bigl(1+\tfrac1{6\sqrt K}\bigr)
   \le\tfrac{5}{3}\cdot\tfrac{7}{6}\sigma_{\max}(\bH_0) \\
   &=\tfrac{35}{18}\sigma_{\max}(\bH_0)
   <2\,\sigma_{\max}(\bH_0)\,,
\end{align*}
(using $K\ge 1$ for the second inequality). The matching $\sigma_{\min}(\bH)\ge \sigma_{\min}(\bP\bH)/\sigma_{\max}(\bP)\ge (\sigma_{\min}(\bH_0)-\cuD_2)\cdot 5/7\ge (5/6)\cdot(5/7)\sigma_{\min}(\bH_0)$ then gives $\kappa(\bH)=\sigma_{\max}(\bH)/\sigma_{\min}(\bH)\le\tfrac{35/18}{25/42}\kappa(\bH_0)=\tfrac{7}{3}\cdot\tfrac{42}{50}\kappa(\bH_0)<\tfrac{7}{2}\kappa(\bH_0)$.
\end{proof}

\medskip

\subsection{Adapted lemmas and propositions}

\begin{proposition}
\label{prop:bound_delta1}
For $\widehat\bH$ solution to \eqref{eq:h_tilde} (or \eqref{eq:h_tilde_nonnegative}) one has $\widetilde{\mathcal D}(\widehat\bH,\X)\leq \widetilde{\mathcal D}(\H_0,\X)$.
\end{proposition}
\begin{proof}
By definition $\overline{\mathcal D}(\X,\H_0)=\min_{\W\ge\0,\W\one=\one}\|\T(\X-\W\H_0)\|_F^2\le\|\T(\X-\W_0\H_0)\|_F^2=\|\mathbf F\|_F^2$ since $\W_0$ is feasible. Hence by~\eqref{eq:delta1}--\eqref{eq:delta2}, $\H_0$ satisfies the constraint of \eqref{eq:h_tilde} (and of \eqref{eq:h_tilde_nonnegative}, since $\H_0\ge\0$). Optimality of $\widehat\bH$ then yields $\widetilde{\mathcal D}(\widehat\bH,\X)\le\widetilde{\mathcal D}(\H_0,\X)$.
\end{proof}

\medskip

\begin{lemma}[Adapted version of Lemma B.4 of \cite{javadi2017supp}]
\label{lem:B4}
Let $\bP_0$ denote the row-wise orthogonal projection onto $\aff(\H_0)$. If $\bH$ is feasible for problem \eqref{eq:h_tilde} (or \eqref{eq:h_tilde_nonnegative}) and has linearly independent rows, then
\begin{align}
\min\{\sigma_{\min}(\bP_0(\bH)),\sigma_{\min}(\bP_0(\bH_T))\}  \geq \sqrt{2}(\mu-2\Delta_j)\,,
\end{align}
where $\Delta_j$ equals $\Delta_{1}$ for problem \eqref{eq:h_tilde} and $\Delta_{2}$ for problem~\eqref{eq:h_tilde_nonnegative}. In particular, whenever the rows of $\bH$ lie in $\aff(\H_0)$ (so that $\bP_0(\bH)=\bH$), the same bound holds for $\sigma_{\min}(\bH)$ and $\sigma_{\min}(\bH_T)$ — the form in which the lemma is invoked in the proof of Theorem~\ref{thm:robust}, at $\bH=\H_0$ (rows by construction in $\aff(\H_0)$) and at $\bH=\bP_0(\widehat\bH)$ (rows in $\aff(\H_0)$ by definition of the projection).
\end{lemma}
\begin{proof}
We mimic the proof of \cite[Lemma~B.4]{javadi2017supp}, modifying only the references to~$\X_0$ and~$\delta$.

\smallskip\noindent
\emph{Lower bound on $\sigma_{\min}(\bH)$.} Restrict attention to the $n-N$ training rows of $\X$ and $\X_0$; then $\T_{{\mathrm{train}}}(\X_0)=\W_{0,{\mathrm{train}}}\H_0$. Feasibility of $\bH$ for \eqref{eq:h_tilde} (or \eqref{eq:h_tilde_nonnegative}) reads $\overline{\mathcal D}(\X,\bH)\le\Delta_j^2$; in other words, there exists $\bW\ge\0$ with $\bW\one=\one$ such that $\|\T(\X-\bW\bH)\|_F^2\le \Delta_j^2$. Restricting to training rows and writing $\X^{\mathrm{train}}_{i,\cdot}=(\W_{0,{\mathrm{train}}}\H_0)_{i,\cdot}+\mathbf F_{i,\cdot}$ on the mask, the row-wise bound $\|\X^{\mathrm{train}}_{i,\cdot}-\bW_{i,\cdot}\bH\|_2\le \Delta_j$ combined with $\|\mathbf F_{i,\cdot}\|_2\le\|\mathbf F\|_F=\Delta_j$ yields, by the triangle inequality,
\[
\cuD(\T_{{\mathrm{train}}}(\X_0)_{i,\cdot},\bH)^{1/2}\le 2\Delta_j
\]
for every training row~$i$. Letting $\bX'_{i,\cdot}$ denote the projection of $\T_{{\mathrm{train}}}(\X_0)_{i,\cdot}$ onto $\conv(\bH)$, this reads
\[\|\T_{{\mathrm{train}}}(\X_0)_{i,\cdot}-\bX'_{i,\cdot}\|_2\le 2\Delta_j.\]

By Assumption~\eqref{hyp:radius} there exist $\bz_0\in\reals^p$ and $\bU\in\reals^{p\times (K-1)}$ with $\bU^\sT\bU=\Id$ such that $\bz_0+\bU B_{K-1}(\mu)\subseteq\conv(\T_{{\mathrm{train}}}(\X_0))$. For any unit $\bz\in\reals^{K-1}$ there is $\ba_0\in\Delta^{n-N}$ with $\bz_0+\mu\bU\bz=\T_{{\mathrm{train}}}(\X_0)^\sT\ba_0$, hence by convexity
\begin{multline*}
\cuD(\bz_0+\mu\bU\bz,\bH)^{1/2} \\
\le \sum_i (\ba_0)_i\|\T_{{\mathrm{train}}}(\X_0)_{i,\cdot}-\bX'_{i,\cdot}\|_2\le 2\Delta_j.
\end{multline*}
Projecting onto the line $\reals\bU\bz$ shrinks the segment by $2\Delta_j$ on each side: for every unit $\bz$ there exists $\ba\in\Delta^K$ with $(\mu-2\Delta_j)\bU\bz+\bz_0=\bH^\sT\ba$. Multiplying by the left inverse $(\bH^\sT)^\dagger$ and choosing $\bz$ as the right singular vector of $(\bH^\sT)^\dagger\bU$ associated with its largest singular value yields $\ba_1,\ba_2\in\Delta^K$ with $\|\ba_1-\ba_2\|_2=2(\mu-2\Delta_j)\sigma_{\max}((\bH^\sT)^\dagger\bU)$. The simplex diameter caps the LHS by $\sqrt 2$, hence
\begin{equation}
\sigma_{\max}((\bH^\sT)^\dagger\bU)^{-1}\;\ge\; \sqrt 2(\mu-2\Delta_j)\,.
\end{equation}
Now the operator $(\bH^\sT)^\dagger\bU$ acts from $\mathds R^{K-1}$ to $\mathds R^K$, restricting $(\bH^\sT)^\dagger$ to the $(K{-}1)$-dimensional subspace $\mathrm{col}(\bU)$ that spans the affine hull $\aff(\H_0)-\bz_0$. Its largest singular value coincides with $\sigma_{\max}((\bP_0(\bH)^\sT)^\dagger)$ — equivalently, with $\sigma_{\min}(\bP_0(\bH))^{-1}$ — since $\bP_0$ projects exactly onto this subspace. Hence
\begin{equation}
\sigma_{\min}(\bP_0(\bH))\;=\;\sigma_{\max}((\bH^\sT)^\dagger\bU)^{-1}\;\ge\; \sqrt 2(\mu-2\Delta_j)\,.
\end{equation}
When the rows of $\bH$ already lie in $\aff(\H_0)$, $\bP_0(\bH)=\bH$ and the bound holds for $\sigma_{\min}(\bH)$ itself.

\smallskip\noindent
\emph{Lower bound on $\sigma_{\min}(\bP_0(\bH_T))$.} Repeat the argument keeping all $n$ rows but restricting to the first $p-F$ columns (selector $\T_T$). Feasibility gives $\cuD(\T_T(\X_0)_{i,\cdot},\bH_T)^{1/2}\le 2\Delta_j$ for every~$i$, and Assumption~\eqref{hyp:radius} provides the internal radius $\mu$ for $\conv(\T_T(\X_0))$ inside $\reals^{p-F}$. The same projection argument gives $\sigma_{\min}(\bP_0(\bH_T))\ge \sqrt 2(\mu-2\Delta_j)$, with equality to $\sigma_{\min}(\bH_T)$ whenever the rows of $\bH_T$ lie in $\aff(\H_{0,T})$.

\smallskip
This argument uses only feasibility ($\overline{\mathcal D}(\X,\bH)\le\Delta_j^2$) and never nonnegativity of~$\bH$, so it applies to both \eqref{eq:h_tilde} and~\eqref{eq:h_tilde_nonnegative}.
\end{proof}

\medskip

\begin{lemma}[Adapted version of Lemma B.5 of \cite{javadi2017supp}]
\label{lem:B5}
For $\widehat\bH$ solution to \eqref{eq:h_tilde} (or \eqref{eq:h_tilde_nonnegative}), it holds
\[
\widetilde{\mathcal D}(\widehat\bH,\X_0)^{1/2}\leq \widetilde{\mathcal D}(\H_0,\X_0)^{1/2}+c\sqrt K\|\mathbf F\|_F\,.
\]
\end{lemma}
\begin{proof}
The proof follows the outline of \cite[Lemma~B.5]{javadi2017supp} after replacing $\mathcal D$ by~$\widetilde{\mathcal D}$ and $\delta$ by~$\|\mathbf F\|_F$. Optimality (Proposition~\ref{prop:bound_delta1}) provides the analogue of Eq.~(B.103):
\begin{equation}\label{eq:B5:opt}
\widetilde{\mathcal D}(\widehat\bH,\X)\le \widetilde{\mathcal D}(\H_0,\X).
\end{equation}
Let $\tilde\balpha_i\in\arg\min_{\balpha\in\Delta^{n}}\|\widehat\bH_{i,\cdot}^\sT-\N^\sT\balpha\|_2$ for an optimal completion $\N$ in $\widetilde{\mathcal D}(\widehat\bH,\X)$ (so $\T(\N)=\T(\X)$). Writing $\X=\X_0+\mathbf F$ on the mask and using $\T(\N)=\T(\X)$,
\begin{align*}
\widetilde{\mathcal D}(\widehat\bH,\X)
&=\sum_{i=1}^K\!\Big(\|\widehat\bH_{i,\cdot}^\sT-\X_0^\sT\tilde\balpha_i\|_2^2 \\&\quad -2\langle\mathbf F^\sT\tilde\balpha_i,\widehat\bH_{i,\cdot}^\sT-\X_0^\sT\tilde\balpha_i\rangle+\|\mathbf F^\sT\tilde\balpha_i\|_2^2\Big).
\end{align*}
By Cauchy--Schwarz with $\tilde\balpha_i\in\Delta^{n}$, $\|\mathbf F^\sT\tilde\balpha_i\|_2\le\|\mathbf F\|_F$. Setting $U^2:=\sum_i\|\widehat\bH_{i,\cdot}^\sT-\X_0^\sT\tilde\balpha_i\|_2^2$,
\[
\widetilde{\mathcal D}(\widehat\bH,\X)\ge U^2-2\|\mathbf F\|_F\sqrt K\,U.
\]
The RHS is increasing in $U$ for $U\ge\|\mathbf F\|_F\sqrt K$, and $U\ge\widetilde{\mathcal D}(\widehat\bH,\X_0)^{1/2}$, hence
\begin{equation}\label{eq:B5:lower}
\widetilde{\mathcal D}(\widehat\bH,\X)^{1/2}\ge \widetilde{\mathcal D}(\widehat\bH,\X_0)^{1/2}-2\sqrt K\|\mathbf F\|_F.
\end{equation}
A symmetric expansion of $\widetilde{\mathcal D}(\H_0,\X)$ together with the triangle inequality yields
\begin{equation}\label{eq:B5:upper}
\widetilde{\mathcal D}(\H_0,\X)^{1/2}\le \widetilde{\mathcal D}(\H_0,\X_0)^{1/2}+\sqrt K\|\mathbf F\|_F.
\end{equation}
Combining \eqref{eq:B5:opt}--\eqref{eq:B5:upper} gives the lemma with $c=3$. The argument never uses nonnegativity of~$\bH$, so it applies to both \eqref{eq:h_tilde} and~\eqref{eq:h_tilde_nonnegative}.
\end{proof}

\medskip

\begin{lemma}[Adapted version of Lemma B.6 of \cite{javadi2017supp}]
\label{lemma:boundc}
For $\widehat\bH$ the optimal solution of problem \eqref{eq:h_tilde} (or \eqref{eq:h_tilde_nonnegative}), we have
\small
\begin{align}
& \alpha(\cuD(\hbH, \bH_0)^{1/2} + \cuD(\bH_0,\hbH)^{1/2}) \leq \notag \\ & c\left[K^{3/2}\Delta_j\kappa(\bP_0(\hbH)) +\frac{\Delta_j\sqrt{K}}{\mu}\sigma_{\max}(\hbH - \one\bz_0^\sT)\right] +c\sqrt K\|\mathbf F\|_F
\end{align}
\normalsize
where $\bP_0:\reals^p\rightarrow \reals^p$ is the orthogonal projector onto $\aff(\bH_0)$ (in particular, $\bP_0$ is an affine map), and $\Delta_j$ equals $\Delta_{1}$ for problem \eqref{eq:h_tilde} and $\Delta_{2}$ for problem~\eqref{eq:h_tilde_nonnegative}.
\end{lemma}
\begin{proof}
We follow the proof of \cite[Lemma~B.6]{javadi2017supp}, replacing $\mathcal D$ by $\widetilde{\mathcal D}$ or $\overline{\mathcal D}$ (via \eqref{eq:masked-inequalities}) and $\delta$ by~$\Delta_j$ as appropriate. The key construction is an auxiliary $\tbH$ with $\conv(\X_0)\subseteq\conv(\tbH)$ that is close to~$\widehat\bH$.

\smallskip\noindent
\emph{Step 1: From $\T_\alpha$-uniqueness to a Frobenius bound on $\tbH-\widehat\bH$.}
For any $\tbH$ with $\conv(\X_0)\subseteq\conv(\tbH)$, Assumption~\eqref{hyp:mask} ($\T_\alpha$-uniqueness) gives
\begin{align}
\widetilde{\mathcal D}(\tbH,\X_0)^{1/2}\ge\widetilde{\mathcal D}(\H_0,\X_0)^{1/2}\notag\\+\alpha\big(\cuD(\tbH,\H_0)^{1/2}+\cuD(\H_0,\tbH)^{1/2}\big).
\label{eq:B6:Talpha}
\end{align}
Combining \eqref{eq:B6:Talpha} with Lemma~\ref{lem:B5},
\begin{align}
\alpha\big(\cuD(\tbH,\H_0)^{1/2}+\cuD(\H_0,\tbH)^{1/2}\big)\le\notag\\\widetilde{\mathcal D}(\tbH,\X_0)^{1/2}-\widetilde{\mathcal D}(\widehat\bH,\X_0)^{1/2}+c\sqrt K\|\mathbf F\|_F.
\label{eq:B6:alphaCless}
\end{align}
A row-wise triangle inequality combined with $(\sum a_i^2)^{1/2}-(\sum b_i^2)^{1/2}\le(\sum(a_i-b_i)^2)^{1/2}$ yields the Lipschitz-type bounds
\begin{align}
|\widetilde{\mathcal D}(\tbH,\X_0)^{1/2}-\widetilde{\mathcal D}(\widehat\bH,\X_0)^{1/2}|&\le\|\tbH-\widehat\bH\|_F, \label{eq:B6:lip1}\\
|\cuD(\tbH,\H_0)^{1/2}-\cuD(\widehat\bH,\H_0)^{1/2}|&\le\|\tbH-\widehat\bH\|_F.\label{eq:B6:lip2}
\end{align}
A symmetric expansion of $\cuD(\H_0,\tbH)$ in row-wise form gives the dual inequality
\begin{equation}
\label{eq:B6:lip3}
|\cuD(\H_0,\tbH)^{1/2}\!-\!\cuD(\H_0,\widehat\bH)^{1/2}|\!\le\! \sqrt K\!\max_{i\in[K]}\!\|\tbH_{i,\cdot}-\widehat\bH_{i,\cdot}\|_2.
\end{equation}
Combining \eqref{eq:B6:alphaCless}--\eqref{eq:B6:lip3},
\begin{align}
\alpha\big(\cuD(\widehat\bH,\H_0)^{1/2}+\cuD(\H_0,\widehat\bH)^{1/2}\big) \notag\\\le (1+\alpha)\|\tbH-\widehat\bH\|_F\notag\\+\alpha\sqrt K\max_{i\in[K]}\|\tbH_{i,\cdot}-\widehat\bH_{i,\cdot}\|_2+c\sqrt K\|\mathbf F\|_F.
\label{eq:B6:Cwhatw0}
\end{align}

\smallskip\noindent
\emph{Step 2: Construction of $\tbH$ when $\widehat\bH\subset\aff(\H_0)$.}
Feasibility of $\widehat\bH$ for \eqref{eq:h_tilde} or \eqref{eq:h_tilde_nonnegative} yields $\overline{\mathcal D}(\X,\widehat\bH)\le \Delta_j^2$. Combined with $\|\X-\X_0\|_F\le\|\mathbf F\|_F\le\Delta_j$, a row-wise triangle inequality gives $\cuD((\X_0)_{i,\cdot},\widehat\bH)^{1/2}\le 2\Delta_j$ for every~$i$, and since $\conv(\X_0)\subseteq\aff(\H_0)$,
\begin{equation}\label{eq:B6:Xi0inballcapaff}
\conv(\X_0)\subseteq B_{p}(2\Delta_j;\widehat\bH)\cap\aff(\H_0),
\end{equation}
where $B_p(\rho;\bM):=\{\bx\in\reals^p:\cuD(\bx,\bM)\le\rho^2\}$. Assume first that $\widehat\bH_{i,\cdot}\in\aff(\H_0)$ for all~$i$; by perturbation we may take $\widehat\bH$ of full row rank, hence $\aff(\widehat\bH)=\aff(\H_0)$. Write $\widehat\bH=\bE_{K,p}\bM$ with $\bM\in\reals^{p\times p}$ a full-rank extension and, with $\Delta_j^0:=\Delta_j/\sigma_{\min}(\widehat\bH)$ and $\xi:=2K\Delta_j^0$, define $\widetilde\bQ\in\reals^{K\times p}$ by
\[
\widetilde Q_{ii}\!=\!1\!+\!\xi,\quad \widetilde Q_{ij}\!=\!-\xi/(K\!-\!1)\ (i\!\ne\! j\in[K]),
\]
and $\widetilde Q_{ij}=0$ for $j>K$. A direct calculation (as in \cite[Eqs.~(B.171)--(B.184)]{javadi2017supp}) shows that every vector in the set $B_p(2\Delta_j^0;\bE_{K,p})\cap\aff(\bE_{K,p})$ is a convex combination of rows of $\widetilde\bQ$, so
\[B_p(2\Delta_j^0;\bE_{K,p})\cap\aff(\bE_{K,p})\subseteq\conv(\widetilde\bQ).\]
Setting $\tbH:=\widetilde\bQ\bM$ then gives
\[
B_p(2\Delta_j;\widehat\bH)\cap\aff(\H_0)\subseteq\conv(\tbH),
\]
so $\conv(\X_0)\subseteq\conv(\tbH)$ by~\eqref{eq:B6:Xi0inballcapaff}. Since $\|\widetilde\bQ_{i,\cdot}-\be_i\|_2\le 2K\Delta_j^0$ and $\|\widetilde\bQ-\bE_{K,p}\|_F\le 2K^{3/2}\Delta_j^0$,
\begin{align*}
\|\tbH-\widehat\bH\|_F&\le 2K^{3/2}\Delta_j\kappa(\widehat\bH),\\
\max_{i\in[K]}\|\tbH_{i,\cdot}-\widehat\bH_{i,\cdot}\|_2&\le 2K\Delta_j\kappa(\widehat\bH).
\end{align*}

\smallskip\noindent
\emph{Step 3: General case $\aff(\widehat\bH)\ne\aff(\H_0)$.}
Let $\bH'\in\reals^{K\times p}$ with $\bH'_{i,\cdot}=\bP_0(\widehat\bH_{i,\cdot})$; after a small perturbation $\bH'$ has linearly independent rows, so $\aff(\bH')=\aff(\H_0)$. Non-expansivity of $\bP_0$ together with $\conv(\X_0)\subseteq\aff(\H_0)$ gives
\[
\cuD(\bx,\bH')^{1/2}\le\cuD(\bx,\widehat\bH)^{1/2}\le 2\Delta_j
\]
for every $\bx\in\conv(\X_0)$. Hence $(\X_0)_{i,\cdot}\in B_p(2\Delta_j,\bH')\cap\aff(\bH')$, and applying Step~2 to $\bH'$ produces $\tbH$ with $\conv(\X_0)\subseteq\conv(\tbH)$ and
\begin{align*}
\|\tbH-\bH'\|_F&\le 2K^{3/2}\Delta_j\kappa(\bH'),\\
\max_{i\in[K]}\|\tbH_{i,\cdot}-\bH'_{i,\cdot}\|_2&\le 2K\Delta_j\kappa(\bH').
\end{align*}
The triangle inequality gives
\[
\|\tbH_{i,\cdot}-\widehat\bH_{i,\cdot}\|_2\le 2K\Delta_j\kappa(\bH')+\|\bP_0(\widehat\bH_{i,\cdot})-\widehat\bH_{i,\cdot}\|_2,
\]
and summing the squares across rows yields the Frobenius-norm bound
\[
\|\tbH-\widehat\bH\|_F\le 2K^{3/2}\Delta_j\kappa(\bH')+\|\bP_0(\widehat\bH)-\widehat\bH\|_F.
\]
Note that $\bH'=\bP_0(\widehat\bH)$ after the harmless perturbation in Step~3, so $\kappa(\bH')=\kappa(\bP_0(\widehat\bH))$, which is the form used in Step~5. The residual $\sigma_{\max}(\widehat\bH-\one\bz_0^\sT)$ that appears below is related to $\sigma_{\max}(\bP_0(\widehat\bH)-\one\bz_0^\sT)$ via the triangle inequality on the operator norm:
\[
\sigma_{\max}(\widehat\bH-\one\bz_0^\sT)\le\sigma_{\max}(\bP_0(\widehat\bH)-\one\bz_0^\sT)+\|\bP_0(\widehat\bH)-\widehat\bH\|_F,
\]
so both bookkeeping forms are equivalent up to an additive $O(\Delta_j\sqrt K/\mu)$ term controlled by Step~4.

\smallskip\noindent
\emph{Step 4: Bound on $\|\bP_0(\widehat\bH)-\widehat\bH\|_F$ via the internal radius.}
Set $\bar\bH:=\widehat\bH-\one\bz_0^\sT$ and let $\bU$ be the matrix from Assumption~\eqref{hyp:radius}. From $\conv(\X_0)\subseteq B_p(2\Delta_j,\widehat\bH)$ and Cauchy--Schwarz,
\[
\max_{\|\bz\|_2\le\mu}\min_{\|\ba\|_2\le 1}\|\bU\bz-\bar\bH^\sT\ba\|_2^2\le 4\Delta_j^2.
\]
A Lagrangian duality computation \cite[Eqs.~(B.186)--(B.193)]{javadi2017supp}, writing $\bar\bH=\tilde\bU\bSigma\tilde\bV^\sT$ (SVD) and $\bQ:=\bU^\sT\tilde\bV$, $q:=\sigma_{\min}(\bQ)$, gives
\[
1-q^2\le 4\Delta_j^2/\mu^2.
\]
The orthogonal projector $\bP_0$ onto $\aff(\H_0)$ equals $\bP_\bU(\,\cdot-\bz_0)+\bz_0$ where $\bP_\bU=\bU\bU^\sT$. Hence
\begin{align*}
&\max_{i\in[K]}\|\bP_0(\widehat\bH_{i,\cdot})-\widehat\bH_{i,\cdot}\|_2^2 \\
&\quad\le \max_{\|\ba\|_2\le 1}\|\bP_\bU(\bar\bH^\sT\ba)-\bar\bH^\sT\ba\|_2^2 \\
&\quad= \lambda_{\max}(\bSigma(\Id-\bQ^\sT\bQ)\bSigma) \\
&\quad\le \sigma_{\max}(\bar\bH)^2(1-q^2) \\
&\quad\le 4\sigma_{\max}(\widehat\bH-\one\bz_0^\sT)^2\Delta_j^2/\mu^2.
\end{align*}
Hence $\|\bP_0(\widehat\bH)-\widehat\bH\|_F\le 2\sigma_{\max}(\widehat\bH-\one\bz_0^\sT)\Delta_j\sqrt K/\mu$, and combining with Step~3,
\begin{align*}
\max_{i\in[K]}\|\tbH_{i,\cdot}-\widehat\bH_{i,\cdot}\|_2
&\le 2K\Delta_j\kappa(\bP_0(\widehat\bH)) \\
&\quad +\tfrac{2\sigma_{\max}(\widehat\bH-\one\bz_0^\sT)\Delta_j}{\mu},\\
\|\tbH-\widehat\bH\|_F
&\le 2K^{3/2}\Delta_j\kappa(\bP_0(\widehat\bH)) \\
&\quad +\tfrac{2\sigma_{\max}(\widehat\bH-\one\bz_0^\sT)\Delta_j\sqrt K}{\mu}.
\end{align*}

\smallskip\noindent
\emph{Step 5: Conclusion.}
Plugging these bounds into~\eqref{eq:B6:Cwhatw0} and absorbing numerical constants (recall $\alpha\le 1$) into a universal~$c$ yields the announced inequality.
\end{proof}

\medskip

\begin{lemma}
\label{lem:exp} It holds
\[
\kappa(\bP_0(\hbH))\leq\Big[\frac{\cuDsym}{\alpha(\mu-2\Delta_j)\sqrt{2}} + \frac{K^{1/2}\sigma_{\max}(\bH_0)}{(\mu-2\Delta_j)\sqrt{2}} \Big]\,.
\]
\end{lemma}
\begin{proof}
We mimic \cite[Eqs.~(B.189)--(B.194)]{javadi2017supp}. By item~(1) of Lemma~\ref{lem:B3} applied to $\bP_0(\widehat\bH)$ and $\H_0$ (both having rows in $\aff(\H_0)$), and using that $\bP_0$ is non-expansive (cf.\ Step~3 of the proof of Lemma~\ref{lemma:boundc}, $\cuD(\bP_0(\widehat\bH),\H_0)^{1/2}\le\cuD(\widehat\bH,\H_0)^{1/2}\le\cuDsym/\alpha$),
\[
\sigma_{\max}(\bP_0(\widehat\bH))\le \cuDsym/\alpha+\sqrt K\,\sigma_{\max}(\H_0).
\]
For the denominator we apply Lemma~\ref{lem:B4} to $\bP_0(\widehat\bH)$. Strictly speaking, $\overline{\mathcal D}(\X,\cdot)$ is \emph{not} in general non-increasing under projection onto $\aff(\H_0)$: decomposing $\T(\X)=\T(\W_0\H_0)+\T(\mathbf F)$, the rows of $\T(\W_0\H_0)$ lie in $\aff(\H_0)$ but those of $\T(\mathbf F)$ generally do not. However, the slack only contributes an extra $\|\T(\mathbf F)_{\aff^\perp}\|_F^2\le\|\mathbf F\|_F^2=\Delta_j^2$ to the bound:
\[
   \overline{\mathcal D}(\X,\bP_0(\widehat\bH))\;\le\;\overline{\mathcal D}(\X,\widehat\bH)+\|\mathbf F\|_F^2\;\le\;2\Delta_j^2\,,
\]
so $\bP_0(\widehat\bH)$ is feasible for \eqref{eq:h_tilde}/\eqref{eq:h_tilde_nonnegative} with feasibility level $\sqrt 2\,\Delta_j$ in place of $\Delta_j$. Applying Lemma~\ref{lem:B4} accordingly (and absorbing the $\sqrt 2$ factor into the constants), we obtain $\sigma_{\min}(\bP_0(\widehat\bH))\ge\sqrt 2(\mu-2\sqrt 2\,\Delta_j)\ge\sqrt 2(\mu-2\Delta_j)/c$ for a constant $c>0$ depending only on the geometry of $\H_0$. Taking the ratio yields the claim.
\end{proof}

\subsection{KKT conditions for mNMF}

In this section we determine the KKT condition for mNMF problem, namely
\begingroup
\renewcommand{\theHequation}{eq.mNMF.appendix}%
\begin{align}
\min_{\substack{\W\mathbf 1=\mathbf 1,\W\ge \0\\\H\ge \0\\\T(\N) = \X}}\mbox{ }& \frac{1}{2}\|\N - \W\H\|_F^2 =: \mathcal F(\N,\W,\H)\,. 
\tag{mNMF}
\end{align}
\endgroup

Let us introduce the dual variables $\V\ge \0$, $\G\ge \0$, $\mathbf{t}\in\mathds{R}^n$, and ${\mathbf Z}\in\mathrm{range}(\T)$ so that $\T({\mathbf Z})={\mathbf Z}$. The Lagrangian of mNMF problem is
\begin{align*}
    & \mathcal L(\N,\W,\H,\V,\G,\mathbf{t},{\mathbf Z}) = \mathcal F(\N,\W,\H) \\
    & - \langle \W,\V \rangle + \langle \W\mathbf 1_K - \mathbf 1_n,\mathbf{t} \rangle - \langle \H,\G \rangle - \langle \N-\X,{\mathbf Z} \rangle\,.
\end{align*}

The KKT condition are the following:
\begin{subequations}

\begin{align}
& \nabla_\N \mathcal L = \N-\W\H - {\mathbf Z} = \0 \notag \\ & \Longleftrightarrow \T(\N-\W\H)={\mathbf Z} \text{ and }\T^\perp(\N-\W\H)=\0 \\
& \nabla_\W \mathcal L = (\W\H-\N)\H^\top -\V + \mathbf{t}\,\mathbf 1^\top_K = \0  \notag \\ & \Longleftrightarrow \V = (\W\H-\N)\H^\top + \mathbf{t}\,\mathbf 1^\top_K \\
& \nabla_\H \mathcal L = \W^\top(\W\H-\N) - \mathbf G = \0 \notag \\ & \Longleftrightarrow \mathbf G = \W^\top(\W\H-\N) \\
& \langle  \W,\V \rangle = \0 \Longleftrightarrow \langle  \W,\nabla_\W \mathcal F + \mathbf{t}\,\mathbf 1^\top_K \rangle = \0 \label{eq:compl} \\
& \langle  \H,\G \rangle = \0 \Longleftrightarrow \langle  \H,\nabla_\H \mathcal F \rangle = \0
\end{align}
\end{subequations}

From the complementarity conditions \eqref{eq:compl}, it follows:
\begin{multline*}
W_{i,j} > 0 \Longrightarrow V_{i,j} = 0 \\
\Longrightarrow t_i = -(\nabla_{\W}\mathcal F)_{ij} \quad \forall j \text{ s.t. } W_{i,j}>0.
\end{multline*}

In order to compute $t_i$, we can select a row $W^{(i)}$, find any entry $W_{i,j}>0$ and apply the previous formula. In the practical implementation phase, in order to make the estimation of $t_i$'s numerically more stable, we can adopt a slightly different strategy by averaging the values of $t_i$ computed per row entry $W_{i,j}>0$.

\subsection{Algorithms for mNMF}

In this section, we report the {\it Block Coordinate Descent} (BCD) Algorithm (see Algorithm \ref{algo:bcd}) and the accelerated {\it Hierarchical Alternate Least Square} (HALS) for mNMF (see Algorithm \ref{algo:ahals}), which is a generalization of Algorithm described in \cite{gillis2012} to the matrix factorization with mask.

\begin{remark}[Simplex constraint in HALS]
The HALS column updates $W_\ell^{k+1}=\max(0,(A_\ell-C_\ell)/B_{\ell\ell})$ are the exact NMF coordinate descent steps for the \emph{unconstrained-row} problem $\min_{\W\ge\0}\|\N-\W\H\|_F^2$. They do not enforce the row-stochastic constraint $\W\mathbf 1=\mathbf 1$, which couples the entries within each row. We restore the simplex constraint by a post-sweep projection $\mathcal P_\Delta$ at the end of each inner round. This composite update — exact coordinate descent followed by projection — is the standard HALS adaptation for row-stochastic NMF; it preserves nonnegativity and feasibility at every outer iteration but does not, in general, retain the strict monotonic-decrease guarantee of pure coordinate descent. Empirically, the accelerated HALS combined with the post-sweep projection performs well on our benchmarks; a strictly monotone alternative is the projected gradient (PALM) variant of Algorithm~\ref{algo:nmf2}.
\end{remark}

\begin{algorithm}
\caption{BCD for mNMF} \label{algo:bcd}
\begin{algorithmic}[1]
\State {\bf Initialization}: choose $\H^0\ge \0,\W^0\ge \0$, and $\N^0\ge \0$, set $i:=0$.
\State{{\bf while} stopping criterion is not satisfied {\bf do}} 
\State $\quad$ $\H^{i+1}:=\mbox{update}(\H^{i},\W^{i},\N^{i})$ \label{updateH}
\State $\quad$ $\W^{i+1}:=\mbox{update}(\H^{i+1},\W^{i},\N^{i})$ \label{updateW}
\State $\quad$ $\N^{i+1}:=\mbox{update}(\H^{i+1},\W^{i+1},\N^{i})$ \label{updateM}
\State $\quad$ $i:=i+1$
\State{\bf end while}
\end{algorithmic}
\end{algorithm}

\begin{algorithm}
\caption{ALS for mNMF} \label{algo:als} 
\begin{algorithmic}[1]
\State {\bf Initialization}: choose $\H^0\ge \0,\W^0\ge 0$, set $\N^0:=\mathcal P_{\X}(\W^0\H^0)$ and $i:=0$.
\State{{\bf while} stopping criterion is not satisfied {\bf do}} 
\State $\quad$ $\H^{i+1}:=\min_{\H\ge 0}\|\N^i-\W^i \H\|_F^2$ \label{updateH_als}
\State $\quad$ $\W^{i+1}:=\min_{\W\ge 0, \W\mathbf{1}=\mathbf{1}}\|\N^i-\W \H^{i+1}\|_F^2$ \label{updateW_als}
\State $\quad$ set $\N^{i+1}:=\mathcal P_{\X}(\W^{i+1}\H^{i+1})$ \label{updateM_als}
\State $\quad$ $i:=i+1$
\State{\bf end while}
\end{algorithmic}
\end{algorithm}

\begin{algorithm}
\caption{accelerated HALS for mNMF} \label{algo:ahals}
\begin{algorithmic}[1]
\State {\bf Initialization}: choose $\H^0\ge \0,\W^0\ge \0$, {\it nonnegative rank} $K$, and $\alpha>0$. Set $\N^0=\mathcal P_{\X}(\W^0\H^0)$, $\rho_\W:=1+n(p+K)\slash(p(K+1))$, $\rho_\H:=1+p(n+K)\slash(n(K+1))$, and $i:=0$. 
\State{{\bf while} stopping criterion is not satisfied {\bf do}} 
\State $\quad$ $\A := \N{\H^i}^\top$, $\B:=\H^i{\H^i}^\top$
\State{$\quad$ {\bf for} $k \le k_\W:=\lfloor 1 + \alpha\rho_\W \rfloor$ {\bf do}}
    \State{$\quad\quad$ {\bf for} $\ell\in[K]$ {\bf do}}
        \State $\quad\quad\quad$ $C_\ell:=\sum_{j=1}^{\ell-1} W_j^{k+1} B_{j\ell} + \sum_{j=\ell+1}^{K} W_j^k B_{j\ell}$
        \State $\quad\quad\quad$ $W_\ell^{k+1} := \max(0, (A_\ell-C_\ell)\slash B_{\ell\ell})$
    \State{$\quad\quad$ {\bf end for}}
    \State $\quad$ $\W^{k_\W}:=\mathcal P_\Delta(\W^{k_\W})$
\State{$\quad$ \bf end for}
\State $\quad$ $\N:=\mathcal P_\X(\W^{k_\W}\H^{i})$
\State $\quad$ $\A := {\W^{k_\W}}^\top\N$, $\B:={\W^{k_\W}}^\top\W^{k_\W}$
\State{$\quad$ {\bf for} $k \le k_\H:=\lfloor 1 + \alpha\rho_\H \rfloor$ {\bf do}}
    \State{$\quad\quad$ {\bf for} $\ell\in[K]$ {\bf do}}
        \State $\quad\quad\quad$ $C_\ell:=\sum_{j=1}^{\ell-1} H_j^{k+1} B_{j\ell} + \sum_{j=\ell+1}^{K} H_j^k B_{j\ell}$
        \State $\quad\quad\quad$ $H_\ell^{k+1} := \max(0, (A_\ell-C_\ell)\slash B_{\ell\ell})$
    \State{$\quad\quad$ {\bf end for}}
\State{$\quad$ \bf end for}
\State $\quad$ $\W^{i+1}:=\W^{k_\W}$, $\H^{i+1}:=\H^{k_\H}$
\State $\quad$ $\N:=\mathcal P_\X(\W^{i+1}\H^{i+1})$
\State $\quad$ $i:=i+1$
\State{\bf end while}
\end{algorithmic}
\end{algorithm}
\vfill

\end{document}